\documentclass[msom,sglanonrev]{informs4}
\RequirePackage{tgtermes}
\RequirePackage{newtxtext}
\RequirePackage{newtxmath}
\RequirePackage{bm}
\RequirePackage{endnotes}

\OneAndAHalfSpacedXII % Current default line spacing
\usepackage{algorithm}
\usepackage{algorithmic}
\usepackage{tikz}
\EquationsNumberedThrough    % Default: (1), (2), ...
\TheoremsNumberedThrough     % Preferred (Theorem 1, Lemma 1, Theorem 2)
\ECRepeatTheorems  %  

\usepackage[round, authoryear]{natbib}
\renewcommand{\bibname}{References}
\renewcommand{\bibsection}{\subsubsection*{\bibname}}

\makeatletter
\renewcommand{\@biblabel}[1]{} % no numbers/brackets
\makeatother
\usepackage{nicefrac}       % compact symbols for 1/2, etc.
\usepackage{microtype}      % microtypography
\usepackage{xcolor}         % colors
\usepackage{amsmath}
\usepackage{mathtools}
\usepackage{graphicx} % Required for inserting images
\usepackage{subcaption}
\usepackage{booktabs}
\usepackage{amscd}
\usepackage{amsbsy}
\usepackage{amssymb}
\usepackage{float}
\usepackage{comment} 
\usepackage{cleveref}
\usepackage{nicefrac}
\usepackage{enumitem}
\usepackage[most]{tcolorbox}
\newtcolorbox{promptbox}[1][]{
  enhanced jigsaw,    % drawing engine compatible with breaking
  breakable,          % allow the box to continue across page breaks
  colback=gray!10,    % background color
  colframe=black,     % border color
  boxrule=0.5pt,      % line width
  arc=4pt,            % rounded corners
  left=6pt, right=6pt, top=4pt, bottom=4pt,
  fonttitle=\bfseries,
  title=#1            % user-specified title
}

\newif\ifarxiv
\arxivtrue
\ifarxiv
\makeatletter
\def\theARTICLETOP{}
\RRHSecondLine{}
\LRHSecondLine{}
\long\def\theARTICLEABSTRACT{%
  \HOOKb
  \vspace*{18pt}
  \begin{minipage}[t]{\textwidth}\parindent1em
    \ABSfont
    \noindent\theABSTRACT\endgraf
    \vskip5pt
    \theKEYWORDS
    \theSUBJECTCLASS
    \theAREAOFREVIEW
    \theMSCCLASS
    \theORMSCLASS
    \theHISTORY
  \noindent\hrulefill
  \end{minipage}%
  \vspace*{0pt}
}
\makeatother
\fi

\ifarxiv
\definecolor{revisionblue}{RGB}{0,0,0}% arXiv: all text black
\else
\definecolor{revisionblue}{RGB}{0,0,0}
\fi

\newenvironment{revision}{\par\begingroup\color{revisionblue}}{\par\endgroup}

\usepackage{aliascnt}
\theoremstyle{TH}
\newaliascnt{setting}{theorem}
\newtheorem{setting}[setting]{Setting}
\aliascntresetthe{setting}
\crefname{setting}{setting}{settings}
\Crefname{setting}{Setting}{Settings}

\usepackage[textsize=tiny]{todonotes}

\crefname{assumption}{assumption}{assumptions}

\usepackage{url}

\newcommand{\E}{\mathrm{E}}
\newcommand{\Eb}[1]{\E\left[#1\right]}
\newcommand{\bracks}[1]{\left[#1\right]}
\newcommand{\independent}{\mathrel{\perp\!\!\!\perp}}

\newcommand{\Enb}[1]{\E_n\bracks{{#1}}}
\newcommand{\norm}[1]{\left\lVert#1\right\rVert}
\newcommand{\abs}[1]{\left|#1\right|}

\MANUSCRIPTNO{MSOM-0001-2024.00}

\begin{document}
%%%%%%%%%%%%%%%%

%\setcounter{page}{1}
%\DOI{10.1287/xxxx.0000.0000}%

\ifarxiv
\RUNAUTHOR{Nwankwo, Goldkind, and Zhou}
\else
\RUNAUTHOR{}
\fi

% \RUNTITLE{Predictive Maintenance in Manufacturing}
\RUNTITLE{Optimal Causal Annotations}

% \TITLE{Optimal Resource Allocation in Humanitarian Logistics: A Stochastic Programming Approach}
\TITLE{Optimal Causal Annotations: An Application to Casenotes in Social Services}

%\ARTICLEAUTHORS{%
%\AUTHOR{Ezinne Nwankwo\textsuperscript{a}, Lauri Goldkind\textsuperscript{b}, Angela Zhou \textsuperscript{c}}
%\AFF{\textsuperscript{a}Department of Electrical Engineering and Computer Sciences, University of California, Berkeley, \EMAIL{ezinne\_nwankwo@berkeley.edu},  \textsuperscript{b} Department of Social Work,
%Fordham University, \EMAIL{goldkind@fordham.edu}, \textsuperscript{c} Department of Data Sciences and Operations,
%University of Southern California, \EMAIL{zhoua@usc.edu}} 
%\AUTHOR{}
%\AFF{Department of Social Work,
%Fordham University, \EMAIL{goldkind@fordham.edu}}
%
%\AUTHOR{Angela Zhou}
%\AFF{Department of Data Sciences and Operations,
%University of Southern California, \EMAIL{zhoua@usc.edu}}
%} % end of the block

\ifarxiv
\ARTICLEAUTHORS{%
\AUTHOR{Ezinne Nwankwo}
\AFF{Department of Electrical Engineering and Computer Sciences, University of California, Berkeley, \EMAIL{ezinne\_nwankwo@berkeley.edu}}
\AUTHOR{Lauri Goldkind}
\AFF{Department of Social Work, Fordham University, \EMAIL{goldkind@fordham.edu}}
\AUTHOR{Angela Zhou}
\AFF{Department of Data Sciences and Operations, University of Southern California, \EMAIL{zhoua@usc.edu}}
} % end of the block
\fi

\ABSTRACT{%
\textbf{Problem definition:} Estimating the causal effects of interventions is central to policy and operations, but outcome data are often missing or costly to obtain. LLMs and novel AI/ML tools can provide text annotation at scale but may be subject to unknown bias. When ground-truth outcomes require expensive expert labeling or follow-up (e.g., in healthcare or social services), budget limits typically allow only a fraction of the data to be labeled. Our work is motivated by collaboration with a nonprofit conducting street outreach in homelessness services, whose most interesting data and outcomes are in unstructured casenotes. A key question is: which observations should be selected for labeling under a fixed labeling budget? 

\textbf{Methodology/results:} Our method optimizes annotation probabilities to minimize the estimation variance of average treatment effect estimation. We derive a closed-form solution and establish that a feasible two-batch estimator achieves the best possible asymptotic variance. On simulated and real-world datasets, our method achieves lower MSE than random sampling and achieves 43\%-91\% reductions in labels and hence costs for the same interval widths. In the street outreach setting, we develop a case study on classifying progress towards a housing application from casenotes with LLMs, where we find out-of-the-box LLMs under-recognize client progress, highlighting the importance of grounding LLM judgment. Our estimates indicate that $8.6\%$ of clients improve 2-year housing outcomes due to more street outreach in the first six months, and increasing outreach from the first to third quartile increases maximum progress towards a housing application by an estimated half-step in our classification schema.

\textbf{Managerial implications:} LLM-as-a-judge is becoming increasingly used, though it may threaten estimation validity, as we see in our case study. %In our case study, we find out-of-the-box LLMs miss important organizational context and under-recognize progress. 
Our method enables valid causal estimation by leveraging a limited annotation budget effectively. 
Our new ability to evaluate causal impacts on intermediate nonprofit outcomes can inform operations management: our substantive findings suggest the causal returns of greater outreach are concave, so that it may be resource-efficient to expand the extensive margin of who receives outreach to new or underreached clients.
}%

\FUNDING{}

%Supplemental Material:
%Data Ethics & Reproducibility Note:

%\KEYWORDS{Stochastic programming, Decision support,Uncertainty, Disaster response, Optimization}

%\HISTORY{Received: Month DD, YYYY; Accepted: Month DD, YYYY; Published Online: Month DD, YYYY}

\maketitle
%%%%%%%%%%%%%%%%%%%%%%%%%%%%%%%%%%%%%%%%%%%%%%%%%%%%%%%%%%%%%%%%%%%%%%

% Text of your paper here

\section{Introduction}
\begin{revision}
    Empirical research in operations management has been enriched by the rise of micro-level structured transaction data \citep{terwiesch2019om}, as well as causal inference \citep{ho2017om}. Nonetheless, micro-level transaction data, such as appointment start and end times, sales or customer service calls, may administratively record that service was provided, but do not provide detail as to the quality of service or activities conducted therein. Especially in human services, unstructured text data such as clinical notes or call recordings are informative of service activities and qualities, but in their original text or multi-modal form are not usable for structured data and operational analysis. Recent work in social media and text analytics leverages text mining to extract service quality \citep{mejia2021service}, and interacting with unstructured data has only been accelerated with recent advances in artificial intelligence (AI), natural language processing, and large language models (LLMs). Yet, the scale and convenience of NLP/LLMs also bring risks of errors, biases, and/or hallucinations. As a result, ground-truth data annotation from professional experts and skilled humans is required to ensure validity, correctness, and reliable inferential conclusions --- rapidly cementing a new multi-billion dollar industry. 
\end{revision}

 Recent advances in AI/LLMs therefore enable extracting intermediate service activities as outcome measures in causal inference, rather than distal administrative records. Yet key managerial questions remain. How should scarce
expert labels be combined with abundant but imperfect LLM judgments to
support valid causal evaluation? Under a fixed measurement budget, which records should
receive costly expert ground-truth annotation? We study these questions in causal inference
settings in which treatment and covariates are observed for all units, but
ground-truth outcomes can be obtained only for a budget-constrained subset\footnote{A preliminary version of this paper was accepted to AISTATS 2026 \citep{nwankwo2025batch}. That initial version primarily validated the initial method. In this version, we develop the continuous treatments case further, and develop the progress annotation case study which illustrates the method in a bona-fide AI annotation setting, which required hiring human annotators and fine-tuning. These case study results are in \Cref{sec-casestudy}.}.
Random annotation permits valid estimation, but may use scarce expert effort
inefficiently. Our framework instead treats LLM outputs as low-cost auxiliary
signals, strategically selects records for expert review, and combines the
resulting information to improve the precision of average treatment effect
estimation while preserving valid statistical inference. We do so by analyzing annotation probabilities that optimize the estimation variance for causal inference with missing outcomes. Crucially, we focus on \textit{decision-oriented data collection} to best estimate the causal effect, a difference between potential outcomes. Therefore, our setting differs from active learning which merely targets improved prediction. Our goal of deciding based on causal effects changes our desiderata for causal estimation.   

% We focus on extracting structured information from text to extract structured \textit{outcomes} for causal inference. We study observational causal inference when treatment and covariates are
% observed for all units, but ground-truth outcomes can be obtained only for a
% budget-constrained subset. Random annotation permits valid estimation, but can
% % produce imprecise estimates. We therefore ask which observations should be
% labeled to minimize the asymptotic variance of an average treatment effect
% estimator. Our approach optimizes annotation probabilities within a doubly
% robust missing-outcome framework.

Our framework applies whenever we can obtain ground-truth labels at a cost. %This problem is not unique to the social work domain and can generally apply to cases of measurement error with misaligned modalities (such as text or images), where we can query the ground truth for some portion of the data at a cost.
% Examples of missing outcomes and surrogate labels 
% When outcome data is missing, it is often the case that information can be obtained from other data sources. 
For example, when an outcome variable, wages, is only observed from noisy or missing self-reports, surveyors could conduct follow-up interviews or verifications with participants to obtain wage data at a cost. 
In other settings, more closely related to the modern data annotation regime of AI, outcomes may be recorded in complex information such as text or images, which we refer to as ``complex embedded outcomes'', or $\tilde{Y}$. Such text or images measure the scalar outcome $Y \in \mathbb{R}^d$ whose causal effect is of final interest, and we assume the complex outcomes can be annotated at a cost. For example, in medical imaging, X-rays measure physical properties (such as tumor size), and the clinically relevant treatment effect is on the scalar measurements of physical properties (rather than in the space of general pixel images); expert radiologists can annotate X-rays at a cost.
% . In other settings, we do not have access to the ground truth but 
%Noisy measures from the same dataset (such as last year's wages) or transporting prediction models from national wage databases can be predictive. 
Trade-offs between costly expert annotation and scalable, weaker imputation are pervasive in data-intensive machine learning, for example as in the recent ``LLM-as-a-judge'' framework \citep{zheng2023judging}.%

%We consider a ground-truth data annotation setting where outcome data is missing, but potentially observable via the decision to obtain ground-truth annotation for an observation. 
% In this paper, we study observational causal inference with missing outcomes, where we can obtain information about ground-truth outcomes at a cost, via expert data annotation or follow-up. %Recent machine learning tools can label outcome information from noisy text or image observations, but n
% Naively using biased or potentially erroneous label \textit{predictions} as stand-ins for ground truth can invalidate statistical inference and confidence intervals. Small ground-truth annotation budgets allow valid estimation on a subsample, but introduce high variance. We build on doubly-robust causal inference with missing outcomes to determine where to sample additional outcome annotations to minimize the asymptotic variance of downstream treatment effect estimation. %See \Cref{fig:overview} for an overview of our project goals. 

%\angela{outline BG context here}
Our methodology is motivated by a collaboration with a nonprofit to evaluate the impact of street outreach on housing outcomes. Street outreach is an intensive intervention; caseworkers canvass for and build relationships with homeless clients and write case notes after each interaction.
%Outcome information is often missing, as the population of interest can be difficult to locate, but 
These notes provide rich information on service activities and client progress and status during the open-ended process of outreach. 
Was a client progressing towards housing or their goals, or were they facing other barriers? See \Cref{fig:bg-notes-overview} for a high-level description of such casenote data. Yet the organization's conventional
outcome, eventual housing placement, is delayed and sparse and depends partly
on housing availability and downstream placement decisions outside the
organization's control. In New York City alone, approximately $\$80,000,000$ per year is invested in homeless street outreach. The case study therefore raises a central
outcome-definition question: how should the organization measure whether
outreach is advancing clients through the housing process before a terminal
placement is observed?
\begin{revision}
Working with practitioners, we define a text-based measure of progress toward
a housing application, positioned between service outputs and eventual housing
placement. Expert annotators can identify this progress from case notes, but
reviewing the full corpus is prohibitively costly; LLMs can label notes at
scale, but our later empirical analysis shows they systematically undermeasure progress on a population that generally faces severe challenges. %The application therefore embodies
% the two data-collection decisions addressed by our framework: which notes
% should receive expert review, and how should expert labels be combined with
% model-generated judgments? 
We apply our framework to estimate the effect of
outreach on both housing placement and progress toward housing, while reducing
the expert annotation required for a given level of statistical precision.
%Importantly, such casenote data encodes rich detail about operationally relevant information regarding client's progress and service milestones; while such service interactions are too open-ended to be recorded in the tabular data. In our experience, outreach workers can extract structured ground-truth information %especially related to the structured processes of completing a housing application,
%from the unstructured text of case notes, but they cannot label millions of casenotes. 
We illustrate the benefits of our approach on ground-truthed housing outcome data, and on new client-based measures of progress defined on annotated casenotes. %They can provide context and recognize important milestones. Yet under-resourced outreach workers cannot label millions of case notes. While modern natural language processing tools can facilitate annotation at scale, they are often inaccurate. \textit{Given an annotation budget constraint, how can we strategically assign expert labels while leveraging weaker ML-predicted annotations to optimize causal effect estimation?} In this paper, we develop general methodology for optimizing data annotation %Because it is ideal to validate the methodological innovations with ground-truth data, we %later 
%including on ground-truthed housing outcome data. %; while the topic of validating the qualitative data analysis for intermediate outcomes will be the subject of a companion empirical paper. 

% \begin{figure}[t!]
%     \centering
%     \includegraphics[width=\linewidth]{figs/expert-annotation.pdf}
%     \caption{High-level overview of our method.}
%     \label{fig:overview}
% \end{figure}

\begin{figure}
    \centering
    \includegraphics[width=\linewidth]{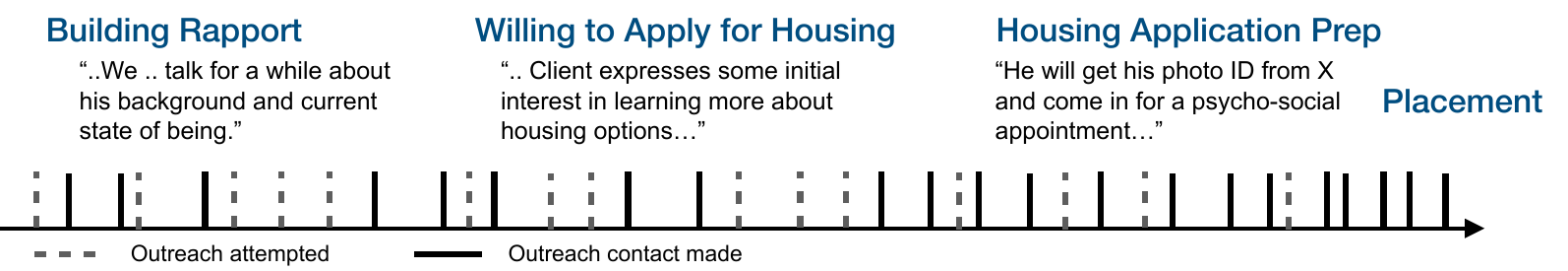}
    \caption{Overview of our nonprofit partner's casenotes and process.}
    \label{fig:bg-notes-overview}
\end{figure}

More broadly, there is a longstanding debate in nonprofit services regarding performance evaluation based on achieved (but uncertain) outcomes, versus process outputs (activities that the nonprofit conducts) \citep{ssirGettingResults}. %For example, the outputs of a soup kitchen are the number of bowls it serves while the outcomes and impact relate to how such assistance improves participants' well-being and other outcomes. 
The organization's former founder argues that ``everyone should be measured by one metric: how effectively they contribute to getting people housed and reducing homelessness. Period.'', rather than just measured for providing services \citep{columbiaHomelessness}. Outcomes and impacts can be difficult to measure. But our method also opens up a \textit{third} option of leveraging LLMs to identify intermediate outcomes, in between close-in nonprofit outputs (squarely within their control) and distal impacts on individuals (which are ultimately uncertain). 
\end{revision}
This study makes three contributions. First, we show how the objective of
causal evaluation changes which records should receive expert review.
% \Cref{thm-global-budget-solution} derives the annotation policy that minimizes
% the asymptotic variance of average treatment effect estimation under a fixed
% labeling budget, and 
\Cref{thm-asymptotic-convergence} shows that a feasible
batch-adaptive estimator attains the minimum asymptotic variance achievable after the
first-batch allocation. By characterizing the closed form solution, we further obtain key managerial insights: datapoints with low propensity of receiving observed treatment and high outcome uncertainty should be prioritized. We also show why balancing estimators that avoid plug-in inversion of treatment propensities work so well, and analyze when our method improves upon random sampling: under propensity variation and small budgets. Second, we quantify the \textit{operational} value of targeted ground-truth
collection. Across simulated, retail, and street-outreach data, targeted annotation requires
\(43\)--\(75\%\) fewer labels with the plug-in estimator and \(53\)--\(91\%\)
fewer labels with the balancing estimator to attain the same
confidence-interval width
(\Cref{fig:retail_hero_budget_saved,fig:outreach_budget_saved}), translating directly into valid inference while saving costs. Third, we show how to ground a validated LLM-as-a-judge system in organizational
process knowledge and expert data. Working with street-outreach
practitioners, we construct a progress schema that measures operational progress toward a housing application from text casenotes
(\Cref{tab:progress_scale}). Our evaluation shows that out-of-the-box LLMs small or large (12B or 235B) under-code progress, while fine-tuning the Gemma 3 12B model on
expert labels reduces this shortfall
(\Cref{tab:supp-band-error}) and improves merged accuracy upon a larger Qwen 235B zero-shot model
(\Cref{tab:gemma3-july-qwen-summary}). Model scale does
not substitute for organization-specific outcome design and ground-truth
data.  Substantively, we find a concave dose-response: increasing outreach from minimal to moderate levels raises expected maximum housing progress by roughly 0.3 schema points (a half-step on our scale), while further intensification yields no detectable gains — evidence for the resource-efficiency of expanding the extensive margin of who is outreached, rather than the intensive margin of already-engaged clients.
Together, these contributions tell managers \textit{what} to measure,
\textit{where to collect scarce ground truth}, and \textit{how to combine human and AI
judgments for causal evaluation}.

The remainder of the paper is organized as follows.
\Cref{sec-relatedwork} positions the study relative to work on data
annotation, active learning, adaptive experimental design, and causal
inference with imperfect outcomes. \Cref{section-problem-setup} formalizes
the costly-outcome measurement problem.
\Cref{section-methods,sec-analysis} derive the optimal annotation policy,
present the feasible estimator, and establish its statistical properties.
\Cref{sec-experiments} evaluates the method on simulated, retail, and
street-outreach data. \Cref{sec-casestudy} develops and validates
the progress schema and applies the framework to evaluating street outreach.

\section{Related work}\label{sec-relatedwork}
% 
% Most directly related work: 

\begin{revision}
    A significant line of work in operations focuses on the \textit{operations of data annotation}, and therefore the operational data lifecycle of AI. The central importance of strategically collecting data under binding budget constraints is a throughline from early work focused on crowdsourcing annotations \citep{wang2017cost,ipeirotis2010quality} to recent work on even more data-hungry modern multi-modal AI systems \citep{liao2025minority,JMLR:v26:23-0292}. Our work focuses on how the goal of \textit{informing better decisions} via causal effect estimation \textit{changes how we should optimally collect data}.
\end{revision}

Our model is closest to optimizing a validation set for causal inference with missing outcomes, which can be broadly useful for causal inference with non-standard measurement error. % but a small validation set.
Typical distributional conditions for non-standard measurement error \citep{schennach2016recent} are generally inapplicable to text or images, our motivating application. The most related works are those of \citet{naoki2023dsl}, which assumes that sampling probabilities for data annotation are known in order to obtain doubly-robust pseudo-outcomes, and \citet{zrnic2024active} which does optimize data sampling probabilities, but not for causal estimation. Both of these papers address non-causal estimands such as mean or M-estimation, whereas we focus on treatment effect estimation.

Our work follows the basic approach of finding outcome annotation probabilities that optimize the semiparametric efficient lower bound, whether via batch or full adaptivity. \citet{hahn2011adaptive} studied a two-stage procedure for estimating the ATE with a proportional asymptotic and showed asymptotic equivalence of their batched adaptive estimator to the optimal asymptotic variance. 
Other treatment-choice variants in the same framework include \citet{kato2020efficient,cook2024semiparametric,li2024csbae}. \citet{ao2024prediction} study which units to sample and label, but for a different estimand of a single population mean under a new algorithm, whereas we post-hoc annotate outcomes for the average treatment effect. 
 Crucially, prior papers allocate treatment, while we allocate the probability of revealing the outcome. We face new technical challenges of finite-sample instability from \textit{multiplied} inverse importance weights, where our closed-form characterization yields insights and our improvements from balancing-weight methods for final estimation.
% \citet{armstrong2022asymptotic} proves the semiparametric efficiency lower bound cannot be beat in general by adaptive designs; so this algorithmic paradigm is the best possible for the asymptotic inference objective. 

To be sure, the literature on adaptive treatment allocation or outcome annotation is vast, even in causal inference specifically. The problem of optimal annotation for efficient estimation differs from methods for active learning, online learning and bandits based on: whether methods choose treatments vs. annotate outcomes, and what is the objective: prediction error, decision regret, or best estimation of the ATE (smallest asymptotic variance).
Other paradigms such as active learning or bandits optimize different objectives from ours, and do not solve our problem directly \citep{Settles2009ActiveLL}. Active learning (AL) \emph{optimizes for prediction error, which is suboptimal for best estimation of the ATE}. 
Prior works \citep{jesson2021causal,sundin2019active,mittal2025planning} build on vanilla active learning for conditional average treatment effect (CATE) estimation, but reduce the problem to learning two regression functions, without targeting efficient ATE estimation.
%We provide an extensive discussion and empirical comparisons to active learning baselines in 
See \Cref{appendix-al-baselines} for further discussion on key differences, and additional experiments comparing to AL baselines. Many exciting recent works study adaptive experimentation under different desiderata, such as full adaptivity, in-sample decision regret or finite-sample guarantees \citep{gao2019batched,zhao2023adaptive,cook2024semiparametric,shiusing}\footnote{See \citep{zhao2024experimental} for a survey on experimental design in causal inference, and \citep{simchi2023multi,qin2024optimizing} for discussions on trading-off bandit regret vs. best-arm identification. }. %These are not directly applicable to our work. 
Some desiderata for \textit{treatment allocation} are {irrelevant} to our work on \textit{outcome/data annotation}. Batch annotation is more relevant for querying human annotators, instead of full adaptivity. Simple decision regret is important when changing treatment decisions online, but not relevant for outcome annotation of historically collected data. \citet{armstrong2022asymptotic} proves that further adaptivity cannot further improve asymptotic efficiency bounds. However, other technical tools like more advanced adaptive inference could be further adapted to our setting. 
% For example, identifying assumptions include either structural assumptions on the relationship between the surrogate and a distal outcome \citep{athey2019surrogate}, or sampling/data-combination assumptions across  two observational (possibly confounded) and experimental datasets \citep{yang2020combining,kallus2024role}. Although 

Regarding the use of auxiliary information in causal inference, many recent works have studied the use of surrogate or proxy information. Although our use of context $\tilde{Y}$ aligns with \textit{colloquial} notions of surrogates or proxies, recent advances in surrogate and proxy methods refer to specific models that differ from our direct measurement/costly observation setting \citep{athey2019surrogate,kallus2024role,naoki2023dsl}. See \Cref{apx-additional-related-work} for more discussion on the distinctions.

\section{Problem setup\label{section-problem-setup}}

We study causal inference with missing outcomes, where a simpler ground truth outcome $Y\in\mathbb{R}$ can be revealed via annotation of a more complex observation thereof (e.g., text or images), denoted $\tilde{Y}$. %After, 
We also discuss extensions to a setting where we can use $\tilde{Y}$ to enhance nuisance function estimation. 

In both cases, we assume the ground-truth data-generating process follows that of standard causal inference. The ground-truth data $(X,Z,Y(Z))$ includes covariates $X\in\mathcal{X}$, a binary treatment $Z\in\{0,1\}$, and potential outcomes $Y(Z)$ in the Neyman-Rubin potential outcome framework. We only observe $Y(Z)$ for the historically-assigned $Z$ and assume the usual stable unit value treatment assumption (SUTVA). \textit{If} all ground-truth outcomes were observed, estimation would reduce to the standard causal setting; the key challenge is missingness. Let $R\in\{0,1\}$ denote the presence ($R=1$) or absence ($R=0$) of the %latent
outcome $Y$. The \textit{observed} dataset is $(X,Z,R,RY)$, i.e. with missing outcomes.
(\Cref{tab:notation} in the appendix summarizes notation.)
Causal identification relies on the following assumptions:
% \begin{assumption}[Consistency]\label{asn-consistency}
% $Y = Z R Y(1) + (1-Z) R  Y(0).$
% \end{assumption}
\begin{assumption}[Treatment ignorability \citep{hernan2025causal}]\label{asn-tx-ignorability}
 $Y(Z) \independent Z \mid X.$
\end{assumption}

\begin{assumption}[$R$-ignorability \citep{rubin1976missing,bia2021dmlsampleselection}]\label{asn-r-ignorability}
      $R \independent Y(Z) \mid Z,X.$ 
\end{assumption}
\Cref{asn-tx-ignorability}, or unconfoundedness, posits that the observed covariates are fully informative of treatment. It is generally untestable but robust estimation is possible in its absence, e.g. via sensitivity analysis and partial identification \citep{zhao2019sensitivity,kallus2021minimax}. On the other hand, \Cref{asn-r-ignorability} \textit{is true by design}, since we choose what datapoints are annotated for ground-truth labels based on $(Z,X)$ alone. %\Cref{asn-r-ignorability} rules out instances of domain shift where documents from the target population are not available at annotation time, but captures the case where the full corpus of documents needed to be annotated is available from the outset. 

Though completely random sampling enables doubly-robust causal inference, we ask: how can we optimize our choice of annotated datapoints to improve the \textit{variance} of downstream estimation?
% Although one annotation approach is to sample completely at random, we are particularly concerned with \textit{how can we select datapoints for expert annotation for optimal estimation}? 
We assume a fixed annotation budget $B\in[0,1]$ that determines the fraction of the dataset that can be annotated. We define the propensity score and annotation (outcome observation)
probability as follows: 
\begin{align*}
          % \mu_z(X)&:= \Eb{Y|Z=z,R=1,X} \tag{outcome model}\\
        e_z(X) &:=P(Z=z|X) 
        \tag{propensity score}%), and } \;
         \\
        \pi(Z,X) &:= P(R=1|Z,X) \tag{annotation probability} %). }
        %}
\end{align*}
We assume positivity/overlap; that we observe treatment and outcome with nonzero probability.
\begin{assumption}[Treatment and annotation positivity \citep{hernan2025causal}]\label{asn-positivity} $ \epsilon < \pi(z,X) \leq 1-\epsilon, z \in \{0,1\}$ and $\epsilon < e_1(X)<1-\epsilon,$ with $\epsilon >0.$
\end{assumption}

\Cref{asn-tx-ignorability,asn-r-ignorability,asn-positivity} are standard in causal inference and we point the reader to textbook references for further discussion \citep{hernan2025causal,imbens2004nonparametricest,kennedy2020missingexposure}.

We define the outcome model, which is identified on the $R=1$ data by \Cref{asn-r-ignorability}, and the conditional variance:
    \begin{align*}
    \mu_z(X) &\coloneqq \E[Y \mid Z=z, X] \underset{asn. \ref{asn-r-ignorability}}{=} \E[Y \mid Z=z, R=1, X]\\
    \sigma^2_z(X) &\coloneqq \E[(Y - \mu_z(X))^2 \mid Z=z,X].
    \end{align*}
% $$ \mu_z(X) = E[Y \mid Z=z, R=1, X] $$

\textbf{Batch allocation setup.} We consider a two-batch adaptive protocol, where $n$ i.i.d. observations are randomly split into two batches. We consider a proportional asymptotic where the size of first batch, $n_1,$ is a fixed proportion $\kappa \in (0,1)$ of $n$.
\begin{assumption}[Proportional asymptotic \citep{hahn2011adaptive,li2024csbae}]
    $\lim_{n \rightarrow\infty} \frac{n_1}{n} = \kappa$. %^as $n\rightarrow \infty$.
\end{assumption}
In the first batch, we randomly assign annotations according to a small but asymptotically nontrivial fraction of the budget. Outcomes are realized and observed, and the nuisance models ($\hat{\mu}_z(x), \hat{e}_z(x), \hat{\sigma}^2_z(x) $) are trained on the observed data. In the second batch, we estimate the feasible pooled annotation
probability $\pi^\dagger$ and sample data so that the mixture
distribution over outcome observations achieves $\pi^\dagger$. We combine the results from both batches and use the data for ATE estimation.%, which we describe in \Cref{section-methods}.           

\textbf{Extension to decoding $Y$ from $\tilde{Y}$.}
\begin{revision}
So far, our missing outcomes framework aligns with multi-stage approaches in measurement error with a validation set. We will also discuss extensions to handle the case of AI/LLM annotation, where the plentiful $\tilde{Y}$ information is used to obtain outcome data $Y$ and enhance outcome model predictions. 
\end{revision}
% We provide an extension of our missing outcomes framework to settings where complex-embedded outcomes might be used not only for data annotation but also to enhance outcome model predictions. 
Though our method assumes ground-truth outcomes could be revealed for each datapoint, for example via follow-up surveys, in practice this is most likely relevant in \textit{data annotation} settings. Expert data annotation only works when there is some data to annotate: we denote this noisy observation $\tilde{Y}$, which could be text or images.  Given that a noisy observation 
$\tilde{Y}$ is available, a natural question is, when can $\tilde{Y}$ be included to further improve outcome prediction? 
\begin{revision}
% \paragraph{What changes when we decode $Y$ from $\tilde{Y}$ and therefore annotation probabilities depend on $\tilde{Y}$} 
% If annotation probabilities
% also use \(\tilde Y\), for example through an LLM prediction \(f_z(X,\tilde Y)\), we call this a decoding setting.
\begin{setting}[Decoding $Y$ from $\tilde{Y}$]\label{setting-yytilde}
    The annotation probability depends on $\tilde{Y}$ so that $\pi(z,X,\tilde Y)=\mathbb P(R=1\mid Z=z,X,\tilde Y)$ and we have an analogous annotation ignorability condition, which is satisfied by design.
\begin{assumption}[$R$-ignorability, decoding $Y$ from $\tilde{Y}$ ]\label{asn-r-ignorability-yytilde}
        $R\perp Y(Z)\mid Z,X,\tilde Y$.
\end{assumption}
\end{setting}
This changes estimation, which we discuss in more detail later on, but not the high-level conclusions regarding annotation probabilities. 

\end{revision}
There are several ways of incorporating the context into the outcome model. We denote an ML prediction based on $\tilde{Y}$ (with $X$ covariates and treatment information) as $f_z(X, \tilde{Y});$ for example zero-shot prediction using an LLM. If using black-box ML or LLM predictions, we recommend ensembling with $\E[Y|Z,X]$ or estimating $\E[Y\mid Z,R=1,f_z(X, \tilde{Y})]$ to calibrate LLM predictions, in order to satisfy statistical consistency conditions. %\footnote{Variants of $\hat\mu_z(X,\tilde{Y})$ include calibrating zero-shot predictions to the ground-truth as $\E[Y\mid Z,R=1,f_z(X, \tilde{Y})]$.
% , $\E[Y\mid Z=z,R=1,f_z(X, \tilde{Y})]$, 
% that is predicting $Y$ and including ML predictions as a covariate alongside $X$, or using various ensembling combinations. 
 (\citet{naoki2023dsl} also suggests this). 

\section{Method\label{section-methods}}

%\angela{add a paragraph to summarize what's in here and the flow of different subsections}
%In this work, we study the efficient use of data annotation, which produces biased and imperfect labels, in the downstream task of causal effects estimation. To the best of our knowledge, no work has addressed this challenge in causal estimation. The average treatment effect (ATE) is: 
We outline our method, starting with a recap of the  augmented inverse-propensity weighting (AIPW) estimator for causal inference with missing outcomes.
Then we optimize its asymptotic variance, characterize the optimal $\pi^*$, and give a feasible estimation procedure. (In \Cref{appendix:additional-results} we discuss the very similar case of treatment-specific budgets.) %We design a two-batch adaptive experiment.
%We describe feasible estimation of the ATE by the AIPW estimator (with missing outcomes).

\paragraph{Recap: Optimal asymptotic variance for the ATE with missing outcomes.} We seek to estimate the average treatment effect 
(ATE) on ground-truth outcomes $Y$. Define
\vspace{-0.05in}
\[ \tau = \E[Y(1) - Y(0)].\]
\vspace{-0.25in}
% \angela{revisit identification conditional on $\tilde{Y}$}

% it is implied that 
% \[ \E[Y(z)|X] = \E[Y|Z=z,X] = \E[Y|Z=z,R=1,X], \] 
\citet{bia2021dmlsampleselection} derives a double-machine learning estimator for ATE estimation with missing outcomes: 
\begin{align}
    \E[Y(z)] &= \E[\psi_z], 
    % \\
    \text{ where }
    \;\psi_z %&
    % &
    = \frac{\mathbf{1}[Z=z] R  (Y-\mu_z(X))}{e_z(X)  \pi(z,X)} + \mu_z(X),
    % \\
     \text{ and } \tau_{AIPW} = \E[\psi_1 - \psi_0].   
     \label{eqn-aipw-standard}
\end{align}
The outcome model $\mu_z(X)$ is estimated on data with observed outcomes. Under SUTVA and \Cref{asn-r-ignorability}, $\textstyle\E[Y(z)|X] = {\E[Y|Z=z,X] }= {\E[Y|Z=z,R=1,X]}.$
\begin{revision}
    \paragraph{Extension to the \Cref{setting-yytilde}, decoding $Y$ from $\tilde{Y}$.}
In the setting with complex embedded outcomes, define the $\tilde{Y}$-observation augmented outcome model $\mu_z^{\tilde Y}(X,\tilde Y)
    :=
    \mathbb E[Y\mid Z=z,X,\tilde Y].$ %where the outcome predictions $\mu_z(X,\tilde{Y})$ predict based on $\tilde{Y}$ information, this only changes the outcome model for evaluating the AIPW estimator. 
However, the relevant $\tilde{Y}(Z)$ information is only relevant for decoding observations for ``factual'' observed treatments, i.e. the outcome model must therefore switch between $X$-only prediction and $\tilde{Y}$-given annotation, $$\mu_z(Z,X,\tilde{Y}) = \mu_z(X) + \mathbb{I}[Z=z] \{\mu_z^{\tilde{Y}}(X,\tilde{Y}) - \mu_z(X) \}.$$
The corresponding AIPW estimator retains favorable properties when $\mu_z^{\tilde{Y}}$ and $\mu_z(X)$ are consistent, at a potential loss of consistency. 
The efficient score in a larger nonparametric observed-data model is
\begin{equation}
\phi_z^{\tilde Y}(O)
=
\mu_z(X)
%-\theta_z
+
\frac{\mathbb{I}[Z=z]}{e_z(X)}
\{\mu_z^{\tilde Y}(X,\tilde Y)-\mu_z(X)\}
+
\frac{\mathbb{I}[Z=z]R}
{e_z(X)\pi(z,X,\tilde Y)}
\{Y-\mu_z^{\tilde Y}(X,\tilde Y)\}.
\label{eqn-post-treatment-estimator}
\end{equation}

Since we optimize annotation probabilities, the optimization objective and solution remain analogous.

\end{revision}
We optimize the semiparametric efficient asymptotic variance with missing outcomes. We express the asymptotic variance of \citep{bia2021dmlsampleselection} in terms of $\mu_z, e_z, \pi$: %(proven in \citep{bia2021dmlsampleselection}), %which is closely related to the ATE of \citep{hahn1998variancebound}. 
\begin{proposition}\label{prop-avar-ate}

  %  Any regular estimator for the ATE with missing outcomes %$\hat{\mu_z(X)}$ for $\mu$
    %has the following lower bound on the asymptotic variance 
    The asymptotic variance $(\textrm{AVar})$ of the estimator according to \cref{eqn-aipw-standard} is: 
    \begin{align}
    \mathrm{AVar} %&
 =  
    % \\
     \mathrm{Var}[\mu_1(X) - \mu_{0}(X)] +
    \sum_{z\in\{0,1\}} \E\left[ \frac{\sigma^2_z(X)}{e_z(X)  \pi(z,X)} \right] \label{eqn-X-avar}
    \end{align}
    % + \Eb{ \frac{1}{e_{z'}(X) \cdot \pi(z',X)} \cdot [Y-\mu_{z'}(1,X)]^2} 
Define $v_z^{\tilde Y}(X)
=
\operatorname{Var}[\mu_z^{\tilde Y}(X,\tilde Y)\mid Z=z,X]$ and $\sigma_{z}^2(X, \tilde{Y})=\operatorname{Var}[Y \mid Z=z, X, \tilde{Y}].$
The asymptotic variance of the estimator according to the score \cref{eqn-post-treatment-estimator} is: 
    \begin{align*}
    \mathrm{AVar} %&
 =  
    % \\
     \mathrm{Var}[\mu_1(X) - \mu_{0}(X)] +
    \sum_{z\in\{0,1\}} 
\left\{ \E\left[ \frac{
v_z^{\tilde Y}(X)
}{e_z(X) } \right] +
\E\left[ \frac{\sigma_{z}^2(X, \tilde{Y})}{e_z(X)  \pi(z,X, \tilde{Y})} \right] 
\right\}
\end{align*}

\end{proposition}
Although for either estimator the asymptotic variance includes additional terms, we focus on optimizing the last $\pi$-dependent term and therefore the \textit{structure} of optimal annotation probabilities is identical with or without use of $\tilde{Y}$. For ease of exposition, in the following we focus on the $X$-dependent annotation probabilities optimizing \Cref{eqn-X-avar}, and discuss variations thereof for $X,\tilde{Y}$-dependent annotation probabilities.
% The first term is independent of $\pi$; we focus on optimizing the second term with respect to $\pi$. 
% \begin{remark}
    %We state the results for the base model, though they extend directly to the case with complex embedded outcomes. With these contextual outcomes, marginalizing over $\tilde{Y}$ leads to analogous expressions of the ATE and asymptotic variance that use $\hat{\mu}_z(X,\tilde{Y})$ instead of $\hat\mu_z(X)$.
    %On the other hand, $\hat{\sigma}^2_z(X)$ stays the same (sampling probabilities depend only on $(Z,X)$ and just correspondingly marginalizes over $\tilde{Y}$, ${\hat{\sigma}^2_z(X) = \E[(Y - \hat{\mu}_z(X,\tilde{Y}))^2 \mid Z=z,X=x]}$.
    % In this setting with noisy measurements $\tilde{Y},$ under the exclusion restriction \Cref{asn-exclusionrestriction}, the mean potential outcome is identified by regression adjustment: $\E[Y(z)] 
% =\E[\E[Y|Z=z,R=1,X, \tilde{Y} ]]
% =\E[\E[Y|Z=z,R=1,X]]$.%\] %and let 
%\[ \Delta(X) =\E[Y(z) - Y(z')|X] \] 
% \end{remark}

\paragraph{%Two-batch adaptive design and
Characterizing the optimal $\pi^*(z,x)$.}
 %\angela{Introduce the idea of two batches and talk about optimizing things in the population }

We first characterize the population optimal sampling probabilities $\pi^*(z,x)$, assuming the nuisance functions are known. We optimize the asymptotic variance over $\pi$ under a global sampling budget $B \in [0,1]$ over all annotations. 
 %For the oracle, we assume that we know the true values of the typically unknown nuisance functions. In the first batch, we sample data points according to probability $\pi_1$. We already know the true values for $\mu_z,\sigma^2_z(X),$ and $e_z(X)$. In the second batch, we plug in the values for the nuisance parameters and find the $\pi^*(z,X)$ that minimizes the asymptotic variance bound on $\tau_{AIPW}$ and satisfies the sampling budget constraint. The remainder of this sections shows the solution to the optimal annotation probability.  
%The setting is meaningful when the budget binds, $B \ll 1$, which is still practically relevant.%, without additional restrictions on how many are treated within each treatment subgroup. 
$\pi^*(z,x)$ solves 
\begin{align}
  \min_{0 < \pi(z,x) \leq 1, \forall z,x} &
  %\left\{
  \sum_{z\in\{0,1\}}
  \Eb{ \frac{\sigma^2_z(X)}{e_z(X)  \pi(z,X)}  } % \nonumber \\
  % &
  \text{ s.t. } \E[\pi(Z,X)]\leq B 
  %\right\} 
  % \tag{OPT (global budget)}
  \label{eqn-opt-global-budget}
\end{align}
Note that in the global budget constraint, $\E[\pi(Z,X)]=\E[\pi(1,X)\mathbf{1}[Z=1] +\pi(0,X)\mathbf{1}\left[Z=0\right] ]$. We can characterize the solution as follows.

\begin{theorem}[Optimal annotation probabilities]\label{thm-global-budget-solution}
First assume that the unconstrained allocation is strictly feasible: $\max_{z\in\{0,1\}}
\operatorname*{ess\,sup}_{x\in\mathcal X}
\frac{
B\sqrt{\sigma_z^2(x)}
}{
e_z(x)\,
\mathbb E\!\left[
\sqrt{\sigma_1^2(X)}+\sqrt{\sigma_0^2(X)}
\right]
}
<1$. Then the optimal $X$-dependent annotation probabilities are:     
\begin{align*}
% &\pi^*(z,X) 
% % \\
% % &
% = 
%  \sqrt{\frac{\sigma^2_z(X)}{  e_z^2(X)}}B
% \left(
% \Eb{ \mathbb{I}[Z=1 ]\sqrt{\frac{\sigma^2_1(X)}{  e_1^2(X)}} 
% + 
% \mathbb{I}[Z=0]\sqrt{\frac{\sigma^2_0(X)}{  e_0^2(X)}} }\right)^{-1} 
&
\textstyle 
\pi^*(z,X) 
% \\
% &
= 
\frac{ \sqrt{\sigma^2_z(X)}}{  e_z(X)}B
\left(
\Eb{ \sqrt{\sigma^2_1(X)} + \sqrt{\sigma^2_0(X)}} \right)^{-1} 
\end{align*} 
Analogously, under \Cref{setting-yytilde}, the optimal $(X,\tilde{Y})$-dependent annotation probabilities are $\pi^*(z,X,\tilde{Y}) \textstyle 
% \\
% &
= \textstyle 
\frac{ \sqrt{\sigma^2_z(X,\tilde{Y})}}{  e_z(X)}B
(
\E[ \sqrt{\sigma^2_1(X,\tilde{Y})} + \sqrt{\sigma^2_0(X,\tilde{Y})}] )^{-1}.$
Otherwise, the optimal solutions are \[ \textstyle 
\pi^{\star}(z,x)=\min\{1,\;c_B\frac{\sqrt{\sigma_z^2(x)}}{e_z(x)}\},\qquad c_B>0\ \text{chosen so that}\ \mathbb{E}\!\left[\pi^{\star}(Z,X)\right]=B.
\]
\end{theorem}

Note that sampling probabilities increase in the conditional variance/uncertainty of the model, $\sigma^2(X)$, and the inverse propensity score. %The latter is specific to the ATE setting, though the solution recovers that of mean estimation, for example, if treatments were uniformly assigned. 
Characterizing the closed-form solution is useful for our analysis later on. 
For the full proof, see \Cref{proof:thm1}.

\paragraph{Feasible two-batch adaptive design and estimator.}
%\angela{Feasible estimation, DML and cross-fitting (CSBAE)}

Although the oracle rule $\pi^*$ in
\Cref{thm-global-budget-solution} is a valid pooled annotation
probability, it may not be attainable after the fixed first-batch
allocation, since any two-batch design must satisfy
$\kappa\pi_1(z,x)\leq\pi(z,x)\leq
\kappa\pi_1(z,x)+(1-\kappa)$.
Following the corner-solution logic of \citet{hahn2011adaptive}, we use
\begin{equation}
\hat{\pi}_{2,z}(x;\hat c)
=
[ (1-\kappa)^{-1}
({
\hat c\sqrt{\hat{\sigma}_z^2(x)}/\hat e_z(x)
-\kappa\pi_1(z,x)})
]_0^1,
\qquad
[u]_0^1:=\min\{1,\max\{0,u\}\} \label{eqn-clippingpi2}
\end{equation}
where $\hat c$ is chosen by bisection so that
$\mathbb E_n[\kappa\pi_1(Z,X)+(1-\kappa)
\hat\pi_2(Z,X;\hat c)]=B$. Let $\pi_{2,z}^\dagger(x)$ denote the population analogue of
\Cref{eqn-clippingpi2}. Define the feasible pooled probabilities
\[
\pi^\dagger(z,x)
=
\kappa\pi_1(z,x)+(1-\kappa)\pi_{2,z}^\dagger(x),
\qquad
\hat\pi^{\dagger,(-k)}(z,x)
=
\kappa\pi_1(z,x)
+
(1-\kappa)\hat\pi_{2,z}^{(-k)}(x;\hat c),
\]
where $\hat\pi_{2,z}^{(-k)}$ is \Cref{eqn-clippingpi2} evaluated
with out-of-fold nuisance estimates and one common $\hat c$ across
folds. If no second-batch clipping bound binds, then
$\pi^\dagger=\pi^*$; otherwise, $\pi^\dagger$ is the
variance-minimizing pooled probability attainable given the
first-batch allocation. In the estimation analysis below, generic $\pi$ and $\hat{\pi}^{(-k)}$ refer to $\pi^{\dagger}$ and $\hat{\pi}^{\dagger,(-k)}$.

% Finding the optimal $\pi^*(z,x)$ as we have 
Our characterizations above assume knowledge of true $\sigma^2_z(x)$ and propensity scores $e_z(x)$. Since these need to be estimated, we leverage the double machine learning (DML) framework and conduct a feasible two-batch adaptive design \citep{chernozhukov2018double,bia2021dmlsampleselection}. %Cross-fitting was recently proposed in \citep{li2024double}, convergent split batch adaptive experiment (CSBAE). 
Standard cross-fitting \citep{chernozhukov2018double} splits the data, estimates nuisance functions on one fold, and evaluates the estimator on a datapoint leveraging nuisance functions from another fold of data. %(The number of folds can vary; $K=2$ to $K=5$ is typical). 
% Cross-fitting guarantees that nuisance function estimates are independent of the observations in fold k. 
% We conduct uniform sampling in a first batch, whose data we use to estimate nuisance functions. We use cross-fitting to obtain the sampling probabilities and to evaluate the final estimator. 
We leverage a variant \citep{li2024csbae} that introduces folds within each batch of data. 
% which is a key argument in ensuring that the feasible estimator $\tau_{AIPW}$ is asymptotically equivalent to $\tau$. 
% Let $D=(X,Z,R,Y)$ and $\mathcal{I}_k$ denote the index set of the $k$th fold. A typical DML estimator with cross-fitting  computes $\hat{\tau}_{AIPW}$ empirically as $\frac 1n \sum_{k=1}^K \sum_{i\in \mathcal{I}_k} \hat{\psi}_1(D_i;\hat{\mu}_1^{(-k)}, \hat{e}_1^{(-k)} ,\hat{\pi}^{(-k)}) - \hat{\psi}_0(D_i;\hat{\mu}_0^{(-k)}, \hat{e}_0^{(-k)},\hat{\pi}^{(-k)})$, where %for each $k=1,\hdots, K$
% $\hat{\mu}_z^{(-k)}, \hat{e}_z^{(-k)},\hat{\pi}^{(-k)}$ are out-of-fold nuisance estimates learned on data outside of fold $k$. 
% \begin{definition}[Cross-fitting for batch adaptive experiment \citep{li2024csbae}]\label{defn-csbae}
    \Cref{fig:cross-fitting} summarizes the cross-fitting approach; see \citet{li2024csbae} for further details.
    \begin{figure*}[ht!]
    \centering
    \includegraphics[width=\textwidth]{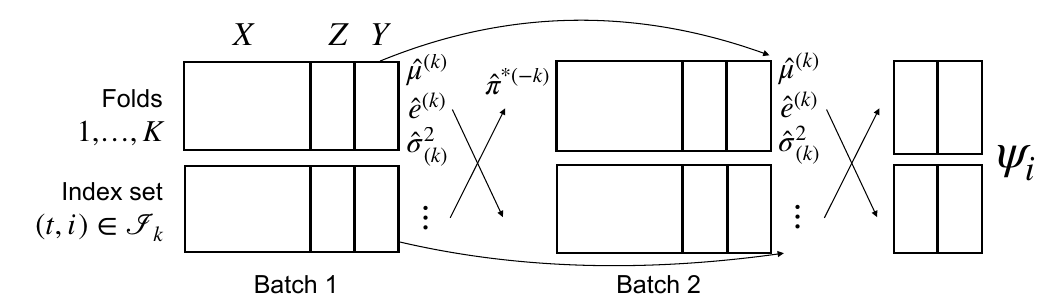}
    \caption{Illustration of cross-fitting ($K$ folds within batches) }
    \label{fig:cross-fitting}
\end{figure*} 
First, we split the observations in each batch $t=1,2$ into $K$ folds (e.g. $K=5$). Let $\mathcal{I}_{k}$ denote the set of batch and observation indices $(t,i)$ assigned to fold $k$ and batch $t$. 
%The number of observations in batch $t$ assigned to fold $k$ is $n_{t,k} = \mid \{(t,i) \in \mathcal{I}_k\mid i = 1,\hdots,n_t \}\mid$. 
Then within each fold, we estimate nuisance models on observations in batch 1. We use cross-fitting to estimate the feasible pooled probabilities
$\hat\pi^{\dagger,(-k)}$ from out-of-fold nuisance estimates. Finally we adaptively assign annotation probabilities in batch 2. This ensures independence, meaning that the nuisance models in batch 2, fold $k$ rely only on observations from the \textit{previous} batch 1, in fold $k$.
% \end{definition}
%in which we can use the first stage to estimate the outcome models $\hat \mu_z(X)$, conditional variances $\hat \sigma^2(X)$, and treatment propensities $e(z,X)$ for each treatment arm, and then we use these estimates to assign selection probabilities for data annotation in the second stage. As a result we are able to actively select data points to be expertly labeled for improved causal estimation rather than uniform sampling \citep{naoki2023dsl}.
\Cref{alg:adaptive} includes this procedure,
%The adaptive algorithm with the convergent split batch adaptive experiment (CSBAE) cross-fitting procedure to estimate $\tau_{AIPW}$ is summarized in \Cref{alg:adaptive}. 
see \Cref{alg:batch-allocation-full} for a full description. 

\textbf{Batch Adaptive Causal Estimation (With Complex Embedded Outcomes)} \label{alg:adaptive}. \textit{Input:} Data $\mathcal{D} = \{(X_i,Z_i)\}_{i=1}^n$, sampling budget $B \in [0,1]$.
\begin{itemize}[noitemsep,topsep=2pt]
  \item Partition $\mathcal{D}$ into 2 batches and $K$ folds $\mathcal{D}^{(k)}_{1},\mathcal{D}^{(k)}_{2}$ for $k=1,\hdots,K$.
  \item \textbf{Batch 1:} Sample $R_1 \sim \mathrm{Bern}(B)$. Estimate nuisances within each $k$-fold: $\hat{\mu}_z^{(k)}(X)$ (or $\hat{\mu}_z^{(k)}(X,\tilde{Y})$), $\hat{\sigma}^{2(k)}_z(X)$, and $\hat{e}^{(k)}_z(X)$.
  \item \textbf{Batch 2,} folds $k=1,\hdots,K$: Evaluate
\Cref{eqn-clippingpi2} with out-of-fold nuisance estimates to obtain
$\hat\pi_2^{(-k)}$, using one common $\hat c$ chosen to satisfy the
pooled budget, and form
$\hat\pi^{\dagger,(-k)}
=\kappa\pi_1+(1-\kappa)\hat\pi_2^{(-k)}$.
  %$\hat{\pi}^{(k)}_2 (X_i) = \frac{1}{1-\kappa} (\pi^{*}(X_i) - \kappa \pi_1)$.
  \item On Batch 2, sample
$R_{2i}\sim
\mathrm{Bern}(\hat\pi_2^{(-k)}(Z_i,X_i;\hat c))$
and obtain outcomes.
  \item Pool data across batches and estimate ATE with AIPW estimator in \cref{eq:AIPW-estimator} (or \cref{eqn-RZ-estimation}, or balancing weights) and out-of-fold nuisances.
\end{itemize}

\begin{algorithm}[ht!]
   \caption{(Full Algorithm) Batch Adaptive Causal Estimation With Complex Embedded Outcomes}
   \label{alg:batch-allocation-full}
 \begin{algorithmic}
    \STATE {\bfseries Input:} Data $\mathcal{D} = \{(X_i,Z_i,Y_i,\tilde{Y}_i)\}_{i=1}^n$, sampling budget $B_z$ for $z\in \{0,1\}$ 
    \STATE {\bfseries Output:} ATE estimator $\hat{\tau}_{AIPW}$
    \STATE Partition $\mathcal{D}$ into 2 batches and K folds $\mathcal{D}^{(k)}_{1},\mathcal{D}^{(k)}_{2}$ for $k=1,\hdots,K$
    \STATE \textit{ Batch 1:} 
    \FOR{$k=1,\hdots,K$}
    \STATE
    On $\mathcal{D}^{(k)}_{1}$:  Sample $R_1 \sim \mathrm{Bern}(\pi_1(Z,X))$, where $\pi_1(z,x) = B_z$. 

    Estimate nuisance models:  Where $R=1$, estimate $\hat{\mu}^{(k)}_z$ by regressing $Y$ on $X$ (or $X,\tilde{Y}$), and $\hat{\sigma}^{2(k)}_z$ by regressing $(Y-\hat{\mu}_z)^2$ on $X$. Estimate $\hat{e}^{(k)}_z$ by regressing $Z$ on $X$.
        % \FOR{$\mathcal{D}^{(k)}_{1}$} 
            % \STATE
            % \STATE Train $\hat{\mu}^{(k)}_z$ by regressing $Y$ on $(R_1=1,X,\tilde{Y})$
            % \STATE Train $\hat{\sigma}^{2(k)}_z$ by regressing $(Y-\hat{\mu}_z)^2$ on $(R_1=1,X)$ 
            % \STATE Train $\hat{e}^{(k)}_z$ by regressing $Z$ on $X$
        % \ENDFOR
    \ENDFOR
      \STATE \textit{Batch 2:}
        \FOR{$k=1,\hdots,K$}
            \STATE On $\mathcal{D}^{(k)}_{2}$, compute the out-of-fold
            quantities entering \Cref{eqn-clippingpi2} using
            $\hat{\sigma}^{2(-k)}_z$ and $\hat e_z^{(-k)}$.
        \ENDFOR
        \STATE Choose one common $\hat c$ by bisection so that
        $\mathbb E_n[\kappa\pi_1(Z,X)
        +(1-\kappa)\hat\pi_2(Z,X;\hat c)]=B$.
        \FOR{$k=1,\hdots,K$}
            \STATE Set $\hat\pi_2^{(-k)}(Z_i,X_i;\hat c)$ according to
            \Cref{eqn-clippingpi2} and sample
            $R_{2i}\sim
            \mathrm{Bern}\!\left(
            \hat\pi_2^{(-k)}(Z_i,X_i;\hat c)
            \right)$.
        \ENDFOR
    \STATE Obtain $\mathcal{D}^{(k)}$ for $k=1,\hdots,K$ by pooling across batches $\mathcal{D}^{(k)}_1$ and $\mathcal{D}^{(k)}_2$
    \STATE On $\mathcal{D}^{(k)}$, re-estimate $\hat{\mu}^{(k)}_z$, $\hat{\sigma}^{2(k)}_z$, and $\hat{e}^{(k)}_z$  on observed outcomes $RY$ for $k=1,\hdots,K$
    \STATE For each fold, form the stored feasible pooled probability
    $\hat\pi^{\dagger,(-k)}(z,x)
    =
    \kappa\pi_1(z,x)
    +
    (1-\kappa)\hat\pi_{2,z}^{(-k)}(x;\hat c)$.
    \STATE On full data $\mathcal{D}$, estimate ATE by using AIPW estimator in \cref{eq:AIPW-estimator} and out-of-fold nuisances $\hat\pi^{\dagger,(-k)}$, $\hat{\mu}^{(-k)}_z$, $\hat{\sigma}^{2(-k)}_z$, and $\hat{e}^{(-k)}_z$
%    \STATE Initialize $noChange = true$.
%    \FOR{$i=1$ {\bfseries to} $m-1$}
%    \IF{$x_i > x_{i+1}$}
%    \STATE Swap $x_i$ and $x_{i+1}$
%    \STATE $noChange = false$
%    \ENDIF
%    \ENDFOR
%    \UNTIL{$noChange$ is $true$}
 \end{algorithmic}
\end{algorithm}

Therefore the cross-fitted feasible estimator takes the form $\hat{\tau}_{AIPW} 
% &
    \textstyle = {\frac{1}{n}\sum_{t=1}^2 \sum_{k=1}^K\sum_{(t,i)\in\mathcal{I}_k} \hat{\psi}_{1,i} - \hat{\psi}_{0,i}}$ where 
\begin{align}
    % \hat{\tau}_{AIPW} &
    % \textstyle = \frac{1}{n}\sum_{t=1}^2 \sum_{k=1}^K\sum_{(t,i)\in\mathcal{I}_k} \hat{\psi}_{1,i} - \hat{\psi}_{0,i}, 
    % \nonumber \\
    % &
    % \text{ where } 
    \hat{\psi}_{z,i}=\frac{\mathbf{1}[Z_i=z]R_i(Y_i-\hat{\mu}^{(-k)}_z(X_i))}{\hat{e}^{(-k)}_z(X_i)  \hat{\pi}^{\dagger,(-k)}(z,X_i)}
    +\hat{\mu}^{(-k)}_z(X_i).\label{eq:AIPW-estimator}
\end{align}
\section{Analysis}\label{sec-analysis}

In this section, we provide a central limit theorem for the setting where annotation probabilities are assigned adaptively and nuisance parameters must be estimated. We provide some insights to improve estimation as well as an extension to settings with continuous treatments.   

Denote  $\norm{\cdot}_2 = (\E[(\cdot)^2])^{1/2}$. The following \Cref{asn-consistent-bounded,asn-product-error-rates,asn-vc} are all standard in the double machine learning literature %can also be found in 
\citep{chernozhukov2018double,wager2024causal,athey2021policy,Uehara2020policyeval,bia2021dmlsampleselection}. \Cref{asn-weak-dependence-batches} is specific to our batch adaptive sampling design and can also be found in \cite{li2024csbae}.

\begin{assumption}[Consistent estimation and boundedness]\label{asn-consistent-bounded}
Assume bounded second moments of outcomes and errors, $\norm{Y(z)}_2
% (\E[Y(z)^2])^{1/2} 
\leq C_1,
\norm{\mu_z(X)}_2
% (\E[\mu_z(X)^2])^{1/2} 
\leq C_2$, $
\norm{(Y-\mu_z(X))}_2^2
% \E[(Y-\mu_z(X))^2] 
\leq 4B_{\sigma^2}$, $\forall z$; and consistent estimation $\E[(\mu_z(X) - \hat\mu_z(X))^2]\leq K_\mu n^{-r_\mu}$ for some constants $C_1,C_2,B_{\sigma^2}, K_\mu, r_\mu > 0$. 
\end{assumption}
\begin{assumption}[Consistency and product error rates] For nuisance functions, assume the products of their mean-square convergence rates vanish faster than $n^{-1/2}$: (i) $\sqrt{n}\norm{\hat{\mu}_z(X) - \mu_z(X)}_2 \times \norm{\hat{\pi}(z,X) - \pi(z,X)}_2 \overset{p}{\to} 0$; (ii) $\sqrt{n}\norm{\hat{\mu}_z(X) - \mu_z(X)}_2 \times \norm{\hat{e}_z(X) - e_z(X)}_2 \overset{p}{\to} 0$. For each $z$, assume that $\norm{\hat e_z-e_z}_2=o_p(1),\norm{\hat\sigma_z^2-\sigma_z^2}_2=o_p(1)$.
\label{asn-product-error-rates}
\end{assumption}

\begin{assumption}[VC dimension for nuisance estimation]\label{asn-vc}
The nuisance estimation of $e_z$ and $\sigma_z^2$ occurs over function classes with finite VC-dimension. 
\end{assumption}

\begin{assumption}[Sufficiently weak dependence across batches \citep{li2024csbae}]\label{asn-weak-dependence-batches}
\begin{align*}
    \sqrt{\frac{1}{n_{t, k}} \sum_{i:(t, i) \in\mathcal{I}_k}\left\|\mathbb{E}\left[\hat\psi_i(R;  \hat\eta) -\psi_i(R;\eta) \mid \mathcal{I}^{(-k)}, X_{i}\right]\right\|^2} 
    =o_p(n^{-\frac 14}),
\end{align*}
where $\hat\eta$ is the vector of nuisance functions $\hat{e}^{(-k)},\hat{\pi}^{(-k)}, \hat\mu^{(-k)}$ and $\eta$ is the vector of true population nuisance functions. Then $\hat\psi_i(R; \hat\eta) = \psi_i(R;  \hat{e}^{(-k)},\hat{\pi}^{(-k)}, \hat\mu^{(-k)})$ and $\psi_i(R;\eta) = \psi_i(R;  {e},{\pi}, \mu)$.
\end{assumption}

\begin{theorem}[Asymptotic Normality]\label{thm-asymptotic-convergence} Given \Cref{asn-r-ignorability,asn-tx-ignorability,asn-positivity},
%and \Cref{lemma-convergence-nuissance}, % don't need the explicit assumption of the lemma
suppose that we construct the feasible estimator $\hat{\tau}_{AIPW}$ (\Cref{eq:AIPW-estimator}) using the CSBAE crossfitting procedure in \Cref{fig:cross-fitting} with estimators satisfying \Cref{asn-consistent-bounded,asn-product-error-rates,asn-vc,asn-weak-dependence-batches} (consistency and product error rates). Then
% and also the roles of $\mathcal{I}_{k}$ and $\mathcal{I}_{(-k)}$ for K folds across T batches, 
\[ \sqrt{n}(\hat{\tau}_{AIPW} - \tau)\Rightarrow \mathcal{N}(0,V_{AIPW}),\]
where $\tau$ is the ATE and $V_{AIPW}$ is
\begin{align*}
&\sum_{z\in\{0,1\}}\E\left[\frac{\sigma^2_z(X)}{e_z(X)\pi^\dagger(z,X)}\right] + \mathrm{Var}\left[\mu_1(X)-\mu_{0}(X)\right].
\end{align*} %Here $\tau$ is the ATE.%oracle computed on a nonadaptive batch experiment.  
\end{theorem}
\Cref{thm-asymptotic-convergence} shows that the batch adaptive design and feasible estimator have an asymptotic variance equal to the variance of the true ATE under missing outcomes and the
optimal feasible pooled probability $\pi^\dagger$. %This implies that our procedure successfully minimizes the asymptotic variance bound. This 
Therefore, our procedure gives asymptotically valid level-$\alpha$
confidence intervals for $\tau$ of minimum width among pooled
probabilities attainable after the first-batch allocation.
% With this, we can also quantify the uncertainty of our treatment effect estimates by producing level-$\alpha$ confidence intervals for $\tau$ that achieve coverage with $1-\alpha$ probability.  
The proof of \Cref{thm-asymptotic-convergence} proceeds in two steps. The first step establishes that the feasible AIPW estimator converges to the AIPW estimator with oracle nuisances. Next we show that the oracle estimator with feasible nuisances converges to the same estimator with oracle nuisance functions. Together, with our convergence and product error rate assumptions, we have that our feasible AIPW estimator converges to the oracle (see full proof in \Cref{appendix:proofs}). 

\textbf{Insights and improvements}

\textbf{1) When is our method much better than uniform sampling?} Prior works of \citep{naoki2023dsl,zrnic2024active_statistical}, though they do not study treatment effect estimation, obtain valid inference with uniform sampling (i.e. with the budget probability). When do optimized data annotation probabilities improve upon uniform sampling? To answer this, we analyze the relative efficiency ($\mathrm{RelEff}$) which compares the asymptotic variance $(\mathrm{AVar})$ under optimized or uniform sampling, for the same budget. 
The following closed-form comparison concerns the unrestricted oracle
$\pi^*$; when clipping binds, the corresponding relative efficiency
uses $\pi^\dagger$ and is evaluated numerically.
\begin{corollary}[Relative efficiency]\label{corollary-rel-eff}
    %The relative efficiency of estimation with optimized sampling probabilities $\pi$ vs. uniform sampling, for the same budget, is
\begin{align*}
    \mathrm{RelEff} &= \frac{\mathrm{AVar} \textrm{ of estimation with } \pi^*}{\mathrm{AVar} \textrm{ of estimation with uniform prob. } B} 
    % \\ 
    % &
    =  \frac{
\frac{1}{B}
\left(\Eb{ \sqrt{\sigma^2_1(X)}+\sqrt{\sigma^2_0(X)}
}\right)^2 
+ \mathrm{Var}[\tau(X)]
}{
\frac{1}{B}\Eb{ \frac{\sigma^2_1(X)}{e_1(X)} +\frac{\sigma^2_0(X)}{e_0(X)}}
+ \mathrm{Var}[\tau(X)]
}
\end{align*}
\end{corollary}
By construction, $\mathrm{RelEff} \leq 1$; the smaller it is, the larger the improvement from our method. Our method's improvement increases if the budget is smaller $(B\downarrow)$ or if there are imbalanced propensities where $e_1(X)$ is close to $0$ or $1$. Improvements shrink for large budgets or when treatment variances are similar.

\textbf{2) Direct estimation of $(e\pi^*)^{-1}$ mitigates estimation stability.}
It is well known that estimating propensities and then inverting estimates can be unstable in practice. This problem is doubly-so for causal inference with missing outcomes. We find many papers on adaptive treatment allocation note this challenge and mix their optimized allocation probabilities with uniform in the experimental sections \citep{dimakopoulou2021online,zrnic2024active_statistical,cook2024semiparametric}; just as many papers in causal inference  clip the weights in practice \citep{wang2017optimal}. Our closed-form solution reveals that estimating propensity scores for the \textit{final} ATE estimation on the full dataset is \textit{fundamentally unnecessary}, though it is needed to estimate $\pi^*$. At $\pi^*$, observe that\footnote{This depends on some joint properties of $\kappa, p_1$, whether it is feasible to find second-stage batch sampling probabilities $\pi_2$ so that $\kappa p_1 + (1-\kappa) \pi_2(x)=\pi^*(x)$} 
% ${\psi_z(e,\pi^*) %&
%     = \frac{\mathbb{I}[Z=z]}{\sqrt{\sigma^2}(X)} R  (Y-\mu_z(X))   + \mu_z(X)}.
%     $
$(e_z(x)\pi^*(z,x))^{-1} \propto \sqrt{\sigma^2_z}(x)^{-1}$ and is \textit{independent of the propensity score} $e_z(x)$. Therefore estimating the optimal inverse propensity function directly can exploit its \textit{lower} statistical complexity. In causal inference and covariate shift, many methods (such as balancing weights) avoid the plug-in approach for inverse propensity methods in favor of direct estimation of the inverse propensity score \citep{tsuboi2009direct,zubizarreta2015stable,imai2014covariate,kallus2018balanced,kallus2018optimal,cohn2023balancing,bruns2025augmented}. We recommend estimation on the final dataset with such approaches or other types of direct estimation. For example, even estimation of $P(Z=z,R=1\mid X)$ directly helps:
\begin{equation}
{\psi_z(e,\pi^*)
    = \frac{\mathbb{I}[Z=z,R=1]}{P(Z=z,R=1\mid X)}   (Y-\mu_z(X))   + \mu_z(X)} \qquad \text{($RZ$-plug-in).}
    \label{eqn-RZ-estimation}
\end{equation}

% \textbf{3) When can we achieve consistent estimation with $(X,\tilde{Y})$-dependent annotation probabilities? Ensembling and an oracle inequality.}
% Our convergence results rely on 

\subsection{Extension to continuous treatments\label{appendix:additional-results-cts-ext}}

%\paragraph{Setup for continuous treatments. } 

Our analysis applies readily to other static causal inference estimands, such as those for continuous treatments. We characterize the optimal sampling probabilities. In the continuous setting, consider estimation of a counterfactual mean: $\mathbb{E}[Y(z)].$ (We can extend to contrasts for different values of treatment, in analogy to the ATE). Let $(Y_i,X_i,Z_i)^n_{i=1}$ be an i.i.d. sample from $Q = (Y,X,Z) \in \mathcal{Q} = \mathcal{Y} \times \mathcal{X} \times \mathcal{Z}_0 \subseteq \mathcal{R}^{1\times d_x \times 1}$, i.e. consider a univariate continuous treatment $Z \in \mathcal{Z}_0$. This can extend to the case of multiple continuous treatments $d_Z$ but for ease of mathematical computation, we start with the one-dimensional continuous treatment setting. We derive the form of the asymptotic variance as well as the bias term for an estimator for continuous treatments with missing outcomes. 

We introduce an estimator for continuous treatments with missing outcomes that is a direct extension of \citep{kallus2018policy,colangelo2020double}, while building on the Riesz representer characterization of \citep{klosin2021automatic}'s automatic double machine learning estimator for continuous treatment effects. We introduce what we call the  ``partial" Riesz representer, $\alpha(z, X)=\frac{1}{P(Z=z \mid X)}$ which is the inverse generalized propensity score or the balancing function for treatment alone. (We term it ``partial" since we are optimizing over the $\pi(z,x)$ missingness probabilities in the denominator). We introduce the partial Riesz representer following our earlier insight as to the improved finite-sample performance of using balancing weights estimators on the final collected data. We also introduce $\bar{\alpha}$ to account for mispecification of the nuisance function. Under the correct specification of this nuisance function, $\bar{\alpha}$ = $\alpha$.

The following estimator for continuous treatments with missing outcomes is a direct extension of \citep{kallus2018policy,colangelo2020double}, that replaces the indicator function $\mathbb{I}[Z=z]$ with a local kernel function smoother localizing around $z$ with bandwidth $h$, $K_h(Z-z)$:
\begin{align}
    &E[ \psi_z(\alpha,\mu) ]= \mathbb{E}[Y(z)]+O(h^2), \\
 &\text{ where }  \psi_z(\alpha,\mu) = 
 \mu\left(z, X_i\right)+\frac{K_h\left(Z_i-z\right)\mathbb{I}[R=1]\alpha\left(z,X_i\right)}{\pi(z,X_i)}\left(Y_i-\mu\left(z, X_i\right)\right), \text{ and }
 \alpha(z,x) = \frac{1}{f_{Z|X}(z|x)} \nonumber
\end{align}
% where, 
% \begin{equation}
% % \textstyle
%  \psi_z(\alpha,\mu) = 
%  \mu\left(z, X_i\right)+\frac{K_h\left(Z_i-z\right)\mathbb{I}[R=1]\alpha\left(z,X_i\right)}{\pi(z,X_i)}\left(Y_i-\mu\left(z, X_i\right)\right).
%  \end{equation}

% and 
% $$\alpha(z,x) = \frac{1}{f_{Z|X}(z|x)}.$$

Here $f_{Z|X}(z|x)$ is defined as conditional probability density of treatment given covariates and later we will use $f_{ZX}(z,x)$ to refer to the joint distribution between treatments and covariates.

%\az{Asymptotic variance proposition}

Following our analysis in the binary treatment setting, we derive the asymptotic variance of this estimator. In the continuous treatment setting, the asymptotic variance does incur bias and we derive the expressions of both the variance and bias terms in the following proposition. 

\begin{proposition}
\label{prop:cts-trt-avar}
The integrated mean squared error for the continuous treatment setting is: 
\begin{align*}
MSE &= V_z  + B_z, \\
\text{where } V_z &\equiv h^{-1} \Eb{\frac{\bar{\alpha}^2(z,x)}{\pi(z,x)}f_{Z\mid X}(z\mid x)\Eb{(Y-\mu(z,x))^2\mid Z=z,X=x}} \xi_k,\\
\text{ and }B_z
& \equiv h^4 \Bigg(\bigg[ 2 \frac{d}{d z} \bar{\mu}(z,X)\frac{d}{d z} f_{Z|X}(z|x) + f_{Z|X}(z|x)\frac{d^2}{d z} \bar{\mu}(z,X) + (\bar{\mu}(z,X) - \mu(z,X))\frac{d^2}{d z}f_{Z|X}(z|x)\bigg]\kappa\Bigg)^2. 
\end{align*}
\end{proposition}

Most notably, we see that the bias term does not depend on $\pi(z,x)$. Therefore, we can focus our optimization on $V_z$ with respect to $\pi(z,x)$. 

For this optimization procedure, we consider the same assumptions required as in \citep{colangelo2020double}, standard in kernel density estimation analysis such as sufficient smoothness of the underlying function and kernel function, and rate conditions $h \rightarrow 0, n h\rightarrow \infty$, $n h^{4} \rightarrow C \in[0, \infty)$. Suppose that $\alpha(z,X)$ is well-specified. Let $\sigma^2(z,x)=\mathbb{E}\left[(Y-\mu(z, X))^2 \mid Z=z, X=x\right]$. We need to optimize the expression for variance that explicitly has the integration over $K_h$. The objective function arises from the asymptotic variance expression in \citep[Thm. 3]{colangelo2020double}; it follows readily from following their proof of Thm. 3 with our analysis of the asymptotic variance as in \Cref{prop-avar-ate}. The proof of the optimal solution follows our analysis in \Cref{thm-global-budget-solution} with a few slightly different expressions. The optimization problem can be written as follows:
% The asymptotic variance is $\mathrm{AVar}_z^{\mathrm{cts}} \equiv \mathbb{E}\left[
%\frac{ \sigma^2(z,X)}{{e}(z, X)\pi(z,X)}\right] R_k^{d_T}$. 
%The optimal sampling probabilities minimize the part of the asymptotic variance of $\E[Y(z)]$ depending on $\pi$, subject to a budget constraint: 
%$$
% \pi^*( z,x)\in \arg\min_{\pi(z,x)} 
%\left\{ \mathbb{E}\left[\frac{\sigma^2(z,X)}{e(z,X) \pi(z,X)}\right] \right\}
%$$ 
%$$
%\text{ s.t. } \E[\pi(z,X) K_h(Z-z)] \leq B. \;\;
%$$
%Further let $\mathbb{E}\left[|Y-\bar{\gamma}(T, X)|^3 \mid T=t, X\right]$ and its derivatives with respect to $t$ be bounded uniformly over $\left(t^{\prime}, x^{\prime}\right)^{\prime} \in \mathcal{T} \times \mathcal{X}$. Let $\int_{-\infty}^{\infty} k(u)^3 d u<\infty$. 
% - else simplifying the integration over $Kh$ simplifies ``too far". It accurately describes the estimation variance for $E[Y(z)]$, that is driven by $\sigma^2(z,x)$ but doesn't take into account the $\sigma^2(s,x)$ for $s\neq z$ but near to it. 
$$
\pi^*(z,x) \in \arg\min_{\pi(z,x)}\int_{\mathcal{X}} \int_{Z_0} \frac{K_h^2(s-z) \bar{\alpha}^2(s, x) }{\pi(s, x)} \sigma^2(s,x) f_{Z X}(s, x) d s d x \text{ s.t. } \E[\pi(z,X)] \leq B. \;\;
$$

% $$
% %\text{ s.t. } \E[\pi(z,X)] \leq B. \;\;
% $$ 

The closed-form solution for the optimal annotation probability for the continuous treatments case is:
\begin{equation}
\label{eqn-cts-tx-optimal}
\pi^*(z,X) = \frac{K_h(Z-z)\bar{\alpha}(z,X)\sqrt{\sigma^2(z,X)}}{\Eb{K_h(Z-z)\sqrt{\bar{\alpha}^2(Z,X)\sigma^2(Z,X)}}} B.
\end{equation}

Crucially, note that the optimal sampling probability is similar to the binary-treatment solution \Cref{thm-global-budget-solution}, with analogous outcome conditional variance and inverse propensity, but additional localization around the treatment z. Therefore, key insights carry over to the continuous setting. 
\section{Experiments: Methodological validation}\label{sec-experiments}

We evaluate our batch adaptive allocation protocol on synthetic and real-world datasets. We show that our method enables consistent and efficient ATE estimation even under limited labeling budgets, ultimately helping resource-constrained organizations obtain reliable estimates from their data.  
\subsection{Adaptive causal annotation results} 
\paragraph{Baselines.} Across all experimental setups, we compare against completely random sampling for AIPW, and evaluate MSE relative to an oracle full-data skyline (infeasible in practice). % a method with additional information than ours that nonetheless informs of best (infeasible) performance.
% The baseline %does not use our adaptively learned $\hat{\pi}(z,x)$, but instead 
% simply uses uniform random sampling at different budget values. 
We compare MSE to the skyline of the standard AIPW estimator with fully observed outcomes, that is when the budget equals 1 or $R=1$ for all data points. In our setting, completely random sampling for AIPW is the strong baseline, because AIPW is optimal causal estimation for the ATE. Other more complicated methods target other objectives. % instead of the objective in active learning, outcome prediction error. 
See \Cref{appendix-al-baselines} for more detailed discussion and experimental results including pool-based active learning baselines (which do worse). %We allocate $55\%$ and $45\%$ of the data to batch 1 and 2. 

%We compute standard errors for each method. We report the distribution of ATE estimates and the average interval width (on the log scale) for varying budgets, averaged over $n=20$ and $n=100$ trials (see Appendix). We report the empirical MSEs for the ATE $\tau$; that is, $\frac{1}{n}\sum_{i=1}^{100} (\tau - \hat{\tau}_{AIPW}^{(i)})^2$, where $\hat{\tau}_{AIPW}^{(i)}$ denotes the estimator of the ATE in the $i$-th trial.  
\begin{figure*}[t!]
    \centering
    % Figures from figs/ (sibling of for-aistats2026)
    \includegraphics[width=\textwidth]{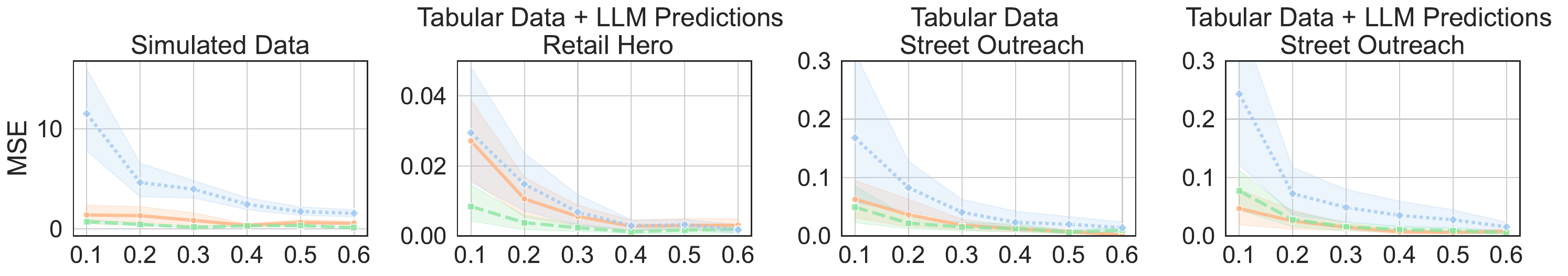}

    \hspace{-0.5cm}
    \includegraphics[width=\textwidth]{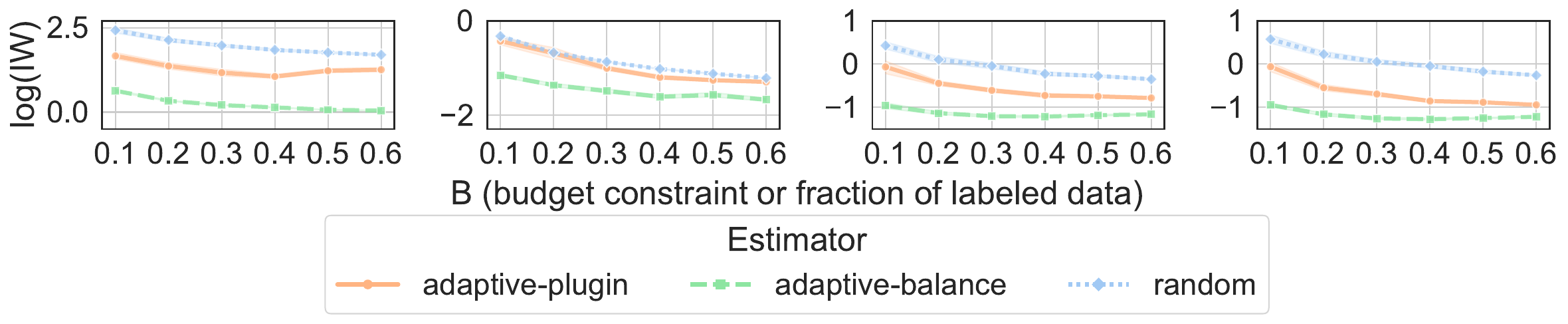}
    \caption{\textbf{Experiments on synthetic data (leftmost), Retail Hero (center left), and Street Outreach data (center right and rightmost).} Results of performance measure mean squared error (top) and $95\%$ confidence interval width on the log scale (bottom) averaged over 20 and 100 (for simulated data) trials across budget percentages of the data. For tabular data experiments, we use random forest prediction on tabular data alone (center right). For tabular data including LLM predictions on text, we use LLM predictions and serialized features as covariates into model (center left and rightmost).}
    \label{fig:retailhero_lineplots}
\end{figure*} 

%Any other sampling strategy in our two-stage framework with AIPW performs suboptimally (since we prove that ours is optimal). %Additionally, the baselines used in related papers are either random sampling or the exclusion of model-based predictions (i.e., $\hat{\mu}$ or $\tilde{Y}$). However, because our task is inherently causal, our AIPW estimator relies on $\hat\mu$.

\paragraph{Simulated Data.} For our simulation study, we generate synthetic data following the data generation process defined in \Cref{appendix-synthetic-exp}. The synthetic data does not include $\tilde{Y}$, but best showcases the utility of our batch adaptive procedure for data annotation under a labeling budget. The leftmost plot of \Cref{fig:retailhero_lineplots} shows that our approach achieves the greatest percentage gains at the smallest budgets $0.1 - 0.3$ with $71\%-95\%$ average percentage gain. We also see a large reduction in confidence interval width on the log scale.

\paragraph{Retail Hero Data.} We study a semi-synthetic dataset, RetailHero \citep{retailhero}, augmented by  \citet{dhawanend} to include outcomes recorded in text. The %RetailHero \citep{retailhero} 
dataset contains background customer information $X$, treatment $Z$ as a text message ad sent to the customer, and outcomes $Y$ of whether the customer made a purchase or not. 
\begin{table}[h!]
    \centering
    \small
    \renewcommand{\arraystretch}{1.2}
    \begin{tabular}{%
        p{0.2\textwidth}%
        p{0.50\textwidth}%
        p{0.3\textwidth}%
    }
        \hline
        \textbf{Variable} & \textbf{Description} & \textbf{Discrete Category} \\
        \hline
        \multicolumn{3}{l}{\textbf{Outcome}} \\[0.2em]
        Purchase&  whether a customer purchased a product &[Yes,No]\\[0.4em]
        \multicolumn{3}{l}{\textbf{Treatment}} \\[0.2em] 
        SMS communication & whether a text was sent to encourage customer to continue shopping & [Yes, No] \\[0.4em]
        \multicolumn{3}{l}{\textbf{Covariates}} \\[0.2em]
        avg. purchase & avg. purchase value per transaction & [1-263, 264-396, 397-611, $>$ 612]\\
        avg. product quantity & avg. number of products bought & [$\leq$ 7, $>$ 7]\\
        avg. points received & avg. number of points received & [$\leq$ 5, $>$ 5]\\
        num transactions & total number of transactions so far & [$\leq$ 8, 9 - 15, 16 - 27,  $>$ 28]\\
        age & age of user & [$\leq$ 45, $>$ 45]\\
        \hline
    \end{tabular}
    \caption{Covariate, treatment, and outcome descriptions and discrete category definitions for RetailHero dataset.}
    \label{tab:dataset-descriptions-retailhero}
\end{table}
\citet{dhawanend} 
%documents in detail how this dataset was augmented to include textual information. In summary,  
sampled datapoints according to an artificial propensity score and generated text from the binary outcomes prompting LLMs  to generate social media posts following personas (given covariates) (details in \Cref{appendix:real-data-exp}). These text posts are $\tilde{Y}$. The goal is to estimate the causal effect of SMS communication on purchase. 
This illustrates our contextual setting: abundant social media posts provide noisy signals, but only limited validation is feasible.
% This is an example of our contextual setting, where plentiful social media posts can offer insights into customer behaviors 
% Social media posts can offer insights into a customer's opinion about products and purchase propensity, that is they embedded outcome information in a non-standard manner. 
% While obtaining the posts might be easy, this additional information can be noisy and extremely expensive to extract any relevant information. 
% but companies may only be able to allocate a fixed amount of resources for ground-truth validation. 
%We show that by training a machine learning model on tabular data or additionally including LLM predictions from the text-embedded outcomes, we can allocate the annotation budget for more precise causal estimation.
% looks redudant
% in a way that achieves 
% for 
% better causal estimation with higher precision than uniform allocation. 

% To test our proposed method, we randomly split the data into two batches. 
We implement our proposed methods using 1) random forest models to estimate the outcome model $\hat{\mu}_z(X)$ (center, \Cref{fig:retailhero_lineplots}) and 2) a data-driven ensembling\footnote{We estimate the outcome model $\hat{\mu}_z(X,\tilde{Y})$ by ensembling, taking  a weighted average between $\E[Y|X]$ (random forest) and $\E[Y|X,f(X,\tilde{Y})]$ (support vector machine), choosing the best models and weights to minimize the MSE of predicting $Y$ on $20\%$ of the full data.} of $\mu_z(X)$ and $\hat{\mu}_z(X,\tilde{Y})$, where the latter includes zero-shot LLM predictions $f_z(X,\tilde{Y})$ (using Llama-70B) as a covariate (right, \Cref{fig:retailhero_lineplots}). For $f_z(X,\tilde{Y})$, to save computational cost and time\footnote{More detail on models and runtime in \Cref{appendix-data}.}, we cached a set of five LLM predictions for each data point offline that we then sampled from in our experiments.
%of purchase from social media posts $\tilde{Y}$ and then using them as predictors in a random forest to estimate $f(X,\tilde{Y})$ (We run the LLM predictions offline in batch to save cost and time). 
%Then we estimate the outcome model $\hat{\mu}$ by ensembling, taking  a weighted average between $\E[Y|X]$ (random forest) and $\E[Y|X,f(X,\tilde{Y})]$ (support vector machine), choosing the best models and weights to minimize the MSE of predicting $Y$ on $20\%$ of the full data. 
We average the results over 20 random data splits. We compute the AIPW estimator on all available data as a stand-in for ground-truth. (The dataset was too small for a separate held-out validation set). %Ideally, we would have liked to compare to an ATE computed on a validation set; however, due to the small size of the data, this estimate turned out to be too dependent on the split and thus unreliable. 
%We have further experiments with fully synthetic data to validate these results (see left, \Cref{fig:retailhero_lineplots} and find more details on the data-generating process in \Cref{appendix-synthetic-exp}).    

Figure \ref{fig:retailhero_lineplots} shows the improved performance of our adaptive estimator either with a direct estimation of $(e\pi^*)^{-1}$ using logistic regression that we plug in (following \Cref{eqn-RZ-estimation}) or a random forest-based estimator of $(e\pi^*)^{-1}$ extracted from ForestRiesz \citep{chernozhukov2022riesznet}, a random forest-based method to learn balancing weights, over the random sampling baseline. At budget value $\mathrm{B}=0.1$, our batch adaptive procedure with plug-in and balancing weights achieves a $77\%$ and $85\%$ average percentage gain (respectively) in MSE over random sampling, while at $\mathrm{B}=0.4$ we see a $\sim 73\%$ percentage gain for both estimators. %reduces the MSE by almost double and 
%They also reduce the confidence interval widths on average across budgets by $\sim 6\%$ and $\sim 35\%$ (respectively) on the log scale. 
\Cref{fig:retail_hero_budget_saved} in \Cref{appendix:budget-saved-plots} shows the impact of our approach most clearly when we compute the percentage of the budget saved to reach the same interval width. We observe a minimum budget saved of $\sim10\%$ with the adaptive plug-in estimator and $\sim45\%$ with the adaptive balance estimator on tabular data. The LLM prediction we generate is based on simple zero-shot learning and direct serialization of the tabular data; further fine-tuning could improve performance. Nonetheless, our method can provide robust valid guardrails around statistical inference using these black-box predictions. %We see that for higher budgets, random sometimes proves to be much more cost effective when incorporating LLM predictions. We hypothesize that this is because the LLM predictions for this dataset did not prove to be much more informative over the tabular data alone when assessing the performance of each model. 

\paragraph{Street Outreach Data.} Next, we demonstrate our method on street outreach casenote data collected by a partnering nonprofit providing homelessness services. This analysis  
was approved by the Institutional Review Boards at [blinded for review].%UC Berkeley, Fordham, and USC. 

The covariate data $X$ consists of baseline characteristics on each client as tabular data (center right, \Cref{fig:retailhero_lineplots}), such as the number of previous outreach engagements, and (rightmost, \Cref{fig:retailhero_lineplots}) LLM generated summaries of case notes recorded before treatment. We construct the cohort in our dataset to include clients who are seen consistently at least once per 6 months from 2019-2021. (Being on caseload requires a more intensive frequency of being seen at least three times per month, so our cohort definition does not leak post-treatment information). The binary treatment $Z$ was based on the number of outreach engagements within the first 6 months of 2019. Clients with 1-2 engagements were assigned $Z=0$ (131 clients), and those with 3-15 were assigned $Z=1$ (355 clients). The outcome $Y$ can take on values in $\{0,1,2,3\}$, where $0$ indicates that a client is still on the streets and $3$ indicates that a client has found permanent housing. $Y$ is the highest housing placement reached by 2021. Our final data set contained $471$ clients. %of whom we have complete-case data. 
More information on the data can be found in \Cref{appendix:real-data-exp}. % We seek to estimate the causal effect of street outreach on housing placement. 
We use housing placement as an illustrative example because the ground truth data is available in our dataset. However, this is still illustrative since it might be missing in other settings, in which case nonprofits have to decide how to expend their limited resources to obtain more information (i.e., caseworker follow-up calls or analyzing more recent casenotes $\tilde{Y}$). %Housing placement is an extremely noisy outcome due to measurement challenges. While housing placement is not consistently logged, casenotes are fully completed by outreach workers after each shift. 
%If we had the budget to annotate each casenote and log the correct housing placement then we could easily estimate the causal effect, however we are resource constrained and we have to budget which casenotes to annotate and which ones we can rely on model predictions for. We can select a single budget and compute the optimal annotation probability $\pi^*$, collect the annotations, and combine with model predictions to estimate the ATE. 
%For the purpose of demonstrating the utility of our approach, we restrict ourselves to the cases where we have complete housing placement data and casenotes for clients so that we may evaluate how well our batch adaptive method does for different budgets. 
Similar to Retail Hero, our "Tabular Data" model  %demonstrate the utility of our approach by using 
uses random forests to estimate the outcome model on tabular data alone, $\hat{\mu}_z(X)$, and in ``..+LLM Predictions'' we include LLM predictions $f_z(X,\tilde{Y})$ as additional covariates and estimate $\hat{\mu}_z(X,\tilde{Y})$.

The LLM predictions are given by serializing the tabular data, and/or including summaries --- see the appendix for the prompts used. 

In \Cref{fig:retailhero_lineplots}, we see that overall our adaptive approach shows improvements over uniform random sampling. The MSE approximately doubles when going from both adaptive estimators to random sampling in the tabular data setting and with LLM predictions. For budgets $B = 0.2 - 0.5$, the confidence interval widths for our causal effect estimates range from $0.31 - 0.38$ for the adaptive estimator with balancing weights, $0.47-0.98$ for the plug-in estimator, and $0.76-1.56$ for random sampling. Thus, our batch-adaptive estimators yield far more precise estimates than random sampling, and at lower cost. %from the adaptive estimator with balancing weights to random sampling.

\subsection{Budget saved plots \label{appendix:budget-saved-plots}}

We compute the amount of budget saved due to our batch adaptive sampling approach. We find the sample size required to achieve
the same confidence interval width with batch adaptive annotations using balancing weights (green) and RZ-plug-in (orange) compared to uniform random sampling. \Cref{fig:outreach_budget_saved} in \Cref{appendix:budget-saved-plots} shows that we can save between $43-75\%$ of the budget using the plugin estimator on tabular data alone and by incorporating LLM predictions, and between $53-91\%$ using the balance estimator over the random sampling baseline.

\begin{figure*}[ht!]
    \centering
    \begin{subfigure}[t]{0.49\textwidth}
        \centering
        \includegraphics[width=\textwidth]{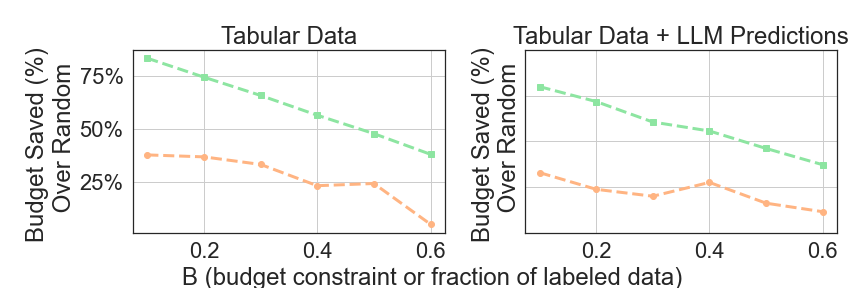}
        \caption{RetailHero data.}
        \label{fig:retail_hero_budget_saved}
    \end{subfigure}
    \hfill
    \begin{subfigure}[t]{0.49\textwidth}
        \centering
        \includegraphics[width=\textwidth]{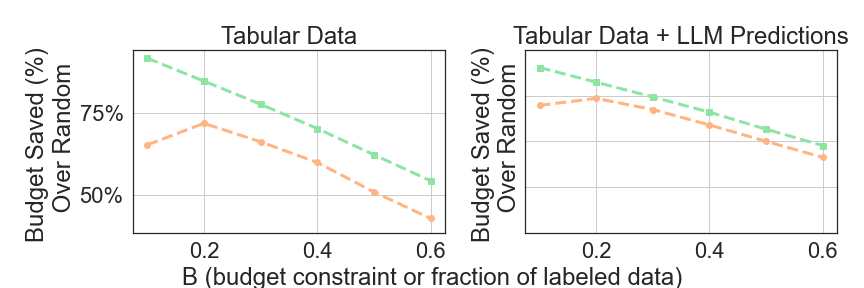}
        \caption{Street Outreach data.}
        \label{fig:outreach_budget_saved}
    \end{subfigure}
    \caption{Budget saved due to batch adaptive annotation on (a) RetailHero and (b) Street Outreach data. Each panel shows the reduction in annotation sample size needed to achieve the same confidence interval width with batch adaptive annotation on tabular data (left) and on tabular data + complex embedded outcomes (right) compared to random sampling.}
    \label{fig:budget-saved}
\end{figure*}

\subsection{Comparison to active learning baselines}\label{appendix-al-baselines}
Active learning is a widely studied subfield in machine learning and operations. A natural question is how our optimal-variance annotation policies compare to prediction-focused active learning. Since our goal is ultimately estimating the average treatment effect, we argue on theoretical and empirical fronts that active learning, which reduces prediction error, is not a strong baseline and would behave poorly in the causal context. 

By their very design, active learning primarily improves $\mu_z$, but the outcome model contributes $\frac{\sigma^2_z(x)}{e_z(x)\pi(z,x)}$ to the causal Avar, and our optimal annotation correctly balances the effect of all factors, but active learning only considers the first. In summary, active learning does something \emph{completely different to improve prediction error, which is generally suboptimal for causal effect estimation}.

On the theoretical side, prior results in batch \emph{pool-based active learning}, \citet{chaudhuri2015covergenceAL} and \citet{gentile2024ratesal} show that active learning doesn't improve convergence rates for \textit{regression}, only multiplicative constants. Instead, the AIPW estimator is optimal for causal estimation: if the outcome and propensity scores can only achieve $n^{-1/4}$ convergence, the AIPW estimator is $O(n^{-1/2})$-rate convergent, so AIPW can speed up outcome model convergence rates. Therefore using the AIPW estimator is best since it leverages the rate-improvements of orthogonal estimation, and random sampling + AIPW is a stronger baseline than active learning.
%Additionally, using pool-based active learning algorithms in AIPW blows up variance due to near-deterministic annotation probabilities. 

Empirically, we run active learning algorithms to learn $\mu$ in AIPW and find that it \emph{totally fails} for these reasons; if these objectives line up, it can do well, but in general, the prediction and causal error objectives are different.

\paragraph{Theoretical comparison to active learning.}

As a reminder, we optimize: 

$$A{Var}_{A T E}=Var[CATE(X)]+\sum_{z \in\{0,1\}} {E}[\frac{\sigma_z^2(X)}{e_z(X) \pi(z, X)}]$$
(The first term is the variance of $CATE = E[Y(1)-Y(0)|X]$; it is never observed.)

% To emphasize the different objectives, consider a simple example with two regions:

% \begin{itemize}
%     \item Region 1 (Poor Overlap), $X>0$: Propensity score $e(X)=0.01$; outcome noise $\sigma_1(X),\sigma_0(X)$=1.
%     \item Region 2 (High Prediction Uncertainty), $X<0$: Propensity score $e(X)=0.5$; outcome noise $\sigma_1(X),\sigma_0(X)=10$ and the outcome model is complex.
% \end{itemize}

% Our method compares the ATE variance contribution in either region, in Region 1: $\frac{\sqrt{1}}{0.01}=100$ and in Region 2: $\frac{\sqrt{100}}{0.5}={20}$ and samples in Region 1, where the causal variance is five times higher. Uncertainty-based active learning samples in Region 2, to the detriment of causal variance. 

\paragraph{Active Learning Empirical Evaluations.}

We evaluate our method against 2-3 active learning baselines for each experiment from two popular and well-established python packages (scikit-activeML and modAL). Different active learning algorithms are appropriate for different outcome models, so we choose the sampling strategy based on our modeling task, and we use pool-based active learning matching our two-batch approach. (Note our approach is \emph{model-agnostic}, while active learning methods are not). For the classification tasks on our two real-world datasets (RetailHero/Street Outreach), we use UncertaintySampling with margin sampling and least confident sampling as query strategies, which both choose x with highest uncertainty measure based on classification probabilities $P(\hat Y=1\mid x)$ \citep{Settles2009ActiveLL}. For the regression tasks, we use Expected Model Variance Reduction \citep{cohn1996ALstats}, Expected Model Change Maximization \citep{cai2013maxexp}, and Improved Greedy Sampling \citep{wu2019active}; these choose $x$ that maximizes greatest future variance reduction, maximally change the current model via the loss gradient, and diversity in feature and output space, respectively. 

We run each approach over 50 trials and take the average MSE, reported in \Cref{tab:synthetic_mse,tab:retailhero_mse,tab:street_mse}. Across the board, we see that our approach does better than the popular active learning strategies that are not optimized for causal estimation.

% \textbf{Result Tables}

\begin{table}[t!]
\centering
\resizebox{\textwidth}{!}{%
\renewcommand{\arraystretch}{1.3}
\setlength{\tabcolsep}{6pt}
\begin{tabular}{|l|ccccccccc|}
\hline
\textbf{Estimator} & 0.1 & 0.2 & 0.3 & 0.4 & 0.5 & 0.6 & 0.7 & 0.8 & 0.9 \\
\hline
active-evar      & \textbf{0.313} & 17.3  & 85.1  & 579   & 1.31e+03 & 3.87e+03 & 1.27e+04 & 5.03e+04 & 8.93e+05 \\
active-greedy    & 6.13  & 79.9  & 369   & 852   & 1.99e+03 & 5.06e+03 & 1.33e+04 & 5.09e+04 & 2.95e+05 \\
active-mvar      & 10.6  & 94.3  & 314   & 883   & 2.17e+03 & 5.70e+03 & 1.21e+04 & 3.87e+04 & 2.99e+05 \\
adaptive-balance  & 0.471 & \textbf{0.227} & \textbf{0.276} & 0.236 & \textbf{0.265} & \textbf{0.246} & \textbf{0.198} & \textbf{0.176} & \textbf{0.203} \\
adaptive-plugin   & 1.7   & 1.17  & 0.831 & \textbf{0.196} & 0.83  & 0.449 & 0.507 & 0.93 & 0.481 \\
random            & 8.99  & 4.56  & 2.19  & 1.54  & 1.7   & 1.61  & 1.46  & 0.956 & 0.987 \\
\hline
\end{tabular}%
}
\caption{Averaged MSEs for Synthetic Data.}
\label{tab:synthetic_mse}
\end{table}

\begin{table}[t!]
\centering
\resizebox{\textwidth}{!}{%
\renewcommand{\arraystretch}{1.3}
\setlength{\tabcolsep}{6pt}
\begin{tabular}{|l|ccccccccc|}
\hline
\textbf{Estimator} & 0.1 & 0.2 & 0.3 & 0.4 & 0.5 & 0.6 & 0.7 & 0.8 & 0.9 \\
\hline
active-margin     & 3.53e+03 & 0.047 & 0.087 & 12.5 & 8.38e+03 & 2.25e+06 & 1.49e+06 & 6.53e+05 & 1.43e+07 \\
active-uncertain  & 16.1     & 38.9  & 70.4  & 75.9 & 115      & 112      & 168      & 250      & 402 \\
adaptive-balance   & \textbf{0.004} & 0.002 & 0.002 & \textbf{0.001} & \textbf{0.001} & 0.001 & \textbf{0} & \textbf{0} & \textbf{0} \\
adaptive-plugin    & \textbf{0.004} & \textbf{0.001} & \textbf{0.001} & \textbf{0.001} & \textbf{0.001} & \textbf{0} & \textbf{0} & \textbf{0} & \textbf{0} \\
random             & 0.027 & 0.012 & 0.009 & 0.006 & 0.005 & 0.003 & 0.001 & 0.001 & \textbf{0} \\
\hline
\end{tabular}%
}
\caption{Averaged MSEs for RetailHero Data.}
\label{tab:retailhero_mse}
\end{table}

\begin{table}[t!]
\centering
\renewcommand{\arraystretch}{1.3}
\setlength{\tabcolsep}{6pt}
\begin{tabular}{|l|ccccccccc|}
\hline
\textbf{Estimator} & 0.1 & 0.2 & 0.3 & 0.4 & 0.5 & 0.6 & 0.7 & 0.8 & 0.9 \\
\hline
active-margin     & \textbf{0.009} & 28.5 & 4.47 & 0.501 & 0.449 & 0.044 & 0.099 & 0.412 & 0.209 \\
active-uncertain  & 0.017 & \textbf{0.009} & 0.018 & 0.008 & 0.017 & 0.018 & 0.025 & 0.023 & 0.024 \\
adaptive-balance   & 0.046 & 0.031 & \textbf{0.013} & \textbf{0.006} & \textbf{0.005} & \textbf{0.003} & \textbf{0.004} & \textbf{0.003} & 0.002 \\
adaptive-plugin    & 0.045 & 0.025 & 0.027 & 0.012 & 0.006 & 0.004 & \textbf{0.004} & 0.006 & \textbf{0.001} \\
random             & 0.113 & 0.061 & 0.037 & 0.045 & 0.014 & 0.012 & 0.011 & \textbf{0.003} & \textbf{0.001} \\
\hline
\end{tabular}
\caption{Averaged MSEs for Street Outreach Data.}
\label{tab:street_mse}
\end{table}

\subsection{Continuous Treatment Results}

Next, we evaluate our batch adaptive allocation protocol in the continuous treatment setting. The setup is similar to the binary treatment experiments where we show that our batch adaptive annotation algorithm enables reliable ATE estimation under varying labeling budgets. In these experiments, we compare our adaptive balancing weights estimator (best in the binary setting) against the estimator with random sampling. 

\paragraph{Simulated Data.} We use the data generating process described in Section 5.1 of \citet{colangelo2020double}, in which the true average dose-response $\beta_z = \Eb{Y(z)}$ at different values of $z$ is known (defined in \Cref{appendix-synthetic-exp}). We evaluate the treatments at a fixed value of $z=1.0$. 
%ßAs we are in the continuous setting, a very key component to our estimator is the kernel function and accompanying bin-width parameter $K_h$. We present results using the 
For $K_h$, we use the Gaussian kernel and set the bin-width as $h=\sigma_Zn^{-0.2}$ by Silverman's rule of thumb, as recommended in \citet{klosin2021automatic}, where $\sigma_Z$ is the standard deviation of the treatment variable and $n$ is the sample size. Because our estimator is kernel-localized around the evaluation point, observations whose treatment values fall far from the target of interest receive zero or near-zero weight in the final dose-response estimate. %Under a fixed labeling budget, allocating annotation effort to such observations does not improve estimation performance and can induce ill-conditioning in the nuisance estimation. These data points contribute almost nothing to the local estimate. 
We therefore restrict first-stage sampling to the set of kernel-relevant observations with $K_h>\epsilon >0$.

Earlier, we saw that the final balancing estimator and $RZ$-plugin had the best finite-sample performance. There is not an analogous form of the plug-in estimator for the continuous setting, so we estimate a "partial" Riesz representer by following \citet{klosin2021automatic} in solving a weighted least squares problem with lasso regularization ($0.01$). In our experiments, 
we use a random forest model to estimate our outcome model on tabular data alone $\hat\mu_z(X)$.

In \Cref{fig:cts_mse_iw_plots}, the leftmost plot shows that our proposed doubly-robust estimator converges to the true average dose-response much faster and consistently across all budgets. Our estimator also benefits from smaller confidence widths on the log scale, leading to greater precision around the estimate.

\begin{figure*}[t!]
    \centering
    % Figures from figs/ (sibling of for-aistats2026)
    \includegraphics[width=\textwidth]{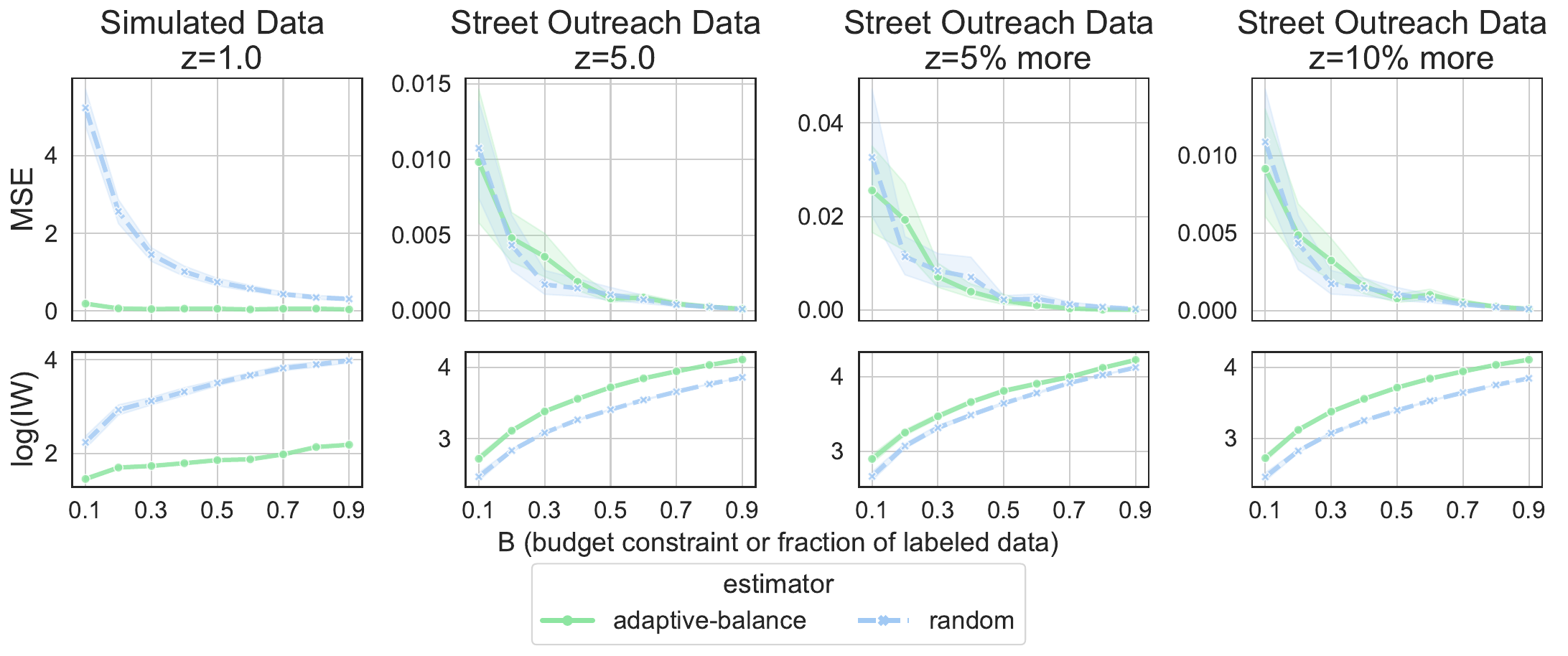}

    %\hspace{-0.5cm}
    
    \caption{\textbf{Experiments on synthetic data (leftmost) and Street Outreach data (center and rightmost) with continuous treatments.} Results of performance measure mean squared error (top row) and $95\%$ confidence interval width on the log scale (bottom row) averaged over 20 and 100 trials (for simulated data) across budget percentages of the data. All experiments are tabular data experiments, where we used a random forest prediction model on the tabular data alone. We evaluate at different treatment values (i.e. $z=1.0$) and counterfactual policies where each unit's treatment is shifted by a percentage (i.e. $z=5\%$ more).}
    \label{fig:cts_mse_iw_plots}
\end{figure*} 

\paragraph{Street Outreach Data.} Next, we evaluate our adaptive annotation algorithm for continuous treatments on the street outreach data from our partnering nonprofit. The covariates and the outcome Y remains the same as in the binary treatment case. Here, we do not binarize the treatment Z, the number of outreach engagements within the first 6 months. %Similar to the binary setting, we also use this as an illustrative example to show how our algorithm performs when outcomes are missing and we have different budget values. 
The results are displayed in \Cref{fig:cts_mse_iw_plots} on the center and rightmost subplots. On this dataset, our approach performs similarly to random sampling. In the simulated setting, the data structure that our adaptive scheme exploits is known and well-behaved, so directing the budget towards relevant observations yields clear gains over uniform sampling. In the real street-outreach data, these quantities must be estimated from noisy, high-dimensional features that may not be informative. This leaves less signal for the adaptive allocation to exploit. As a result, the adaptive estimator's theoretical advantage is attenuated in practice, and its performance is closer to that of random sampling.

\begin{revision}
\section{Case Study: Effect of Street Outreach on Progress Towards Housing}\label{sec-casestudy}
In this section, we answer our original motivating question and investigate the causal effect of street outreach on progress towards housing. We present our initial causal effect estimates on the max housing placement outcome from before and a newly defined progress outcome obtained from a language model finetuned on human annotations. Additionally, we present experiments using our adaptive annotation procedure on the progress label.

\subsection{Background on Project Partnership}
This project is the product of a sustained, five-year research partnership between our team and Breaking Ground. Our multidisciplinary collaboration brought together experts in causal inference, data science, and social work, and has met monthly over this period to explore research directions that are both methodologically rigorous and genuinely useful to the organization. %Finding the right question took time and trust. This specific project was scoped, designed, and implemented over the past two years, including a two-week on-site visit to Breaking Ground's offices in February 2024, where we worked directly with their street outreach team to understand the data challenges that frontline workers and case managers face in their everyday practice. We listened carefully to how outreach workers think about client progress, to the frustrations of undocumented or inconsistently recorded outcomes, and to what kinds of evidence would actually change how the organization operates and advocates for resources. To deepen our understanding of the work, we also spent time embedded in Breaking Ground operations, visiting one of their safe haven facilities and joining their street outreach team for a shift in the field. Witnessing firsthand how outreach workers build trust with unsheltered clients, navigate difficult conversations, and document interactions under real-world constraints was invaluable in shaping how we thought about the data, the outcomes measured, and what a meaningful causal question looks like from the ground up. %Through this process, we secured genuine buy-in from staff at multiple levels and helped the team envision concrete use cases for data they had long collected but never been able to fully leverage. 

%The implementation also included a practical workflow for ongoing use, such as updated case note templates to facilitate future structured data extraction, a documented annotation protocol for expert labelers, and code infrastructure that can be re-run as new case notes accumulate. The causal estimate will be incorporated in future funding applications and it is already being circulated to the organization's current investors, board members, and staff.

In New York City, homeless service organizations have deployed hundreds of outreach workers with precisely this mandate: to engage the hardest-to-reach individuals and connect them to housing and services. Caseworkers canvass communities, build sustained relationships with unsheltered individuals, and document every interaction through detailed case notes (see \Cref{apx-additional-background} for greater detail on the structure of casenotes and the categories used to document interactions).

The qualitative case for street outreach is strong. Research consistently finds that outreach is often the first and only point of contact for individuals who refuse traditional shelter, that 
sustained relationship-building reduces distrust and increases service uptake, and that outreach 
workers play an irreplaceable role in navigating clients through fragmented service systems. Yet rigorous causal evidence remains scarce. \citet{weare2021housing} 
offers one of the only quantitative analyses of housing outcomes, conducting a comparative cost analysis using administrative data to study outreach clients relative to the shelter population. While essential, quantifying impacts can further support the resource allocation decisions that funders and program managers face. %Knowing that outreach clients fare differently than shelter clients is not the same as knowing that outreach \textit{caused} that difference and only the latter justifies continued investment. The causal effect of street outreach on progress toward housing is an important and unanswered question.

% We address this gap by producing the first rigorous causal estimate of street outreach 
% effectiveness. Applied to $272{,}426$ case note records spanning 2019--2021, our method recovers a positive average treatment effect for both the maximum housing placement outcome and a newly defined progress label.

\subsection{Operational motivation for valid LLM-driven causal inference}
We begin by outlining, based on the street outreach operational context, why new methodology for extracting outcome information from text can improve impact evaluation. 

Currently, impact reports of the nonprofit report completed housing placements. Although improving housing placements is a guiding ``North Star'' for the organization, crucially, the nonprofit's service provision has more immediate causal impacts on completing a housing application, whereas whether or not a client receives a housing placement after \textit{submitting} an application is by and large out of the nonprofit and client's control. External factors such as vastly limited housing supply and another agency's matching impact whether or not a client receives a placement. Including the variability of housing placements in the impact evaluation of the nonprofit increases estimation variance. 

However, this operational structure also exposes opportunities for extracting intermediate outcomes from casenote data. It could be reasonable to assume that the nonprofit's outreach efforts don't have a direct causal effect on housing placement after the completion of a housing application, potentially conditional on additional private information of the client such as suitability for interview. We introduce a formal model, where we let $Y^{\mathrm{app}}$ denote completed housing applications and $Y^{\mathrm{place}}$ denote completed placements. We model this situation with the following assumptions:

\begin{assumption}[Nonprofit's outreach influence stops at completed housing applications]\label{asn-outreachinfluence}
For each treatment level \(z\in\{0,1\}\), $\mathbb E\!\left[
Y^{\mathrm{place}}(z)
\mid
Y^{\mathrm{app}}(z)=0,X=x
\right]
=0$,
and there exists a function \(q(x)\in[0,1]\), invariant to \(z\), such that $\mathbb E\!\left[
Y^{\mathrm{place}}(z)
\mid
Y^{\mathrm{app}}(z)=1,X=x
\right]
=q(x)$.
\end{assumption}

The first condition says that packet completion is necessary for placement. The
second says that, conditional on a completed packet and observed client/context
information, downstream conversion to placement is stable with respect to
outreach intensity. The conversion of housing applications and placement is governed by $q(x)$ which is crucially invariant in treatment, $z$. 

If \Cref{asn-outreachinfluence} is true, and if $q(x)$ were known or could be estimated from historical data, then the placement-conversion ($q(x)$)-weighted causal effect of  outreach on $Y^{\mathrm{app}}$ outcomes is \textit{equivalent} to the average treatment effect of nonprofit outreach on housing placement.

\begin{proposition}[Conversion-weighted packet completion has the placement ATE]\label{prop-surrogate}
Under \Cref{asn-outreachinfluence},
% the gateway conversion assumption,
% \[
% \mathbb E\!\left[
% q(X)Y^{\mathrm{app}}(z)
% \mid X=x
% \right]
% =
% \mathbb E\!\left[
% Y^{\mathrm{place}}(z)
% \mid X=x
% \right].
% \]
\begin{align*}
&\mathbb E\!\left[
q(X)\{Y^{\mathrm{app}}(1)-Y^{\mathrm{app}}(0)\}
\right]
=
\mathbb E\!\left[
Y^{\mathrm{place}}(1)-Y^{\mathrm{place}}(0)
\right].\\
& \operatorname{Var}\!\left(
q(X)Y^{\mathrm{app}}
\mid X,Z
\right)
\le
\operatorname{Var}\!\left(
Y^{\mathrm{place}}
\mid X,Z
\right).
\end{align*}
\end{proposition}

The conversion-weighted packet outcome has the same expected placement effect
as terminal placement under the gateway assumption, but removes downstream
placement noise.  However, the former ATE, nonprofit outreach on intermediate outcomes, has lower estimation variance by omitting the randomness of the noisy placement process. Therefore, impact evaluation on intermediate ``gateway'' outcomes can estimate downstream impact more precisely than realized placements. 
%This is the statistical reason a pilot may estimate expected
% impact more precisely using an intermediate outcome than using realized

This formal model outlines conditions where causal inference on intermediate gateway outcomes can be informative of longer-term, higher-variance consequential outcomes, and motivates our efforts to extract operational application progress from text outcomes. Although completed housing applications could be recorded in structured data fields, it can also take a long time ($6$ months) to complete a housing application. Extracting progress signals towards completing a housing application from text data can extract meaningful impact signals across a broader client base, rather than sparse downstream reward signals of completed housing applications or housing placement changes. 

Even though the progress outcomes (such as expressing interest in housing, completing appointments for application completion) may not satisfy \Cref{asn-outreachinfluence} precisely, measuring the average treatment effect on intermediate outcomes within the nonprofit's sphere of causal inference is of independent interest, and can be a higher-precision signal for meaningful downstream impact.

\subsection{Defining a Progress Schema}
Working closely with outreach staff, we defined a new target outcome variable, \textit{progress towards a housing application}. This variable was constructed as a scale ranging from $-1$ to 4, designed to capture the full spectrum of client engagement and movement through the housing placement process. Each casenote records an outreach interaction; we score the casenote for progress. \cref{tab:progress_scale} describes the full progress schema, ranging from simple conversations, to meeting client needs, discussing plans, and administrative prequisites to the housing application, followed by completed milestones like appointments and improved housing placement. The granularity of this scale was motivated by the realities of outreach work with individuals experiencing homelessness, where meaningful progress is rarely linear and often incremental. Traditional binary outcome measures, such as whether a client was housed, fail to capture the small but significant steps that precede a successful placement and may take months or years to materialize. The scale was developed iteratively in consultation with outreach staff to ensure that each label is grounded in the practical knowledge of practitioners working directly with this population. 
%A score of 0 indicates no progress was made, while negative scores denote regression, such as a record of new challenges or a step back from previously achieved milestones. A score of 0.25 captures an attempted but unsuccessful contact to a client demonstrating willingness to share more information with the outreach team. As scores increase, they reflect deeper levels of involvement: a score of 0.9375 corresponds to a neutral discussion of an appointment, a placement request, or a disclosure, as well as meeting of a concrete client need such as arranging transportation or problem solving a barrier. Higher scores capture increasingly goal-directed behavior, with 2 and 2.5 reflecting active planning and administrative progress such as signing documents or completing paperwork. Finally, scores of 3 through 4 represent the most meaningful outcomes: a completed appointment, an improvement in the client's condition, and ultimately a new housing placement. \cref{tab:progress_scale} describes the full progress schema.

\begin{table}[ht!]
    \centering
    \begingroup
    \small
    \setlength{\tabcolsep}{4pt}
    \renewcommand{\arraystretch}{0.95}
    \begin{tabular}{%
        p{0.14\textwidth}%
        p{0.76\textwidth}%
    }
        \hline
        \textbf{Progress Score} & \textbf{Description} \\
        \hline
        \multicolumn{2}{l}{\textbf{Regression}} \\
        $-1$ & Record of challenges or regressed on the progress later \\
        \multicolumn{2}{l}{\textbf{No Progress}} \\
        $0$ & No progress made \\
        $0.25$ & Client attempted call/text, didn't go through \\
        \multicolumn{2}{l}{\textbf{Engagement}} \\
        $0.9375$  & Conversation with outreach worker \\
         & Client interest in conversation, willingness to provide team with more information \\
         & Talk about appointment (fine-grained) in neutral way, placement request, disclosure \\
        & Meeting a client need: completing a transport, problem solving \\
        \multicolumn{2}{l}{\textbf{Planning}} \\
        $2$    & Talk about plans to completing appointment (i.e. reminders, travel plans, status update), plans for rehab or quit drinking, ready for placement \\
        $2.5$  & Signing documents, completing paperwork, progress towards client goals \\
        \multicolumn{2}{l}{\textbf{Completion \& Outcome}} \\
        $3$    & Record of completing appointment \\
        $3.5$  & Improved condition \\
        $4$    & New placement \\
        \hline
    \end{tabular}
    \endgroup
    \caption{Progress schema and descriptions used to measure client advancement through the housing process.}
    \label{tab:progress_scale}
\end{table}

Once we developed and tested this annotation protocol, we recruited 12 Master's of Social Work students to label a budget-limited portion of the case note data; we reconciled their annotations to obtain ground truth. As a proof of concept, instead, we finetune a language model to these progress annotations and use the finetuned model to classify the remainder of the notes. 
Table~\ref{tab:gemma3-july-qwen-summary} reports model-level validation metrics.
RMSE and bias are computed on the original numeric
labels. Our preferred schema merges some labels together, and we report both exact and merged accuracies, within $0.5$ accuracy, as well as metrics that account for uncommon labels: quadratic-weighted Kappa (QWK) and weighted F1 score. 

  \begin{table}[h!]
      \centering
      \small
      \renewcommand{\arraystretch}{1.2}
      \setlength{\tabcolsep}{4pt}
    \begin{revision}
      \begin{tabular}{%
          p{0.19\textwidth}%
          ccccccc
      }
          \hline
          \textbf{Model} & \textbf{Exact} & \textbf{Merged}
  & \textbf{\(\leq 0.5\)}
          & \textbf{QWK} & \textbf{Wt. F1} & \textbf{RMSE} &
  \textbf{Bias} \\
          \hline
          Gemma 3 12B base
            & 0.5145 & 0.6834 & 0.8021 & 0.5951 & 0.6984 &
  0.7319 & -0.2863 \\
          Gemma 3 12B FT
            & \textbf{0.6852} & \textbf{0.8259} &
  \textbf{0.8492} & \textbf{0.7899}
            & \textbf{0.8295} & \textbf{0.5373} &
  \textbf{-0.0278} \\
          Gemma 4 31B FT
            & 0.6135 & 0.7441 & \textbf{0.8378} & 0.6744 &
  0.7863 & 0.7261 & -0.1466 \\
          Qwen3-235B ZS
            & \textbf{0.6517} & 0.7704 & \textbf{0.8391} &
  0.6380 & 0.7800 & 0.7072 & -0.2460 \\
          \hline
      \end{tabular}
      \end{revision}
      \caption{Model-level validation metrics on an $n=379$
  ground-truth
      validation split. Merged accuracy and weighted F1 use
  a merged schema.
      Exact accuracy, \(\leq 0.5\) (near-miss accuracy within $0.5$), QWK, RMSE, and bias are
  computed on the
      original numeric labels. \textbf{Bold} marks best and statistically indistinguishable from best
  (paired bootstrap,
      95\% CI).}
      \label{tab:gemma3-july-qwen-summary}
  \end{table}
  
% \begin{table}[h!]
%     \centering
%     \small
%     \renewcommand{\arraystretch}{1.2}
%     \setlength{\tabcolsep}{4pt}
%   \begin{revision}
%     \begin{tabular}{%
%         p{0.19\textwidth}%
%         ccccccc
%     }
%         \hline
%         \textbf{Model} & \textbf{Exact} & \textbf{Merged} & \textbf{\(\leq 0.5\)}
%         & \textbf{QWK} & \textbf{Wt. F1} & \textbf{RMSE} & \textbf{Bias} \\
%         \hline
%         G3 12B base
%           & 0.5145
%           & 0.6834
%           & 0.8021
%           & 0.5951
%           & 0.6984
%           & 0.7319
%           & -0.2863 \\
%         G3 12B FT
%           & 0.7082
%           & 0.8285
%           & 0.8647
%           & 0.7644
%           & 0.8283
%           & 0.5546
%           & -0.0802 \\
%         G4 31B FT
%           & 0.6135
%           & 0.7441
%           & 0.8378
%           & 0.6744
%           & 0.7863
%           & 0.7261
%           & -0.1466 \\
%         \hline
%     \end{tabular}
%     \end{revision}
%     \caption{Model-level validation metrics. Merged accuracy and weighted F1
%     use the schema that collapses labels \(0.5\), \(0.75\), \(1\), and \(1.5\)
%     into one \(0.5\text{-}1.5\) bucket. RMSE and bias are grouped as numeric
%     error summaries.}
%     \label{tab:gemma3-12b-summary}
% \end{table}

The fine-tuned Gemma 3 12B model improves strict merged accuracy by \(0.1451\)
absolute points relative to the matched Gemma 3 12B base endpoint. The larger
Gemma 4 31B fine-tune does not improve on the 12B fine-tune, likely due to the small fine-tuning dataset.

% We finetuned a Qwen3 32B parameter model to represent our ground-truthed outcomes and we use a base model for our outcome nuisance model. 

\subsection{Results}

\paragraph{Outcome definition}

We operationalize progress by considering the \textit{maximum progress achieved in the 1.5 years past an initial outreach treatment period}. To measure this, we classify each casenote for progress to construct this client-level measures. There are other measures of progress as well, which we leave for future work. Progress towards a housing application is ultimately cumulative due to various appointment and other requirements. Further, clients who complete some appointments at some point need not complete the same ones later on - hence focusing on the maximum progress achieved thus far. 

The key limitation is volume of data: hundreds of thousands of casenotes for our cohort of several hundred clients. As a proof of concept for our annotation procedure, we treat the progress labels from the Gemma 3 12B finetuned model as our ground-truth outcome data $Y$ and we treat the progress labels from the Gemma 3 12B base model as $\tilde{Y}$, which we use to estimate the outcome model $\hat{\mu}_z(X,\tilde{Y})$. This mimics the setting that we are most interested in, where the finetuned model's labels are the expensive ground truth and the base model labels are the cheap but imperfect proxy labels. Because we cannot evaluate our method without ground-truthing the entire dataset, we use the fine-tuned LLM as an approximation.
% \paragraph{\textbf{Max Placement Outcome:}}The method was deployed from 2024-2026 and produced, for the first time in the organization's history, a statistically valid causal estimate of street outreach impact: \textit{Our results provide causal evidence that street outreach improves housing outcomes: our adaptive estimator recovers a positive average treatment effect of $8.6\%$ of clients with improved housing outcomes due to more street outreach.} This is consistent with prior qualitative evidence that sustained street outreach improves housing. We compute this effect estimate on the entire available labeled data, but our experimental results in the previous section show that we can still estimate this effect even with limited labeled data. 

\begin{table}[h]
\centering
\caption{Full-data effect estimates on the 0--4 progress scale. Bold:
90\% CI excludes zero.}
\label{tab:fulldata}
\begin{revision}
\begin{tabular}{lcccc}
\toprule
& \multicolumn{2}{c}{Q3 $-$ Q1 ($n=485$)} & \multicolumn{2}{c}{Q3 $-$ Q2 ($n=699$)} \\
\cmidrule(lr){2-3} \cmidrule(lr){4-5}
Estimator & Estimate (SE) & 90\% CI & Estimate (SE) & 90\% CI \\
\midrule
LLM-adjusted & $\mathbf{0.295}$ $(0.072)$ & $(0.177,\ 0.413)$ & $0.049$ $(0.068)$ & $(-0.063,\ 0.160)$ \\
LLM-selection-adjusted & $0.178$ $(0.179)$ & $(-0.117,\ 0.473)$ & $0.008$ $(0.097)$ & $(-0.152,\ 0.168)$ \\
Balancing     & $\mathbf{0.284}$ $(0.090)$ & $(0.137,\ 0.432)$ & $\mathbf{0.114}$ $(0.065)$ & $(0.006,\ 0.221)$ \\
\bottomrule
\end{tabular}
\end{revision}
\end{table}

\begin{figure}[h]
\centering
\includegraphics[width=\textwidth]{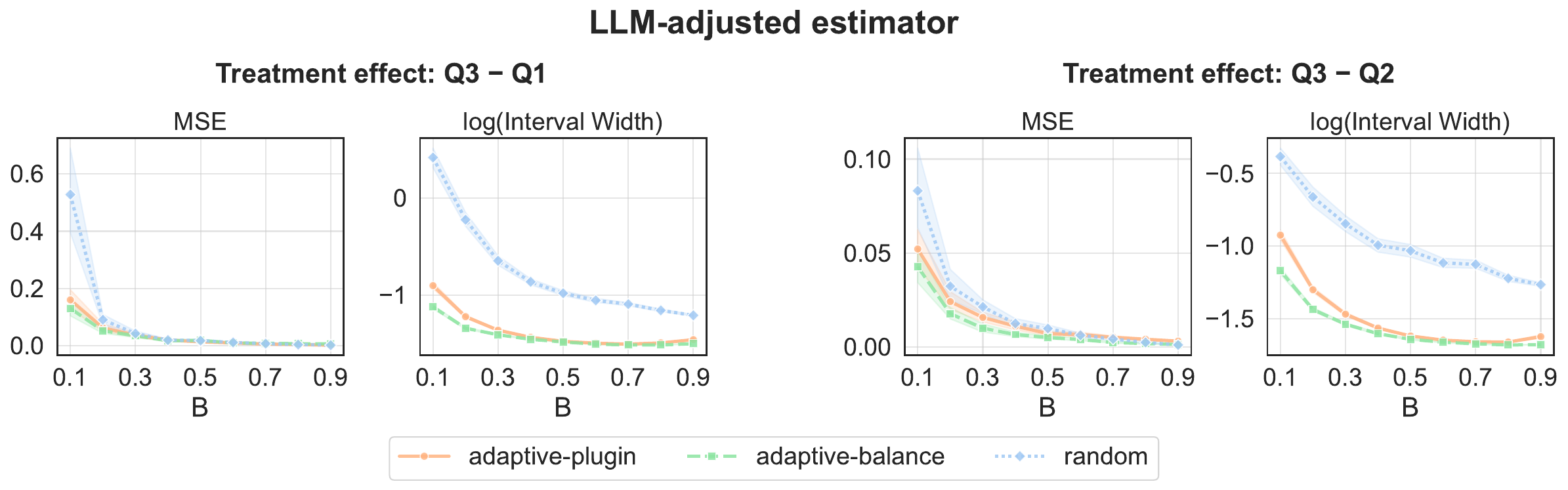}
\includegraphics[width=\textwidth]{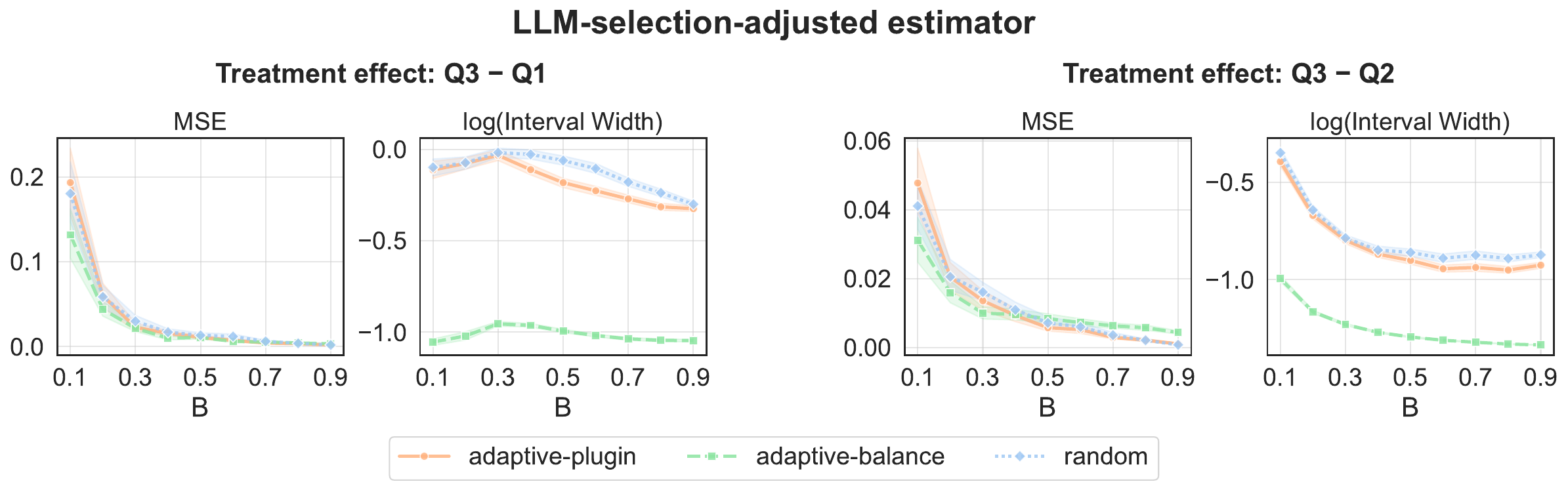}
\caption{MSE and log interval width vs annotation
budget $B$, LLm-adjusted and LLM-selection-adjusted estimators, for the Q1-vs-Q3 contrast (left pair) and the Q2-vs-Q3 contrast
(right pair); 200 trials per budget. Each arm's MSE is computed against its
own full-data benchmark (balance vs $0.284$ / $0.114$; plug-in and random vs
$0.178$ / $0.007$; see note in text). Coverage panels in
Appendix~\ref{apx:pgs-full}.}
\label{fig:main-row}
\end{figure}

\paragraph{Experimental setup} We first run our adaptive annotation procedure on the progress labels. Treatment is defined in the same way as before. We study incremental effects, between Q3 vs Q2 and Q3 vs Q1. Q3 represents clients with 16--226 outreach engagements ($n=387$ clients), Q2 represents clients with 3--15 engagements ($n=312$), and Q1 is 1--2 engagements ($n=98$). We define the outcome of interest as the maximum progress label achieved during the post treatment 1.5 year period. For each budget level, we pretend to observe subsets of the data and compare these partial data effect estimates against the fully observed benchmark.

In this specific progress setting, where we are decoding $Y$ from $\tilde{Y}$, the annotation probabilities that we estimate in our adaptive procedure are also now dependent on $\tilde{Y}$. The LLM-adjusted estimator here is our original AIPW estimator in \cref{eqn-aipw-standard} where we replace $\mu_z(X)$ with $\mu_z(Z,X,\tilde{Y}) = \mu_z(X) + \mathbb{I}[{Z=z}]\{\mu_z^{\tilde{Y}}(X,\tilde{Y}) - \mu_z(X) \}$, while the LLM-selection-adjusted estimator is represented by \cref{eqn-post-treatment-estimator} and introduces the propensity reweighting to account for selection of factual $\tilde{Y}(Z)$. A more detailed discussion can be found in the earlier \Cref{setting-yytilde} section. Across all of the estimators, $\hat{\mu}_z(X)$ is a random forest model trained on $X$ alone and $\hat{\mu}_z^{\tilde{Y}}(X,\tilde{Y})$ is a random forest model with $(X, \tilde{Y})$ as features and performs similarly as direct ensembling with $\tilde{Y}$. For the propensity and annotation probability model, we use a logistic regression for the plugin variant and a random forest based model to learn the balancing weights. We use the same hyperparameters as before. These results suggest that our procedure performs better at causal effect estimation even when progress outcome labels are missing. We run $N=200$ Monte Carlo simulations of the budget-sampling procedure over the (small) dataset.

First, we discuss our full-data estimates, which are of independent substantive interest, before discussing how our adaptive method performs in approaching the full-data estimates in small samples. \Cref{tab:fulldata} contains the estimates, which we report across different estimator variants as a bracketing exercise. We compare two different treatment effects: a smaller change from quantile 2 to quantile 3, and a more intensive change from quantile 1 of outreach to quantile 3. While most estimates of the smaller Q2 to Q3 effect are small, LLM-unadjusted and balancing estimates of the Q1 to Q3 are statistically significant at the 90\% confidence level at around 0.28-0.3, which is around half of our smallest progress schema resolution (half a progress-step). The LLM-selected-adjusted estimate includes additional inverse treatment propensity terms that amplify estimation variance; but we can bracket the estimate within [0.18,0.3], with some dependence on how one weights. Our results indicate that marginal increases in outreach for those who are already being outreached frequently may not improve max progress, as much as larger changes in outreach for those with less outreach to begin with. 

To put that in context of absolute levels, LLM-adjusted estimates of $\E[Y(Q1)], \E[Y(Q2)], \E[Y(Q3)]$ are $1.9,2.2,2.25$, respectively, indicating the dose-response curve appears concave, with diminishing marginal returns. That is, increasing outreach brings average max-progress from just before $2=$ making plans and setting goals, to just past that step. One hypothesis for diminishing returns of outreach is that outreach is a dyadic relationship, and the later stages of completing a housing application rely heavily on client follow-through, persistence and readiness. Moving past $2$ typically that clients attends appointment, follow-ups and so on - which ultimately is also out of outreach's control. Such analysis can support targeting limited intensive outreach resources. We run a similar analysis on max progress labels in the continuous treatment setting in \Cref{appendix:continuous-progress}.  

Next, we discuss our methods in this setting. For the LLM-adjusted estimator, we see our methods improve further in finite samples and have smaller interval width. In this setting, we don't know the ground-truth, so we estimate coverage of the full-data estimate as a measure of resulting confidence interval stability. At very small budgets, this leads to some undercoverage. For the LLM-selection-adjusted estimator, designed to achieve greater efficiency, we see that it results in relatively smaller error at smaller budgets, so the improvements from our estimator are attenuated. Although the intervals undercover the full-data estimate for small budgets $B\in[0.1,0.3]$, they  achieve coverage for $B\geq 0.4$ at a smaller interval width. Overall, these results highlight the applicability to data-annotation and AI settings, which enable new inferential questions from unstructured data. 

% We also conduct the effect estimation on the fully annotated data. Our adaptive estimator produces a positive average treatment effect estimate 
% (point estimate: 0.064, se: 0.021) but our results do not reveal statistically significant evidence of a treatment effect of street outreach on housing outcomes and we are therefore unable to draw causal conclusions from this analysis alone. While the direction of the estimated effect is consistent with prior qualitative and quantitative evidence that sustained street outreach improves housing outcomes, the lack of significance may reflect limitations in the size of the available labeled data rather than a true absence of effect. Our experimental results above suggest that effect estimation under limited labeled data remains a challenging setting, which may contribute to the uncertainty observed here. We hope to conduct future improvements to the progress schema and annotations to improve the effect estimation. 
\end{revision}

\section{Conclusion, limitations, and future work. } We have introduced a procedure for optimal causal annotations.
%and estimators that provides a framework for efficient data labeling and incorporating complex embedded outcomes into causal estimation. 
Our key managerial insights include the paramount importance of grounding LLM-augmented inferences in ground-truth. One should always ``look at the data'', and our analysis demonstrates that out-of-the-box LLMs can miss important context. Our method extracts the most inferential benefit from limtied ground truth. 

Limitations include assuming that annotations reveal ground truth, since annotators might disagree. Our causal analysis on the street outreach data is vulnerable to violations of assumptions such as unconfoundedness. %Our theory also requires LLM statistical consistency, though we suggest using them in ensembled predictions, where model selection .  In future work, we plan to explore more dynamic causal estimators.

This work has delivered important insights for Breaking Ground and the broader homelessness services sector. Outcomes are most typically distal 2-year housing outcomes which are ultimately beyond a single nonprofit's sphere of influence. Beyond the headline estimate, the project has transformed Breaking Ground's data infrastructure and our team continues to work with them to study other housing outcomes and build out other use cases such as LLM generated summaries of client histories. More broadly, this work demonstrates that rigorous causal impact evaluation is achievable even for under-resourced nonprofits operating with unlabeled administrative text data. Our finer analysis yields operational insights: for example, finding diminishing returns in outreach intensity suggests that program expansion at the extensive margin may be more resource-efficient than at the intensive margin.

%\section*{References}
\bibliographystyle{informs2014}
\bibliography{activeannotation,older-submissions/for-aaai-aigov-workshop/aaai25}

\begin{thebibliography}{79}
\providecommand{\natexlab}[1]{#1}
\providecommand{\url}[1]{\texttt{#1}}
\providecommand{\urlprefix}{URL }

\bibitem[{Ao et~al.(2024)Ao, Chen, \protect\BIBand{}
  Simchi-Levi}]{ao2024prediction}
Ao R, Chen H, Simchi-Levi D (2024) Prediction-guided active experiments.
  \emph{arXiv preprint arXiv:2411.12036} .

\bibitem[{Armstrong(2022)}]{armstrong2022asymptotic}
Armstrong TB (2022) Asymptotic efficiency bounds for a class of experimental
  designs. \emph{arXiv preprint arXiv:2205.02726} .

\bibitem[{Athey et~al.(2019)Athey, Chetty, Imbens, \protect\BIBand{}
  Kang}]{athey2019surrogate}
Athey S, Chetty R, Imbens GW, Kang H (2019) The surrogate index: Combining
  short-term proxies to estimate long-term treatment effects more rapidly and
  precisely. Technical report, National Bureau of Economic Research.

\bibitem[{Athey \protect\BIBand{} Wager(2021)}]{athey2021policy}
Athey S, Wager S (2021) Policy learning with observational data.
  \emph{Econometrica} 89(1):pp. 133--161, ISSN 00129682, 14680262,
  \urlprefix\url{https://www.jstor.org/stable/48628848}.

\bibitem[{Bia et~al.(2021)Bia, Huber, \protect\BIBand{}
  Laff{\"e}rs}]{bia2021dmlsampleselection}
Bia M, Huber M, Laff{\"e}rs L (2021) Double machine learning for sample
  selection models. \emph{arXiv prepint arXiv:2012.00745} .

\bibitem[{Bruns-Smith et~al.(2025)Bruns-Smith, Dukes, Feller, \protect\BIBand{}
  Ogburn}]{bruns2025augmented}
Bruns-Smith D, Dukes O, Feller A, Ogburn EL (2025) Augmented balancing weights
  as linear regression. \emph{Journal of the Royal Statistical Society Series
  B: Statistical Methodology} qkaf019.

\bibitem[{Cai et~al.(2013)Cai, Zhang, \protect\BIBand{} Zhou}]{cai2013maxexp}
Cai W, Zhang Y, Zhou J (2013) Maximizing expected model change for active
  learning in regression. \emph{2013 IEEE 13th International Conference on Data
  Mining}, 51--60, \urlprefix\url{http://dx.doi.org/10.1109/ICDM.2013.104}.

\bibitem[{Chaudhuri et~al.(2017)Chaudhuri, Jain, \protect\BIBand{}
  Natarajan}]{chaudhuri2017active}
Chaudhuri K, Jain P, Natarajan N (2017) Active heteroscedastic regression.
  \emph{International Conference on Machine Learning}, 694--702 (PMLR).

\bibitem[{Chaudhuri et~al.(2015)Chaudhuri, Kakade, Netrapalli,
  \protect\BIBand{} Sanghavi}]{chaudhuri2015covergenceAL}
Chaudhuri K, Kakade SM, Netrapalli P, Sanghavi S (2015) Convergence rates of
  active learning for maximum likelihood estimation. Cortes C, Lawrence N, Lee
  D, Sugiyama M, Garnett R, eds., \emph{Advances in Neural Information
  Processing Systems}, volume~28 (Curran Associates, Inc.),
  \urlprefix\url{https://proceedings.neurips.cc/paper_files/paper/2015/file/ca9c267dad0305d1a6308d2a0cf1c39c-Paper.pdf}.

\bibitem[{Chen et~al.(2024)Chen, Bhattacharya, \protect\BIBand{}
  Keith}]{chen2024proximal}
Chen JM, Bhattacharya R, Keith KA (2024) Proximal causal inference with text
  data. \emph{arXiv preprint arXiv:2401.06687} .

\bibitem[{Chernozhukov et~al.(2018)Chernozhukov, Chetverikov, Demirer, Duflo,
  Hansen, Newey, \protect\BIBand{} Robins}]{chernozhukov2018double}
Chernozhukov V, Chetverikov D, Demirer M, Duflo E, Hansen C, Newey W, Robins J
  (2018) Double/debiased machine learning for treatment and structural
  parameters.

\bibitem[{Chernozhukov et~al.(2022)Chernozhukov, Newey, Quintas-Mart{\i}nez,
  \protect\BIBand{} Syrgkanis}]{chernozhukov2022riesznet}
Chernozhukov V, Newey W, Quintas-Mart{\i}nez VM, Syrgkanis V (2022) Riesznet
  and forestriesz: Automatic debiased machine learning with neural nets and
  random forests. \emph{International Conference on Machine Learning},
  3901--3914 (PMLR).

\bibitem[{Cohn et~al.(1996)Cohn, Ghahramani, \protect\BIBand{}
  Jordan}]{cohn1996ALstats}
Cohn DA, Ghahramani Z, Jordan MI (1996) Active learning with statistical
  models. \emph{J. Artif. Int. Res.} 4(1):129–145, ISSN 1076-9757.

\bibitem[{Cohn et~al.(2023)Cohn, Ben-Michael, Feller, \protect\BIBand{}
  Zubizarreta}]{cohn2023balancing}
Cohn ER, Ben-Michael E, Feller A, Zubizarreta JR (2023) Balancing weights for
  causal inference. \emph{Handbook of Matching and Weighting Adjustments for
  Causal Inference}, 293--312 (Chapman and Hall/CRC).

\bibitem[{Colangelo \protect\BIBand{} Lee(2020)}]{colangelo2020double}
Colangelo K, Lee YY (2020) Double debiased machine learning nonparametric
  inference with continuous treatments. \emph{arXiv preprint arXiv:2004.03036}
  .

\bibitem[{Cook et~al.(2024)Cook, Mishler, \protect\BIBand{}
  Ramdas}]{cook2024semiparametric}
Cook T, Mishler A, Ramdas A (2024) Semiparametric efficient inference in
  adaptive experiments. \emph{Causal Learning and Reasoning}, 1033--1064
  (PMLR).

\bibitem[{Craig(2024-2025)}]{columbiaHomelessness}
Craig D (2024-2025) {H}ow to {E}nd {H}omelessness --- magazine.columbia.edu.
  \url{https://magazine.columbia.edu/article/how-end-homelessness-rosanne-haggerty-community-solutions}.

\bibitem[{Cui et~al.(2024)Cui, Pu, Shi, Miao, \protect\BIBand{}
  Tchetgen~Tchetgen}]{cui2024semiparametric}
Cui Y, Pu H, Shi X, Miao W, Tchetgen~Tchetgen E (2024) Semiparametric proximal
  causal inference. \emph{Journal of the American Statistical Association}
  119(546):1348--1359.

\bibitem[{Dhawan et~al.(2023)Dhawan, Cotta, Ullrich, Krishnan,
  \protect\BIBand{} Maddison}]{dhawanend}
Dhawan N, Cotta L, Ullrich K, Krishnan R, Maddison CJ (2023) End-to-end causal
  effect estimation from unstructured natural language data. \emph{The
  Thirty-eighth Annual Conference on Neural Information Processing Systems}.

\bibitem[{Dimakopoulou et~al.(2021)Dimakopoulou, Ren, \protect\BIBand{}
  Zhou}]{dimakopoulou2021online}
Dimakopoulou M, Ren Z, Zhou Z (2021) Online multi-armed bandits with adaptive
  inference. \emph{Advances in Neural Information Processing Systems}
  34:1939--1951.

\bibitem[{Egami et~al.(2022)Egami, Fong, Grimmer, Roberts, \protect\BIBand{}
  Stewart}]{egami2022make}
Egami N, Fong CJ, Grimmer J, Roberts ME, Stewart BM (2022) How to make causal
  inferences using texts. \emph{Science Advances} 8(42):eabg2652.

\bibitem[{Egami et~al.(2023)Egami, Hinck, Stewart, \protect\BIBand{}
  Wei}]{naoki2023dsl}
Egami N, Hinck M, Stewart B, Wei H (2023) Using imperfect surrogates for
  downstream inference: Design-based supervised learning for social science
  applications of large language models. Oh A, Naumann T, Globerson A, Saenko
  K, Hardt M, Levine S, eds., \emph{Advances in Neural Information Processing
  Systems}, volume~36, 68589--68601 (Curran Associates, Inc.),
  \urlprefix\url{https://proceedings.neurips.cc/paper_files/paper/2023/file/d862f7f5445255090de13b825b880d59-Paper-Conference.pdf}.

\bibitem[{Gao et~al.(2019)Gao, Han, Ren, \protect\BIBand{}
  Zhou}]{gao2019batched}
Gao Z, Han Y, Ren Z, Zhou Z (2019) Batched multi-armed bandits problem.
  \emph{Advances in Neural Information Processing Systems} 32.

\bibitem[{Gentile et~al.(2024)Gentile, Wang, \protect\BIBand{}
  Zhang}]{gentile2024ratesal}
Gentile C, Wang Z, Zhang T (2024) Fast rates in pool-based batch active
  learning. \emph{J. Mach. Learn. Res.} 25(1), ISSN 1532-4435.

\bibitem[{Hahn et~al.(2011)Hahn, Hirano, \protect\BIBand{}
  Karlan}]{hahn2011adaptive}
Hahn J, Hirano K, Karlan D (2011) Adaptive experimental design using the
  propensity score. \emph{Journal of Business \& Economic Statistics}
  29(1):96--108.

\bibitem[{Hernan \protect\BIBand{} Robins(2025)}]{hernan2025causal}
Hernan M, Robins J (2025) \emph{Causal Inference: What If}. Chapman \& Hall/CRC
  Monographs on Statistics \& Applied Probab (CRC Press), ISBN 9781420076165,
  \urlprefix\url{https://books.google.com/books?id=_KnHIAAACAAJ}.

\bibitem[{Ho et~al.(2017)Ho, Lim, Reza, \protect\BIBand{} Xia}]{ho2017om}
Ho TH, Lim N, Reza S, Xia X (2017) Om forum—causal inference models in
  operations management. \emph{Manufacturing \& Service Operations Management}
  19(4):509--525.

\bibitem[{Imai \protect\BIBand{} Ratkovic(2014)}]{imai2014covariate}
Imai K, Ratkovic M (2014) Covariate balancing propensity score. \emph{Journal
  of the Royal Statistical Society Series B: Statistical Methodology}
  76(1):243--263.

\bibitem[{Imbens(2004)}]{imbens2004nonparametricest}
Imbens GW (2004) Nonparametric estimation of average treatment effects under
  exogeneity: A review. \emph{The Review of Economics and Statistics}
  86(1):4--29, ISSN 0034-6535,
  \urlprefix\url{http://dx.doi.org/10.1162/003465304323023651}.

\bibitem[{Ipeirotis et~al.(2010)Ipeirotis, Provost, \protect\BIBand{}
  Wang}]{ipeirotis2010quality}
Ipeirotis PG, Provost F, Wang J (2010) Quality management on amazon mechanical
  turk. \emph{Proceedings of the ACM SIGKDD workshop on human computation},
  64--67.

\bibitem[{Jesson et~al.(2021)Jesson, Tigas, van Amersfoort, Kirsch, Shalit,
  \protect\BIBand{} Gal}]{jesson2021causal}
Jesson A, Tigas P, van Amersfoort J, Kirsch A, Shalit U, Gal Y (2021)
  Causal-bald: Deep bayesian active learning of outcomes to infer
  treatment-effects from observational data. \emph{Advances in Neural
  Information Processing Systems} 34:30465--30478.

\bibitem[{Jin et~al.(2021)Jin, von K{\"u}gelgen, Ni, Vaidhya, Kaushal, Sachan,
  \protect\BIBand{} Schoelkopf}]{jin2021causal}
Jin Z, von K{\"u}gelgen J, Ni J, Vaidhya T, Kaushal A, Sachan M, Schoelkopf B
  (2021) Causal direction of data collection matters: Implications of causal
  and anticausal learning for nlp. \emph{arXiv preprint arXiv:2110.03618} .

\bibitem[{Kallus(2018{\natexlab{a}})}]{kallus2018balanced}
Kallus N (2018{\natexlab{a}}) Balanced policy evaluation and learning.
  \emph{Advances in neural information processing systems} 31.

\bibitem[{Kallus(2018{\natexlab{b}})}]{kallus2018optimal}
Kallus N (2018{\natexlab{b}}) Optimal a priori balance in the design of
  controlled experiments. \emph{Journal of the Royal Statistical Society Series
  B: Statistical Methodology} 80(1):85--112.

\bibitem[{Kallus \protect\BIBand{} Mao(2024)}]{kallus2024role}
Kallus N, Mao X (2024) On the role of surrogates in the efficient estimation of
  treatment effects with limited outcome data. \emph{Journal of the Royal
  Statistical Society Series B: Statistical Methodology} qkae099.

\bibitem[{Kallus \protect\BIBand{} Zhou(2018)}]{kallus2018policy}
Kallus N, Zhou A (2018) Policy evaluation and optimization with continuous
  treatments. \emph{International conference on artificial intelligence and
  statistics}, 1243--1251 (PMLR).

\bibitem[{Kallus \protect\BIBand{} Zhou(2021)}]{kallus2021minimax}
Kallus N, Zhou A (2021) Minimax-optimal policy learning under unobserved
  confounding. \emph{Management Science} 67(5):2870--2890.

\bibitem[{Kato et~al.(2020)Kato, Ishihara, Honda, \protect\BIBand{}
  Narita}]{kato2020efficient}
Kato M, Ishihara T, Honda J, Narita Y (2020) Efficient adaptive experimental
  design for average treatment effect estimation. \emph{arXiv preprint
  arXiv:2002.05308} .

\bibitem[{Kennedy(2020)}]{kennedy2020missingexposure}
Kennedy EH (2020) Efficient nonparametric causal inference with missing
  exposure information. \emph{The International Journal of Biostatistics}
  \urlprefix\url{https://doi.org/10.1515/ijb-2019-0087}.

\bibitem[{Klosin(2021)}]{klosin2021automatic}
Klosin S (2021) Automatic double machine learning for continuous treatment
  effects. \emph{arXiv preprint arXiv:2104.10334} .

\bibitem[{Li \protect\BIBand{} Owen(2024)}]{li2024csbae}
Li HH, Owen AB (2024) Double machine learning and design in batch adaptive
  experiments. \emph{Journal of Causal Inference} 12(1):20230068.

\bibitem[{Liao et~al.(2025)Liao, Klugmann, Kondermann, \protect\BIBand{}
  Mahmood}]{liao2025minority}
Liao HW, Klugmann C, Kondermann D, Mahmood R (2025) Minority reports: Balancing
  cost and quality in ground truth data annotation. \emph{arXiv preprint
  arXiv:2504.09341} .

\bibitem[{Mahmood et~al.(2025)Mahmood, Lucas, Alvarez, Fidler,
  \protect\BIBand{} Law}]{JMLR:v26:23-0292}
Mahmood R, Lucas J, Alvarez JM, Fidler S, Law MT (2025) Optimizing data
  collection for machine learning. \emph{Journal of Machine Learning Research}
  26(38):1--52, \urlprefix\url{http://jmlr.org/papers/v26/23-0292.html}.

\bibitem[{Mejia et~al.(2021)Mejia, Mankad, \protect\BIBand{}
  Gopal}]{mejia2021service}
Mejia J, Mankad S, Gopal A (2021) Service quality using text mining:
  Measurement and consequences. \emph{Manufacturing \& Service Operations
  Management} 23(6):1354--1372.

\bibitem[{Mittal et~al.(2025)Mittal, Ma, Joshi, \protect\BIBand{}
  Namkoong}]{mittal2025planning}
Mittal D, Ma Y, Joshi S, Namkoong H (2025) A planning framework for adaptive
  labeling. \emph{arXiv preprint arXiv:2502.06076} .

\bibitem[{Nwankwo et~al.(2026)Nwankwo, Goldkind, \protect\BIBand{}
  Zhou}]{nwankwo2025batch}
Nwankwo E, Goldkind L, Zhou A (2026) Batch-adaptive annotations for causal
  inference with complex-embedded outcomes. \emph{Proceedings of the 29th
  International Conference on Artificial Intelligence and Statistics (AISTATS)}
  .

\bibitem[{Qin \protect\BIBand{} Russo(2024)}]{qin2024optimizing}
Qin C, Russo D (2024) Optimizing adaptive experiments: A unified approach to
  regret minimization and best-arm identification. \emph{arXiv preprint
  arXiv:2402.10592} .

\bibitem[{Rambachan et~al.(2024)Rambachan, Singh, \protect\BIBand{}
  Viviano}]{rambachan2024program}
Rambachan A, Singh R, Viviano D (2024) Program evaluation with remotely sensed
  outcomes. \emph{arXiv preprint arXiv:2411.10959} .

\bibitem[{Rubin(1976)}]{rubin1976missing}
Rubin DB (1976) Inference and missing data. \emph{Biometrika} 63(3):581--592,
  ISSN 00063444, 14643510, \urlprefix\url{http://www.jstor.org/stable/2335739}.

\bibitem[{Schennach(2016)}]{schennach2016recent}
Schennach SM (2016) Recent advances in the measurement error literature.
  \emph{Annual Review of Economics} 8(1):341--377.

\bibitem[{Sch{\"o}lkopf et~al.(2012)Sch{\"o}lkopf, Janzing, Peters, Sgouritsa,
  Zhang, \protect\BIBand{} Mooij}]{scholkopf2012causal}
Sch{\"o}lkopf B, Janzing D, Peters J, Sgouritsa E, Zhang K, Mooij J (2012) On
  causal and anticausal learning. \emph{arXiv preprint arXiv:1206.6471} .

\bibitem[{Settles(2009)}]{Settles2009ActiveLL}
Settles B (2009) Active learning literature survey.
  \urlprefix\url{https://api.semanticscholar.org/CorpusID:324600}.

\bibitem[{Shi et~al.(2024)Shi, Wei, \protect\BIBand{} Wang}]{shiusing}
Shi L, Wei W, Wang J (2024) Using surrogates in covariate-adjusted
  response-adaptive randomization experiments with delayed outcomes. \emph{The
  Thirty-eighth Annual Conference on Neural Information Processing Systems}.

\bibitem[{Shu \protect\BIBand{} Yi(2019)}]{di2019measurmenterror}
Shu D, Yi GY (2019) Causal inference with measurement error in outcomes: Bias
  analysis and estimation methods. \emph{Statistical Methods in Medical
  Research} 28(7):2049--2068,
  \urlprefix\url{http://dx.doi.org/10.1177/0962280217743777}, pMID: 29241426.

\bibitem[{Simchi-Levi \protect\BIBand{} Wang(2023)}]{simchi2023multi}
Simchi-Levi D, Wang C (2023) Multi-armed bandit experimental design: Online
  decision-making and adaptive inference. \emph{International Conference on
  Artificial Intelligence and Statistics}, 3086--3097 (PMLR).

\bibitem[{Sridhar \protect\BIBand{} Blei(2022)}]{sridhar2022causal}
Sridhar D, Blei DM (2022) Causal inference from text: A commentary.
  \emph{Science Advances} 8(42):eade6585.

\bibitem[{Stannard-Stockton()}]{ssirGettingResults}
Stannard-Stockton S (????) {G}etting {R}esults: {O}utputs, {O}utcomes and
  {I}mpact ({S}{S}{I}{R}) --- ssir.org.
  \url{https://ssir.org/articles/entry/getting\_results\_outputs\_outcomes\_impact}.

\bibitem[{Sundin et~al.(2019)Sundin, Schulam, Siivola, Vehtari, Saria,
  \protect\BIBand{} Kaski}]{sundin2019active}
Sundin I, Schulam P, Siivola E, Vehtari A, Saria S, Kaski S (2019) Active
  learning for decision-making from imbalanced observational data.
  \emph{International conference on machine learning}, 6046--6055 (PMLR).

\bibitem[{Tchetgen~Tchetgen et~al.(2024)Tchetgen~Tchetgen, Ying, Cui, Shi,
  \protect\BIBand{} Miao}]{tchetgen2024introduction}
Tchetgen~Tchetgen EJ, Ying A, Cui Y, Shi X, Miao W (2024) An introduction to
  proximal causal inference. \emph{Statistical Science} 39(3):375--390.

\bibitem[{Terwiesch(2019)}]{terwiesch2019om}
Terwiesch C (2019) Om forum—empirical research in operations management: From
  field studies to analyzing digital exhaust. \emph{Manufacturing \& Service
  Operations Management} 21(4):713--722.

\bibitem[{Tsuboi et~al.(2009)Tsuboi, Kashima, Hido, Bickel, \protect\BIBand{}
  Sugiyama}]{tsuboi2009direct}
Tsuboi Y, Kashima H, Hido S, Bickel S, Sugiyama M (2009) Direct density ratio
  estimation for large-scale covariate shift adaptation. \emph{Journal of
  Information Processing} 17:138--155.

\bibitem[{Uehara et~al.(2020)Uehara, Kato, \protect\BIBand{}
  Yasui}]{Uehara2020policyeval}
Uehara M, Kato M, Yasui S (2020) Off-policy evaluation and learning for
  external validity under a covariate shift. \emph{Proceedings of the 34th
  International Conference on Neural Information Processing Systems}, NIPS '20
  (Red Hook, NY, USA: Curran Associates Inc.), ISBN 9781713829546.

\bibitem[{Veitch et~al.(2020)Veitch, Sridhar, \protect\BIBand{}
  Blei}]{veitch2020adapting}
Veitch V, Sridhar D, Blei D (2020) Adapting text embeddings for causal
  inference. \emph{Conference on Uncertainty in Artificial Intelligence},
  919--928 (PMLR).

\bibitem[{Vershynin(2018)}]{vershynin2018high}
Vershynin R (2018) \emph{High-dimensional probability: An introduction with
  applications in data science}, volume~47 (Cambridge university press).

\bibitem[{Wager(2024)}]{wager2024causal}
Wager S (2024) Causal inference: A statistical learning approach.
  \urlprefix\url{https://web.stanford.edu/~swager/causal_inf_book.pdf}.

\bibitem[{Wang et~al.(2017{\natexlab{a}})Wang, Ipeirotis, \protect\BIBand{}
  Provost}]{wang2017cost}
Wang J, Ipeirotis PG, Provost F (2017{\natexlab{a}}) Cost-effective quality
  assurance in crowd labeling. \emph{Information Systems Research}
  28(1):137--158.

\bibitem[{Wang et~al.(2017{\natexlab{b}})Wang, Agarwal, \protect\BIBand{}
  Dud{\i}k}]{wang2017optimal}
Wang YX, Agarwal A, Dud{\i}k M (2017{\natexlab{b}}) Optimal and adaptive
  off-policy evaluation in contextual bandits. \emph{International Conference
  on Machine Learning}, 3589--3597 (PMLR).

\bibitem[{Weare(2021)}]{weare2021housing}
Weare C (2021) Housing outcomes for homeless individuals in street outreach
  compared to shelter. \emph{Journal of Poverty} 25(6):543--561,
  \urlprefix\url{http://dx.doi.org/10.1080/10875549.2020.1869664}.

\bibitem[{Wu et~al.(2019)Wu, Lin, \protect\BIBand{} Huang}]{wu2019active}
Wu D, Lin CT, Huang J (2019) Active learning for regression using greedy
  sampling. \emph{Information Sciences} 474:90--105, ISSN 0020-0255,
  \urlprefix\url{http://dx.doi.org/https://doi.org/10.1016/j.ins.2018.09.060}.

\bibitem[{X5(2019)}]{retailhero}
X5 (2019) X5 retail hero: Uplift modeling for promotional campaign.
  \emph{Challenge Dataset} .

\bibitem[{Yang \protect\BIBand{} Ding(2020)}]{yang2020combining}
Yang S, Ding P (2020) Combining multiple observational data sources to estimate
  causal effects. \emph{Journal of the American Statistical Association} .

\bibitem[{Zhao(2023)}]{zhao2023adaptive}
Zhao J (2023) Adaptive neyman allocation. \emph{arXiv preprint
  arXiv:2309.08808} .

\bibitem[{Zhao(2024)}]{zhao2024experimental}
Zhao J (2024) Experimental design for causal inference through an optimization
  lens. \emph{Tutorials in Operations Research: Smarter Decisions for a Better
  World}, 146--188 (INFORMS).

\bibitem[{Zhao et~al.(2019)Zhao, Small, \protect\BIBand{}
  Bhattacharya}]{zhao2019sensitivity}
Zhao Q, Small DS, Bhattacharya BB (2019) Sensitivity analysis for inverse
  probability weighting estimators via the percentile bootstrap. \emph{Journal
  of the Royal Statistical Society Series B: Statistical Methodology}
  81(4):735--761.

\bibitem[{Zheng et~al.(2023)Zheng, Chiang, Sheng, Zhuang, Wu, Zhuang, Lin, Li,
  Li, Xing et~al.}]{zheng2023judging}
Zheng L, Chiang WL, Sheng Y, Zhuang S, Wu Z, Zhuang Y, Lin Z, Li Z, Li D, Xing
  E, et~al. (2023) Judging llm-as-a-judge with mt-bench and chatbot arena.
  \emph{Advances in Neural Information Processing Systems} 36:46595--46623.

\bibitem[{Zhu \protect\BIBand{} Nowak(2022)}]{zhu2022active}
Zhu Y, Nowak R (2022) Active learning with neural networks: Insights from
  nonparametric statistics. \emph{Advances in Neural Information Processing
  Systems} 35:142--155.

\bibitem[{Zrnic \protect\BIBand{} Candes(2024)}]{zrnic2024active}
Zrnic T, Candes E (2024) Active statistical inference. \emph{Forty-first
  International Conference on Machine Learning}.

\bibitem[{Zrnic \protect\BIBand{}
  Cand{\`e}s(2024)}]{zrnic2024active_statistical}
Zrnic T, Cand{\`e}s EJ (2024) Active statistical inference. \emph{arXiv
  preprint arXiv:2403.03208} .

\bibitem[{Zubizarreta(2015)}]{zubizarreta2015stable}
Zubizarreta JR (2015) Stable weights that balance covariates for estimation
  with incomplete outcome data. \emph{Journal of the American Statistical
  Association} 110(511):910--922.

\end{thebibliography}

%\THEEndNotes
\begingroup \parindent 0pt \parskip 0.0ex \def\enotesize{\normalsize} \theendnotes \endgroup

% Appendix here
% Options are (1) APPENDIX (with or without general title) or
%             (2) APPENDICES (if it has more than one unrelated sections)
% Outcomment the appropriate case if necessary
%
% \begin{APPENDIX}{<Title of the Appendix>}
% \end{APPENDIX}
%
%   or
%
\clearpage

\begin{APPENDICES}

\section{Notation}
\Cref{tab:notation} summarizes the notation used throughout the paper.
\begin{table}[h]
    \centering
    \resizebox{\textwidth}{!}{
    \renewcommand{\arraystretch}{1.3}
    
    \begin{tabular}{ll}
        %\hline
         %& \textbf{Description} \\ 
        \hline
        \( Y_i \) & Ground truth outcomes, observed when label is provided by experts \\ 
        \( \tilde{Y}_i \) & Complex embedded outcomes, such as raw text \\ 
        
        \( X_i \) & Covariates included in estimation \\ 
        \( Z_i\) & Treatment assignment indicator \\ 
        \( R_i \) & Missingness indicator, indicates whether $i$ is expertly labeled\\ 
        \( e_z(X_i) \) & Propensity score, probability of being assigned treatment $Z=z$\\
        \( \pi(Z_i,X_i)\) & Annotation probability, probability of sampling unit $i$ for expert annotation \\
        \( f_z(X_i,\tilde{Y}_i) \) & Estimated function of covariates and complex embedded outcomes, e.g. zero-shot LLM prediction from raw text \\  
        \(  \hat\mu_z(X_i,f(\tilde{Y}_i))\) & Estimated model predicting $Y$ as function of $(X_i,f(\tilde{Y}_i))$\\ 
        \hline
    \end{tabular}}
    \label{tab:notation}
\end{table}

\section{Additional discussion on background}\label{apx-additional-background}
This section provides additional background on the casenotes data from our street outreach application. \Cref{fig:BG-codes} illustrates the structure of casenotes and the high-level categories/codes used to document client interactions.
\begin{figure}
    \centering
    \includegraphics[width=\linewidth]{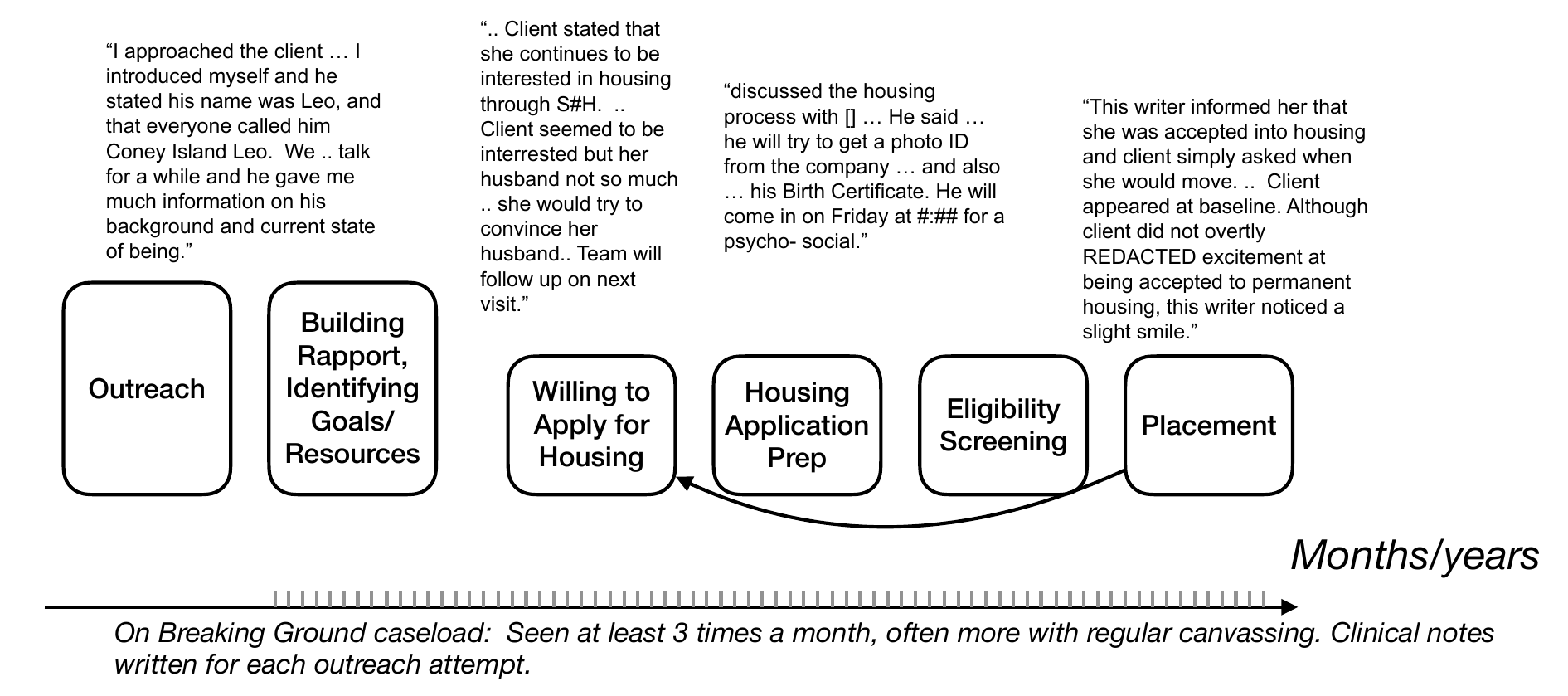}
    \caption{Greater detail on casenotes and high-level categories/codes}
    \label{fig:BG-codes}
\end{figure}

\section{Additional discussion on related work}\label{apx-additional-related-work}

\paragraph{Additional discussion on surrogate estimation}
In much of the surrogate literature, surrogates measure an outcome that is impossible to measure at the time of analysis. The canonical example in \cite{athey2019surrogate} studies the long-term intervention effects of job training on lifetime earnings, by using only short-term outcomes (surrogates) such as yearly earnings. In this regime, the ground truth cannot be obtained at the time of analysis. In this paper, we focus a different regime where obtaining the ground truth from expert data annotators is feasible but budget-binding. 

We leverage the fact that we can design sampling probabilities of outcome observations (ground-truth annotations) or patterns of missingness for doubly-robust estimation, aligning with some methods in the surrogate outcomes and data combination literature \citep{yang2020combining,kallus2024role}. But we treat the underlying setting as a single unconfounded dataset with missingness. 
The different setting of proximal causal inference \citep{tchetgen2024introduction,cui2024semiparametric} seeks proxy outcomes/treatments that are informative of unobserved confounders; we assume unconfoundedness holds. Recently, \citep{chen2024proximal} study the ``design-based supervised learning" perspective of \citep{naoki2023dsl} specifically for proxies for unobserved confounding.

\paragraph{Additional discussion on more adaptive allocation methods beyond batch.}
We outline how our approach is a good fit for our motivating data annotation setting. Full-adaptivity is less relevant in our setting with ground-truth annotation from human experts, due to distributed-computing-type issues with random times of annotation completion. But standard tools such as the martingale CLT can be applied to extend our theoretical results to full adaptivity. Additionally, many recent works primarily focus on the different problem of treatment allocation for ATE estimation. In-sample regret is less relevant for our setting of data annotation, which is a pure-exploration problem.

\paragraph{Optimizing asymptotic variance of the ATE vs. active learning.} An extensive literature in machine learning studies where to sample data to improve machine learning predictors, in the subfield of active learning. The biggest difference is that we target functional estimation, aka improving estimation and inference on the average treatment effect, rather than improving estimation of the black-box nuisance predictors, so our approach is complementary to other approaches for active learning. Approaches for active learning with nonparametric regression include \citet{zhu2022active,chaudhuri2017active}. Active learning generally requires additional structural conditions, such as margin or low-noise conditions, in order to show improvements. Our work highlights optimality leveraging the structure of our final treatment effect inferential goal. 

\paragraph{Other works on causal inference and active learning for heterogeneous treatment effect estimation}
Some papers combine active learning and causal inference, but they primarily focus on estimating the conditional average treatment effect, or CATE = $E[Y(1) - Y(0)\mid X].$ Most of these papers consider estimation via the difference of two regression functions, i.e. CATE estimators that look like $\mu_1(X) - \mu_0(X)$, and therefore focus on active learning for regression methods in general, with a twist of learning the two treated/control regression functions. \citep{jesson2021causal} adapts Bayesian active learning for deep models, but modifies them to avoid sampling in non-overlap regions. \citep{sundin2019active} focuses on sampling \textit{counterfactual} outcome information with a best-arm identification objective (type-S error, to identify the correct sign of treatment effect). While these earlier papers also aim to reveal outcome information when treatment is already assigned, they primarily focus on reducing regression estimation error of an \textit{inefficient/non-doubly-robust} estimator for the CATE. We instead focus on estimating the ATE, and optimizing the asymptotic variance of \textit{semiparametrically efficient} estimation of the averaged ATE functional. 

\paragraph{Relationship to causal inference and NLP}
There is a large and rapidly growing literature on causal inference with text data \citep{egami2022make,sridhar2022causal,veitch2020adapting}. Throughout, we have deliberately used the terminology of measurement error to characterize our approach: that text measures outcomes of interest. \citep{dhawanend} also adopt this stance towards text and note that it differs from prior works on causal inference and NLP, which focuses on questions of substantive interest related to the text itself. 

Although we can define a potential outcome $\tilde{Y}(Z)$, we are generally uninterested in causal inference in the ambient high-dimensional space of $\tilde{Y}(Z)$ itself - corresponding to, in our examples, the effect of the presence of a tumor on the pixel image, the effect of street outreach on the linguistic characteristics of casenotes written for documentation, etc --- $\tilde{Y}(Z)$ is relevant to causal estimation insofar as it is informative of latent outcomes $Y(Z)$.

This is consistent with viewing certain types of NLP tasks as ``anti-causal learning" \citep{scholkopf2012causal}, wherein outcomes cause measurements thereof, in analogy to anti-causal learning in supervised classification where a label of ``cat" or ``dog" causes the classification covariates (e.g. image) \citep{jin2021causal}. Analogously, we view the underlying ground-truth outcomes $Y$ as causing the measurement thereof, $\tilde{Y}$. 

\section{Additional discussion on method}

\subsection{Extensions that permit shared representations or embeddings across treatment arms}

A further way to incorporate annotation or decoding $Y$ from $\tilde{Y}$ is with a common practice of learning embeddings on all outcome data to decode $Y$ from $\tilde{Y}$. We need an additional assumption: an exclusion restriction that the direct causal effect of treatment passes through the ground truth $Y$ alone. Similar assumptions are in the measurement error literature \citep{di2019measurmenterror}. For example, in a medical setting, treatment may shrink a tumor (changing $Y$), which is recorded in clinical notes or imaging data $\tilde{Y}$. But the treatment does not directly affect \textit{how} text or images are \textit{recorded}. This prevents collider bias, and is testable after the first batch of data. 
% We however assume that $Y$ is not a deterministic function of $\tilde{Y}$. 
\begin{assumption}[Complex embedded outcomes: exclusion restriction ]\label{asn-exclusionrestriction} 
% $\tilde{Y}(Z) = g(Y(Z),X) + \epsilon, \; \epsilon \neq 0 \;a.s.; 
$Z \perp \tilde{Y} \mid X,Y$
%, for some function $g$.
\end{assumption}
The assumption justifies pooling labeled data from all treatments to learn a shared encoder, for example a convolutional image encoder of satellite data. However, this assumption is not required for our use of zero-shot or fine-tuned LLMs for casenote prediction when we use LLMs within ensemble models that calibrate LLM predictions within arms. 

\citep{rambachan2024program} studies challenges with ``remotely sensed outcomes'' when transporting models learned between observational and randomized datasets. Crucially, in our setting, we retroactively consider annotation within the same dataset. The exclusion restriction above therefore corresponds to a checkable assumption on the measurement mechanism. If it does not hold, architectural changes in how the representation is incorporated can restore validity. The above assumption is more difficult to satisfy in multiple-dataset settings with potential unobserved confounders in observational data. Within a single dataset, as in our case, it is a milder restriction. 

\subsection{Treatment$-z$-specific budgets $B_z$ \label{appendix:additional-results}}

We also consider a setting with different a priori fixed budgets within each treatment group, where $$\text{sampling budget proportion } B_z \in [0,1]$$ is the max percentage of the treated group $Z=z$ that can be annotated.  
Given that we are trying to choose the $\pi$ that minimizes this variance bound, we only need to focus on the terms that depend on $\pi$ and can drop the rest. Supposing oracle knowledge of propensities and outcome models, the optimization problem, for each $z \in \{0,1\}$ is: 
\begin{align}
  \min_{0 < \pi(z,x) \leq 1, \forall z,x}  & \left\{ \Eb{ \frac{\sigma^2_z(X)}{e_z(X)  \pi(z,X)}  }  
  \colon 
\Eb{\pi(z,X) \mid Z=z} \leq B_z, z\in \{0,1\} \right\} \tag{z-budget}
\end{align}

% two-column
% \begin{align}
%   \min_{\pi(z,x)}  & \Eb{ \frac{\sigma^2_z(X)}{e_z(X)  \pi(z,X)}  }  \tag{z-budget} \\
%     &\text{ subject to } 
% \Eb{\pi(z,X) \mid Z=z} \leq B_z, z\in \{0,1\} \nonumber\\
%     &0 < \pi(z,X) \leq 1, \; \forall z \in \{0,1\}, \forall x \nonumber
% \end{align}

\begin{theorem}[]\label{thm-z-budget-solution}
The solution to the within-$z$-budget problem is: $$
\pi^*(z,X) = \frac{\sqrt{\nicefrac{\sigma^2_z(X)}{  e_z^2(X)}}}{\Eb{ \sqrt{\nicefrac{\sigma^2_z(X)}{  e_z^2(X)}}\mid Z=z } } \cdot B_z 
$$. 
\end{theorem}

\section{Proofs \label{appendix:proofs}}
%%%%%%%%%%%%%%%%%%%%%%%%%%%%%%%%%%%%%%%%%%%%%%%%%%%%%%%%%%%%%%%%%%%%%%%%%%%%%%%
%%%%%%%%%%%%%%%%%%%%%%%%%%%%%%%%%%%%%%%%%%%%%%%%%%%%%%%%%%%%%%%%%%%%%%%%%%%%%%%

% \begin{theorem}[]\label{thm-global-budget-solution}
% The solution to the global budget problem is:     $$
% \pi^*(z,X) = \sqrt{\frac{\sigma^2_z(X)}{  e_z^2(X)}}
% \left(
% \Eb{ \mathbb{I}[Z=1 ]\sqrt{\nicefrac{\sigma^2_1(X)}{  e_1^2(X)}} 
% + 
% \mathbb{I}[Z=0]\sqrt{\nicefrac{\sigma^2_0(X)}{  e_0^2(X)}} }\right)^{-1}\cdot B
% $$. 
% \end{theorem}

\subsection{Results on method motivation}

\proof{Proof of \Cref{prop-surrogate}.}\label{proof:prop-surrogate}
For each \(z\in\{0,1\}\), since \(Y^{\mathrm{app}}(z)\in\{0,1\}\),
\Cref{asn-outreachinfluence} and iterated expectations give
\[
\E[Y^{\mathrm{place}}(z)\mid X]
=
\E\!\left[
\E[Y^{\mathrm{place}}(z)\mid Y^{\mathrm{app}}(z),X]
\mid X
\right]
=
q(X)\E[Y^{\mathrm{app}}(z)\mid X].
\]
Taking expectations and differencing \(z=1\) and \(z=0\) proves the first claim.
By SUTVA/consistency and \Cref{asn-tx-ignorability,asn-outreachinfluence},
\(\E[Y^{\mathrm{place}}\mid Y^{\mathrm{app}},X,Z]
=q(X)Y^{\mathrm{app}}\); hence
\[
\operatorname{Var}(Y^{\mathrm{place}}\mid X,Z)
=
\E\!\left[
\operatorname{Var}(Y^{\mathrm{place}}\mid Y^{\mathrm{app}},X,Z)
\mid X,Z
\right]
+
\operatorname{Var}(q(X)Y^{\mathrm{app}}\mid X,Z)
\geq
\operatorname{Var}(q(X)Y^{\mathrm{app}}\mid X,Z).
\]
\endproof

\subsection{Optimal annotation probability analysis}

\proof{Proof of \Cref{prop-avar-ate}.}\label{proof:prop1}
    % Hahn (1998) provide a lower bound for the variance of the ATE of regular estimators. We derive the terms of that lower bound for our estimator $\hat{\delta}$. 
We simplify the expression for the asymptotic variance of the ATE with missing outcomes to isolate the components affected by the data annotation probability. 

First the variance of the ATE defined in terms of the efficient influence function $\psi_z$ for $z\in\{0,1\}$ is
\begin{align*}
    \mathrm{Var}[\psi_1 - \psi_{0} ] &= \mathrm{Var}\bigg[\frac{\mathbf{1}[Z=1] \cdot R\cdot [Y-\mu_1(X)]}{e_1(X) \cdot \pi(1,X)} + \mu_1(X) - \frac{\mathbf{1}[Z=0] \cdot R\cdot [Y-\mu_{0}(X)]}{e_{0}(X) \cdot \pi(0,X)} + \mu_{0}(X) \bigg] \\
    &= \underbrace{\mathrm{Var}\bigg[\frac{\mathbf{1}[Z=1] \cdot R\cdot [Y-\mu_1(X)]}{e_1(X) \cdot \pi(1,X)}  + \mu_1(X) \bigg]}_{V_1} + \underbrace{\mathrm{Var}\bigg[\frac{\mathbf{1}[Z=0] \cdot R\cdot [Y-\mu_{0}(X)]}{e_{0}(X) \cdot \pi(0,X)} + \mu_{0}(X) \bigg]}_{V_2} \\ &-\underbrace{2\mathrm{Cov}\bigg[ \frac{\mathbf{1}[Z=1] \cdot R\cdot [Y-\mu_1(X)]}{e_1(X) \cdot \pi(1,X)} + \mu_1(X),\frac{\mathbf{1}[Z=0] \cdot R\cdot [Y-\mu_{0}(X)]}{e_{0}(X) \cdot \pi(0,X)} + \mu_{0}(X) \bigg]}_{V_3} 
\end{align*}
\textbf{For $V_3$}: 
\begin{align*}
    &2\mathrm{Cov}\bigg[ \frac{\mathbf{1}[Z=1] \cdot R\cdot [Y-\mu_1(X)]}{e_1(X) \cdot \pi(1,X)} + \mu_1(X),\frac{\mathbf{1}[Z=0] \cdot R\cdot [Y-\mu_{0}(X)]}{e_{0}(X) \cdot \pi(0,X)} + \mu_{0}(X) \bigg]\\
    = &2 \Bigg[ \Eb{\frac{\mathbf{1}[Z=1] \cdot R}{e_1(X) \cdot \pi(1,X)} [\underbrace{\Eb{Y|Z=1,R=1,X}-\mu_1(X)}_{=0}]}\Bigg] \\
    + &\Bigg[\Eb{\mu_1(X) \cdot \frac{\mathbf{1}[Z=0] \cdot R}{e_{0}(X) \cdot \pi(0,X)} [\underbrace{\Eb{Y|Z=0,R=1,X}-\mu_{0}(X)}_{=0}] + \mu_{0}(X)}\Bigg] \\ 
    &- \Eb{\frac{\mathbf{1}[Z=1] \cdot R}{e_1(X) \cdot \pi(1,X)} [\underbrace{\Eb{Y|Z=1,R=1,X}-\mu_1(X)}_{=0}] + \mu_1(X)} \\
    &\times \Eb{\frac{\mathbf{1}[Z=0] \cdot R}{e_{0}(X) \cdot \pi(0,X)} [\underbrace{\Eb{Y|Z=0,R=1,X}-\mu_{0}(X)}_{=0}] + \mu_{0}(X)} \Bigg] \\
    = &2\bigg[ \Eb{\mu_1(X)\cdot \mu_{0}(X)} - \Eb{\mu_1(X)\mu_{0}(X)}\bigg]  
\end{align*}

\textbf{For $V_1$:}
\begin{align*}
    &\mathrm{Var}\bigg[ \frac{\mathbf{1}[Z=1] \cdot R\cdot [Y-\mu_1(X)]}{e_1(X) \cdot \pi(1,X)} + \mu_1(X)\bigg] \\
    &= \mathrm{Var}\bigg[ \frac{\mathbf{1}[Z=1] \cdot R\cdot [Y-\mu_1(X)]}{e_1(X) \cdot \pi(1,X)}\bigg]  + \mathrm{Var}[\mu_1(X)] + 2\underbrace{\mathrm{Cov}\bigg[ \frac{\mathbf{1}[Z=1] \cdot R\cdot [Y-\mu_1(X)]}{e_1(X) \cdot \pi(1,X)},\mu_1(X)\bigg]}_{=0} \\
    &= \Eb{\bigg[ \frac{\mathbf{1}[Z=1] \cdot R\cdot [Y-\mu_1(X)]}{e_1(X) \cdot \pi(1,X)}\bigg]^2} - \bigg[ \frac{\mathbf{1}[Z=1] \cdot R\cdot }{e_1(X) \cdot \pi(1,X)} [\underbrace{\Eb{Y|Z=1,R=1,X}-\mu_1(X)}_{=0}]\bigg]^2 \\
    &+ \Eb{\mu_1(X)^2} -\Eb{\mu_1(X)}^2\\ 
    &= \Eb{\bigg[ \frac{\mathbf{1}[Z=1]^2 \cdot R^2}{e^2_1(X) \cdot \pi^2(1,X)} \cdot [Y-\mu_1(X)]^2\bigg]} + \Eb{\mu_1(X)^2} -\Eb{\mu_1(X)}^2 \\
    &= \Eb{\frac{\mathbf{1}[Z=1]\cdot R}{e^2_1(X) \cdot \pi^2(1,X)} \cdot [Y-\mu_1(X)]^2} + \Eb{\mu_1(X)^2} -\Eb{\mu_1(X)}^2\\
    &=\Eb{ \frac{1}{e_1(X) \cdot \pi(1,X)} \cdot [Y-\mu_1(X)]^2} + \Eb{\mu_1(X)^2} -\Eb{\mu_1(X)}^2
\end{align*}

Lastly, $V_1 = V_2$. So the full variance term is
\begin{align*}
    \mathrm{Var}[\psi_1 - \psi_{0}] &= \Eb{ \frac{1}{e_1(X) \cdot \pi(1,X)} \cdot [Y-\mu_1(X)]^2} + \Eb{ \frac{1}{e_{0}(X) \cdot \pi(0,X)} \cdot [Y-\mu_{0}(X)]^2}\\
    &+ \Eb{(\mu_1(X) - \mu_{0}(X))^2} - \Eb{\mu_1(X) -\mu_{0}(X)}^2 \\
    &= \Eb{ \frac{1}{e_1(X) \cdot \pi(1,X)} \cdot [Y-\mu_1(X)]^2} + \Eb{ \frac{1}{e_{0}(X) \cdot \pi(0,X)} \cdot [Y-\mu_{0}(X)]^2} \\ 
    &+ \mathrm{Var}\bigg[\mu_1(X) - \mu_{0}(X)\bigg]
\end{align*} 

Rewriting the bound from Hahn (1998), we get 

%\textcolor{red}{Not sure about this, should this be $\Delta$ or $\delta$? And if in terms of $\delta$, can we still drop because it will be dependent on $\pi$.}
\begin{align*}
    V &\geq \Eb{ \frac{1}{e_1(X) \cdot \pi(1,X)} \cdot [Y-\mu_1(X)]^2} + \Eb{ \frac{1}{e_{0}(X) \cdot \pi(0,X)} \cdot [Y-\mu_{0}(X)]^2} \\
    &+ \mathrm{Var}\bigg[\mu_1(X) - \mu_{0}(X)\bigg] 
\end{align*}

\endproof

\proof{Proof of \Cref{thm-z-budget-solution}.}\label{proof:thm4}

Finding the optimal $\pi$ can be separated into sub-problems for each treatment $z \in \{0,1 \}$, since the objective and dual variables are separable across $z$. We first look at a solution for $\pi(z,X)$ for a given $z$: 
\begin{align}
  \min_{\pi(z,x)}  &
  \Eb{ \frac{\sigma^2_z(X)}{e_z(X)  \pi(z,X)}  }  \tag{z-budget} \\
   \text{ s.t. } & \Eb{\pi(z,X) \mid Z=z} \leq B_z, \nonumber\\
    &0 < \pi(z,x) \leq 1, \; \forall x \nonumber
\end{align}
    
We define the Lagrangian of the optimization problem and introduce dual variables $\lambda$ for the budget constraint and $\eta$ and $\nu$ for the the constraint that $0 < \pi(z,X) \leq 1$:  
$$ 
\mathcal{L} = \Eb{ \frac{ (Y- \mu_z(X))^2 }{e_z(X)  \pi(z,X)} }  
+ \lambda_z(\Eb{\pi(z,X) \mid Z=z } - B_z) + 
\sum_{x\in\mathcal{X}} (\nu_x^z (\pi(z,x)-1) - \eta_x^z \pi(z,x))
$$
% We also define an ``x-conditional" Lagrangian: 
% $$ 
% {L}(X)  =  \frac{\sigma^2(X)}{e_z(X) \cdot \pi(z,X)} 
% + 
%  \lambda(\pi(z,X) - B_z) + 
%  \frac{1}{p(x)}
% (\nu_x (\pi(z,X)-1) 
% - \eta_x \pi(z,X)),
% $$
% where $\sigma^2(X) = \Eb{(Y -\mu(z,1,X))^2|X}$. Note (in part, by linearity of the budget constraint) that
% $$ \mathcal{L} = \Eb{{L}(X)}$$ 

Define the conditional outcome variance $\sigma^2(X) = \Eb{(Y -\mu(z,1,X))^2|X}$. Note that by iterated expectations, 
$$ 
\mathcal{L} = \Eb{ \frac{\sigma^2_z(X)}{e_z(X)  \pi(z,X)} }  
+ \lambda_z(\Eb{\pi(z,X)  \mid Z=z } - B_z) + 
\sum_{x\in\mathcal{X}} (\nu_x^z (\pi(z,x)-1) - \eta_x^z \pi(z,x))
$$

% Note that for any given value of $x$,
% $$
% \frac{\partial {L}(x)}{\partial \pi(z,x)} = 
% -\frac{\sigma^2(x)}{e_z(x) \cdot \pi(z,x)^2} 
% + \lambda + 
% \nu_x 
% - \eta_x 
% $$
%\textcolor{blue}{Need to add $\pi(z',X)$ terms}
%\todo{ $\Eb{ [Y-\mu(z,1,X)]^2 \mid X}$ should be $\sigma^2(X)\Eb{ [Y-\mu(z,1,X)]^2 \mid X}$}

%$$
%\mathcal{L}(\pi,\lambda,\nu,\eta) = \Eb{\frac{1}{e_z(X)\pi(z,X)}\cdot [Y-\mu(z,1,X)]^2} + \lambda(\Eb{\pi(z,X)} + \Eb{\pi(z',X)} - B) + \nu_1(\pi(z,X)-1) + \nu_2(\pi(z',X) -1) - \eta_1(\pi(z,X)) - \eta_2(\pi(z',X)) = 0
%$$

% Under regularity conditions, swapping expectation and the derivative, we can obtain that 
% $$ \frac{ \partial  \mathcal{L} }{\partial \pi(z,x)} = \Eb{\frac{\partial {L}(X)}{\partial \pi(z,x)}}
% $$

We can find the optimal solution by setting the derivative equal to 0. Since $p(X=x\mid Z=z) = \frac{e_z(x) p(x)}{p(Z=z)}$
\begin{align*}
    \frac{\partial \mathcal{L}}{\partial \pi(z,X)} &= 
 -\frac{{\sigma^2(X)}}{e_z(X)(\pi^2(z,X))}p(x)  + \lambda_z \frac{e_z(x) p(x)}{p(Z=z)}
 + \nu_x - \eta_x = 0, \text{ where } p(x)>0
\\ &= -\frac{{\sigma^2(X)}}{e_z^2(X)\pi^2(z,X)} + \frac{\lambda_z}{p(Z=z)} + 
 \frac{ (\nu_x^z - \eta_x^z) }{p(x)e_z(x) }= 0
 \\
\end{align*}
Therefore
$$
\pi(z,x) = \sqrt{\frac{ \sigma^2(x) }{e_z^2(x)(
\frac{\lambda_z}{p(Z=z)} +  \frac{ (\nu_x^z - \eta_x^z) }{p(x)e_z(x) })}}$$
% We define a clipped probability in terms of $\lambda_z$ as 
% $$\pi^\lambda(z,x) = \text{clip}\left(0,
% \sqrt{\frac{ \sigma^2(x) }{e_z^2(x)(\frac{\lambda_z}{p(Z=z)} +  \frac{ (\nu_x^z - \eta_x^z) }{p(x)e_z(x) })}},1\right)
% $$
Next we give a choice of $\lambda$ that results in an interior solution with $0 \leq \pi(z,x)\leq 1$, so that $\nu^z_x, \eta_x^z$ can be set to $0$ without loss of generality to satisfy complementary slackness.

We posit a closed form solution
$$
\pi^*(z,X) = \frac{\sqrt{\nicefrac{\sigma^2_z(X)}{  e_z^2(X)}}}{\Eb{ \sqrt{\nicefrac{\sigma^2_z(X)}{  e_z^2(X)}}
\mid Z=z
} } \cdot B_z 
$$. 

Note that this solution is self-normalized to satisfy the budget constraint such that 

$$
\Eb{\pi^*(z,X)  \mathbb{I}[Z=z ]} = \Eb{\frac{\sqrt{\nicefrac{\sigma^2(X)}{  e_z^2(X)}}}{\Eb{ \sqrt{\nicefrac{\sigma^2_z(X)}{  e_z^2(X)}} \mid Z=z }  }  B_z \mid Z=z } = B_z
$$

This solution corresponds to a choice of $\lambda^*_z = \nicefrac{p(Z=z)\Eb{ \sqrt{\nicefrac{\sigma^2(X)}{  e_z^2(X)}
} \mid Z=z
}^2 }{B_z^2}$ in the prior parametrized expression. 

\begin{align*}
    \pi_{\lambda}(z,X) &= \pi^*(z,X) \\
    \sqrt{\frac{{\sigma^2_z(X)}}{ e_z^2(X) \frac{\lambda}{p(Z=z)}}} &= \frac{\sqrt{\nicefrac{\sigma^2_z(X)}{e_z^2(X)}}}{\Eb{ \sqrt{\nicefrac{\sigma^2_z(X)}{  e_z^2(X)}} \mid Z=z }}\cdot B_z 
    % \\
    % \frac{\Eb{\sigma^2(X)}}{\lambda e_z(X)}
    % &= \frac{\nicefrac{\sigma^2(X)}{e_z(X)}}{(\Eb{\sqrt{\nicefrac{\sigma^2(X)}{e_z(X)}}})^2}\cdot B_z^2 \\ 
    % \Rightarrow \lambda^* &= \frac{(\Eb{\sqrt{\nicefrac{\sigma^2(X)}{e_z(X)}}})^2 \Eb{\sigma^2(X)}}{\sigma^2(X)\cdot B_z^2}
\end{align*} 

We can check that the KKT conditions are satisfied at  $\pi^*(z,X)$ and $\lambda^*$. We note that since $\pi^*(z,X)$ is an interior solution then w.l.o.g we can fix $\nu_x, \eta_x= 0$ to satisfy complementary slackness. 

It remains to check that $\frac{\partial \mathcal{L} }{\partial \pi^*(z,X)} = 0$, we have that:   

$$
\frac{\partial\mathcal{L}}{\partial \pi(z,X)} = 
-\frac{{\sigma^2_z(X)}}{e_z(X)} \cdot  
\frac{e_z^2(X)\Eb{ \sqrt{\nicefrac{\sigma^2_z(X)}{  e_z(X)}} \mid Z=z
}^2
}{ 
\sigma^2_z(X) \cdot B_z^2
}
+ \frac{\Eb{\sqrt{\nicefrac{\sigma^2(X)}{e_z(X)}} \mid Z=z 
}^2 \sigma^2_z(X) e_z(X)}{\sigma^2_z(X)\cdot B_z^2} + 0 = 0.
$$

Thus we have shown that $\pi^*(z,X)$ is optimal. 

\endproof

\proof{Proof of \Cref{thm-global-budget-solution}.}\label{proof:thm1}
Proceed as in the proof of \Cref{thm-z-budget-solution}. 

The Lagrangian of the optimization problem (with a single global budget constraint) is: 

\begin{align*}
\mathcal{L} &= \sum_{z\in\{0,1\}} 
% \resizebox{\textwidth}{!}{%
% $
% \left\{ 
\Eb{ \frac{ (Y- \mu_z(X))^2 }{e_z(X)  \pi(z,X)} }+ 
\sum_{x\in\mathcal{X}} (\nu_x^z (\pi(z,x)-1) - \eta_x^z \pi(z,x))
\\
&\qquad\qquad + \lambda(\Eb{\pi(1,X) \mathbb{I}[Z=1]+\pi(0,X) \mathbb{I}[Z=0] } - B)
% \right\} 
% $}
\end{align*}

Again by iterated expectations, 
$$ 
\mathcal{L} = \Eb{ \frac{\sigma^2_z(X)}{e_z(X)  \pi(z,X)} }  
+ \lambda(\Eb{\pi(1,X) e_1(X) + \pi(0,X) e_0(X) } - B_z) + 
\sum_{x\in\mathcal{X}} (\nu_x^z (\pi(z,x)-1) - \eta_x^z \pi(z,x))
$$
We can find the optimal solution by setting the derivative equal to 0. 
\begin{align*}
    \frac{\partial \mathcal{L}}{\partial \pi(z,X)} &= 
 -\frac{{\sigma^2(X)}}{e_z(X)(\pi^2(z,X))}p(x)  + \lambda p(x) e_z(x)  + \nu_x^z - \eta_x^z = 0, \text{ where } p(x)>0
\\ &= -\frac{{\sigma^2(X)}}{e_z^2(X)\pi^2(z,X)} + \lambda + 
 \frac{ (\nu_x^z - \eta_x^z) }{p(x)e_z(x) }= 0
\end{align*}
Therefore we obtain a similar expression parametrized in $\lambda$, but this parameter is the same across both groups under a global budget. 
$$
\pi(z,x) = \sqrt{\frac{ \sigma^2(x) }{e_z^2(x)(\lambda +  \frac{ (\nu_x^z - \eta_x^z) }{p(x)e_z(x) })}}$$

We can similarly give a closed-form expression for a different choice of $\lambda$ yielding an interior solution, so that we can set $\nu_x^z,\eta_x^z=0$ without loss of generality. 
$$
\lambda =  \frac{\Eb{ 
\mathbb{I}[Z=1]\sqrt{\nicefrac{\sigma^2_1(X)}{  e_1^2(X)}} 
+ \mathbb{I}[Z=0]\sqrt{\nicefrac{\sigma^2_0(X)}{  e_0^2(X)}}
}^2 
}{B^2}
$$
Notice that this satisfies the normalization requirement that $\E[\pi^\lambda(1,X)\mathbb{I}[Z=1] +\pi^\lambda(0,X)\mathbb{I}[Z=0] ]\leq B,$ and similarly note that the partial derivatives with respect to $\pi(z,x)$ are $0$. 
\endproof

\proof{Proof of \Cref{prop:cts-trt-avar}.}\label{proof:prop2}

We simplify the expression for the asymptotic variance of the ATE with missing outcomes and continuous treatments. We derive the variance and the bias terms and isolate the components affected by the data annotation probability. Again, here $f_{Z|X}(z|x)$ is defined as conditional probability density of treatment given covariates and later we will use $f_{ZX}(z,x)$ to refer to the joint distribution between treatments and covariates. And the "partial" Riesz representer is $\alpha(z,x) = \frac{1}{f_{Z\mid X}(z,x)}$ and we introduce $\bar{\alpha}$ to account for mispecification. 

\begin{align*}
\mathrm{Var}[\psi_z] &= \mathrm{Var}\Bigg[\mu(z,X) + \frac{K_h(Z - z) \alpha(z,X)R}{\pi(z,X)}(Y-\mu(z,X))\Bigg] \\ 
    &= \mathrm{Var}\Bigg[ \frac{K_h(Z - z) \alpha(z,X)R}{\pi(z,X)}(Y-\mu(z,X))\Bigg] + \mathrm{Var}[\mu(z,X)] + \underbrace{2\mathrm{Cov}\Bigg[\frac{K_h(Z - z) \alpha(z,X)R}{\pi(z,X)}(Y-\mu(z,X)),\mu(z,X)\Bigg]}_{=0}
\end{align*}
We focus on the first term as it is the part that depends on $\pi(z,x)$:
\begin{align*}
    V&=V\Bigg[\Eb{\frac{K_h(Z - z) \alpha(z,X)R}{\pi(z,X)}(Y-\mu(z,X))}\Bigg] 
    + \Eb{V\Bigg[\frac{K_h(Z - z) \alpha(z,X)R}{\pi(z,X)}(Y-\mu(z,X))\Bigg]} \tag{Law of total variance}\\
    &\Eb{\Bigg(\Eb{\frac{K_h(Z - z) \alpha(z,X)R}{\pi(z,X)}(Y-\mu(z,X))}\Bigg)^2} - \Bigg(\Eb{\Eb{\frac{K_h(Z - z) \alpha(z,X)R}{\pi(z,X)}(Y-\mu(z,X))}}\Bigg)^2 \\
    &+ \Eb{\Eb{\Bigg(\frac{K_h(Z - z) \alpha(z,X)R}{\pi(z,X)}(Y-\mu(z,X))\Bigg)^2}} - \Eb{\Bigg(\Eb{\frac{K_h(Z - z) \alpha(z,X)R}{\pi(z,X)}(Y-\mu(z,X))}\Bigg)^2} \\
    &= \underbrace{\Eb{\Eb{\Bigg(\frac{K_h(Z - z) \alpha(z,X)R}{\pi(z,X)}(Y-\mu(z,X))\Bigg)^2}}}_{V_z} - \underbrace{\Bigg(\Eb{\Eb{\frac{K_h(Z - z) \alpha(z,X)R}{\pi(z,X)}(Y-\mu(z,X))}}\Bigg)^2}_{B_z} \tag{canceled out first and fourth term of expansion}
\end{align*}

For $V_z$: 
\begin{align*}
    &\Eb{\Eb{\Bigg(\frac{K_h(Z - z) \bar{\alpha}(z,X)R}{\pi(z,X)}(Y-\mu(z,X))\Bigg)^2}} = \Eb{\Eb{\frac{K^2_h(Z - z) \bar{\alpha}^2(z,X)R^2}{\pi^2(z,X)}(Y-\mu(z,X))^2}} \\
    &= \int_{\mathcal{X}} \int_{Z_0}  \frac{K^2_h(s-z) \bar{\alpha}^2(s,x)R^2}{\pi^2(s,x)} \Eb{(Y-\mu(s,x))^2|Z=s,X=x} f_{ZX}(s,x)ds dx \\
    &= h^{-1} \int_{\mathcal{X}}\int_{\mathcal{Q}} \frac{k^2(u)\bar{\alpha}^2(s,x)R^2}{\pi^2(s,x)}\Eb{(Y-\mu(s,x))^2\mid Z=z+uh,X=x}f_{ZX}(z+uh,x)dudx \tag{change of variables: u=\nicefrac{s-z}{h},s=z+uh} \\
    &= h^{-1}\int_{\mathcal{X}}\int_{\mathcal{Q}} \Bigg(\Eb{(Y-\mu_(z,x))^2\mid Z=z,X=x}+uh\frac{d}{d z}  \Eb{(Y-\mu(z,x))^2\mid Z=z,X=x}\bigg\vert_{z=\bar{z}}\Bigg) \\ 
    &\times \Bigg( f_{ZX}(z,x) + uh\frac{d}{d z} f_{ZX}(z,x)\bigg\vert_{z=z'}\Bigg) k^2(u) \frac{\bar{\alpha}^2(z,x)R^2}{\pi^2(z,x)}dudx \tag{taylor expansion and mean value theorem for $\bar{z},z'$ between $z,z+uh$} \\
    &= h^{-1} \int_{\mathcal{X}} \frac{\bar{\alpha}^2(z,x)R^2}{\pi^2(z,x)} \Bigg[ \int_{\mathcal{R}}k^2(u) \Eb{(Y-\mu(z,x))^2\mid Z=z, X=x} f_{ZX}(z,x) + o(h^2) du \Bigg] dx \\
    &= h^{-1} \int_{\mathcal{X}} \frac{\bar{\alpha}^2(z,x)R^2}{\pi^2(z,x)} \xi_k \Eb{(Y-\mu(z,x))^2\mid Z=z, X=x} f_{ZX}(z,x)dx  + o(h^2) \tag{$\xi_k \equiv \int k^2(u)$} \\ 
    &= h^{-1} \Eb{\frac{\bar{\alpha}^2(z,x)R^2}{\pi^2(z,x)}f_{Z\mid X}(z\mid x)\Eb{(Y-\mu( z,x))^2\mid Z=z,X=x}} \xi_k \\ 
    &= h^{-1} \Eb{\frac{\bar{\alpha}^2(z,x)}{\pi(z,x)}f_{Z\mid X}(z\mid x)\Eb{(Y-\mu(z,x))^2\mid Z=z,X=x}} \xi_k \tag{$\Eb{R^2|X}=\Eb{R|X}=\pi(z,x)$}
\end{align*} 

%\textbf{TODO: change to $\alpha$ to $\bar{\alpha}$}

For $B_z$: 
\begin{align*}
    &\Bigg(\Eb{\Eb{\frac{K_h(Z - z) \bar{\alpha}(z,X)R}{\pi(z,X)}(Y-\mu(z,X)}}\Bigg)^2 = \Bigg(\Eb{\frac{\bar{\alpha}(z,X)R}{\pi(z,X)}\Eb{K_h(Z - z) (Y-\mu(z,X))}}\Bigg)^2 \\
    &= \Bigg(\Eb{\bar{\alpha}(z,X)\underbrace{\Eb{K_h(Z - z) (Y-\mu(z,X))\mid Z=z,R=1,X}}}\Bigg)^2 \tag{$\Eb{R|X}=\pi(z,x)$} \\ 
    &= \int_{\mathcal{Z}_0} K_h(s-z)(\bar{\mu}(z,X) - \mu(z,X))f_{Z|X}(s|x)ds \tag{$\bar{\mu}(z,X) = \Eb{Y\mid Z=z,R=1,X}$}\\
    &= \int_{\mathcal{Q}} k(u) (\bar{\mu}(z+uh,X) - \mu(z,X))f_{Z|X}(z+uh|x)du \tag{change of variables}\\
    &= \int_{Q} \bigg((\bar{\mu}(z,X) - \mu(z,X)) +  uh\frac{d}{d z} \bar{\mu}(z,X) +  \frac{u^2 h^2}{2} \frac{d^2}{d z} \bar{\mu}(z,X) \bigg)\\
    &\times \bigg( f_{Z|X}(z|x) + uh\frac{d}{d z} f_{Z|X}(z|x) + \frac{u^2 h^2}{2} \frac{d^2}{d z} f_{Z|X}(z|x) \bigg) \tag{taylor expansion} \\
    &\times k(u) du + O(h^3) \\ 
    &= \int_{Q} (\bar{\mu}(z,X) - \mu(z,X))f_{Z|X}(z|x)k(u) + h\bigg[\underbrace{(\bar{\mu}(z,X) - \mu(z,X))uk(u) \frac{d}{d z} f_{Z|X}(z|x)}_{\int uk(u)du =0} + \underbrace{f_{Z|X}(z|x) uk(u) \frac{d}{d z} \bar{\mu}_z(X)}_{\int uk(u)du=0} \bigg]  \\
    &+ h^2\bigg[ \frac{1}{2}(\bar{\mu}(z,X) -\mu(z,X)) u^2k(u)\frac{d^2}{d z} f_{Z|X}(z|x) + \frac{1}{2}u^2k(u) f_{Z|X}(z|x) \frac{d^2}{d z}\bar{\mu}(z,X)+ u^2k(u)\frac{d}{d z} \bar{\mu}(z,X) \frac{d}{d z} f_{Z|X}(z|x)\bigg] + O(h^3) \\
    &=(\bar{\mu}(z,X) - \mu_z(X))f_{Z|X}(z|x) + h^2\bigg[\frac{d}{d z} \bar{\mu}(z,X) \frac{d}{d z}f_{Z|X}(z|x) + \frac{1}{2}f_{Z|X}(z|x)\frac{d^2}{d z} \bar{\mu}(z,X) + \frac{1}{2} (\bar{\mu}(z,X) - \mu_z(X))\frac{d^2}{d z}f_{z|X}(z|x)\bigg] \\
    &\times \int_{-\infty}^{\infty} u^2k(u)du + O(h^3) \\
    &= \underbrace{\Eb{(\bar{\mu}(z,X) - \mu(z,X))f_{Z|X}(z|x)\bar{\alpha}(z,x)}^2}_{=0} + h^4 \Bigg(\bigg[ 2 \frac{d}{d z} \bar{\mu}(z,X)\frac{d}{d z} f_{Z|X}(z|x) + f_{Z|X}(z|x)\frac{d^2}{d z} \bar{\mu}(z,X) \\
    &= + (\bar{\mu}(z,X) - \mu(z,X))\frac{d^2}{d z}f_{Z|X}(z|x)\bigg]\kappa\Bigg)^2 \tag{$\kappa\equiv \int u^2k(u)du$} \\
    &= h^4 \Bigg(\bigg[ 2 \frac{d}{d z} \bar{\mu}(z,X)\frac{d}{d z} f_{Z|X}(z|x) + f_{Z|X}(z|x)\frac{d^2}{d z} \bar{\mu}(z,X) + (\bar{\mu}(z,X) - \mu(z,X))\frac{d^2}{d z}f_{Z|X}(z|x)\bigg]\kappa\Bigg)^2
\end{align*}
    
\endproof

\proof{Proof of \Cref{eqn-cts-tx-optimal}.}\label{proof:thm3}
The objective function arises from the asymptotic variance expression in \citep[Thm. 3]{colangelo2020double}; it follows readily from following their proof of Thm. 3 with our analysis of the asymptotic variance as in \Cref{prop-avar-ate}. The proof of the optimal solution follows our analysis in \Cref{thm-global-budget-solution} with a few slightly different expressions, discussed as follows. 

The Lagrangian can be written as follows: 
 
\begin{align*}
\mathcal{L} &=
    \int_{\mathcal{X}} \int_{Z_0} \frac{K_h^2(s-z) \bar{\alpha}^2(s, x) }{\pi(s, x)} \sigma^2(s,x) f_{Z X}(s, x) d s d x \\ &+\lambda \int \int (\pi(s,x) -B)f_{zx}(s,x) ds dx  + \nu \int\int (\pi(s,x) - 1) f_{Z X}(s,x) ds dx + \eta \int \int (-\pi(s,x))f_{Z X}(s,x) ds dx
\end{align*}

We can take the pointwise derivative w.r.t. $\pi(s,x)$ to obtain the FOC 

\begin{align*}
    \frac{\partial\mathcal{L}}{\partial \pi(s,x)} = \frac{-K^2_h(s-z)\bar{\alpha}^2(s,x)\sigma^2(s,x)}{\pi^2(s,x)}f_{Z X}(s,x) + (\lambda + \nu - \eta)f_{Z X}(s,x) = 0 
\end{align*}

Solving the FOC, we obtain 

\begin{align*}
    (\lambda + \nu - \eta) f_{Z X}(s,x) &= \frac{K^2_h(s-z)\bar{\alpha}^2(s,x)\sigma^2(s,x)}{\pi^2(s,x)}f_{Z X})(s,x) \\
    \sqrt{\pi^2(s,x)} &= \sqrt{\frac{K^2_h(s-z)\bar{\alpha}^2(s,x)\sigma^2(s,x)}{\lambda+\nu-\eta}} \\
    \pi^*(s,x) &= \sqrt{\frac{K^2_h(s-z)\bar{\alpha}^2(s,x)\sigma^2(s,x)}{\lambda+\nu-\eta}} 
\end{align*}

We can solve for $\lambda^*$ and set $\nu$ and $\eta$ to be zero: 

\begin{align*}
    \Eb{\sqrt{\frac{K^2_h(Z-z)\bar{\alpha}^2(Z,X)\sigma^2(Z,X)}{\lambda}}} = B \\
    \lambda^* = \frac{\Eb{K_h(Z-z)\sqrt{\bar{\alpha}^2(Z,X)\sigma^2(Z,X)}}^2}{B^2} 
\end{align*}

Then plug back into our optimal $\pi^*(Z,X)$, 

\[ \pi^*(Z,X) = \pi_{\lambda}(Z,X) = \sqrt{\frac{K^2_h(Z-z)\bar{\alpha}^2(Z,X)\sigma^2(Z,X)}{\frac{\Eb{K_h(Z-z)\sqrt{\bar{\alpha}^2(Z,X)\sigma^2(Z,X)}}^2}{B^2}}} = \frac{\sqrt{K^2_h(Z-z)\bar{\alpha}^2(Z,X)\sigma^2(Z,X)}}{\Eb{K_h(Z-z)\sqrt{\bar{\alpha}^2(Z,X)\sigma^2(Z,X)}}} B\]

We can check that $\frac{\partial \mathcal{L}}{\partial \pi^*} = 0$ 

\begin{align*}
    -\frac{K^2_h(Z-z)\bar{\alpha}^2(Z,X)\sigma^2(Z,X)}{\pi^2(Z,X)} + (\lambda + \nu - \eta) = 0 \\
    - \frac{K^2_h(Z-z)\bar{\alpha}^2(Z,X)\sigma^2(Z,X)}{\frac{K^2_h(Z-z)\bar{\alpha}^2(Z,X)\sigma^2(Z,X)}{\Eb{K_h(Z-z)\sqrt{\bar{\alpha}^2(Z,X)\sigma^2(Z,X)}}^2} B^2} + \frac{\Eb{K_h(Z-z)\sqrt{\bar{\alpha}^2(Z,X)\sigma^2(Z,X)}}^2}{B^2} = 0 \\ 
    - \frac{\Eb{K_h(Z-z)\sqrt{\bar{\alpha}^2(Z,X)\sigma^2(Z,X)}}^2}{B^2} + \frac{\Eb{K_h(Z-z)\sqrt{\bar{\alpha}^2(Z,X)\sigma^2(Z,X)}}^2}{B^2} = 0
\end{align*}
\endproof

\subsection{Estimation analysis}

\proof{Proof of \Cref{thm-asymptotic-convergence}.}\label{proof:thm2}

\textbf{Proof sketch.} 

The proof proceeds in two steps. The first establishes that the feasible AIPW estimator converges to the AIPW estimator with oracle nuisances. It follows from standard analysis with cross-fitting, in particular the variant used across batches. 

\paragraph{Preliminaries }
In the analysis, we write the score function as a function of $R$ in addition to other nuisance functions: 
$$
{\psi}_{z,i}(R_i, e,\pi,\mu)=\frac{\mathbb{I}[Z_i=z]R_i(Y_i-{\mu}_z(X_i))}{{e}_z(X_i)  {\pi}(z,X_i)} 
    +{\mu}_z(X_i)
    $$

The AIPW estimator can be rewritten as a sum over estimators within batch-$t$, fold-$k$, $\hat{\tau}^{(t,k)}_{AIPW}$, as follows: 
$$  \hat{\tau}_{AIPW} = 
    \sum_{t =1}^2
\sum_{k=1}^K
\frac{ n_{t,k} }{n}\sum_{(t,i)\in\mathcal{I}_k} \frac{1}{n_{t,k} }
\{
\hat {\psi}_{1,i}(R, \hat e, \hat\pi, \hat\mu) - \hat {\psi}_{0,i}(R,\hat e, \hat\pi, \hat\mu) \} 
= \sum_{t =1}^2
\sum_{k=1}^K \frac{ n_{t,k} }{n}
  \hat{\tau}^{(t,k)}_{AIPW}
  $$

We introduce an intermediate quantity. The realized treatments are sampled with probability $\hat\pi(X_i)$, $R_i \sim Bern(\hat\pi(Z_i,X_i))$. In the asymptotic framework, we study treatments $\tilde{R}$ sampled from an oracle mixture distribution over the two batches. 
$$  \tilde{\tau}_{AIPW} = 
    \sum_{t =1}^2
\sum_{k=1}^K
\frac{ n_{t,k} }{n}\sum_{(t,i)\in\mathcal{I}_k} \frac{1}{n_{t,k} }
\{
\hat {\psi}_{1,i}(\tilde R, \hat e, \hat\pi, \hat\mu) - \hat {\psi}_{0,i}(\tilde R,\hat e, \hat\pi, \hat\mu) \} 
      $$
We also denote the AIPW estimator with oracle nuisances, $ \hat{\tau}^*_{AIPW}$, as 
    \begin{align*}
 \hat{\tau}^*_{AIPW} &= 
 \sum_{t =1}^2
\sum_{k=1}^K
\frac{ n_{t,k} }{n}\sum_{(t,i)\in\mathcal{I}_k} \frac{1}{n_{t,k} }
\{
{\psi}_{1,i}(\tilde R_i,e, \pi, \mu) - {\psi}_{0,i}(\tilde R_i,e, \pi, \mu) \} 
    \end{align*}

% $$ \sum_z \E_n [ {\psi}_{z,i}(R, \hat e, \hat \pi, \hat \mu) ] - \E_n [ {\psi}_{z,i}(R,  e,  \pi,  \mu) ] = o_p(n^{-\frac 12}). $$

We study convergence within a batch$-t$, fold$-k$ subset; the decompositions above give that convergence also holds for the original estimators. 

The first step studies the limiting mixture distribution propensity arising from the two-batch process and shows that the use of the double-machine learning estimator (AIPW), under the weaker product error assumptions, gives that the oracle estimator is asymptotically equivalent to the oracle estimator where missingness follows the limiting mixture missingness probability. The latter of these is an independent triangular array or a weighted sum of two batchwise i.i.d. averages and follows a standard central limit theorem. We wish to show $\hat\tau_{AIPW}-\tilde{\tau}_{AIPW} \to_p o_p(n^{-\frac 12})$, i.e. that: 
$$ \sum_z \E_n [ {\psi}_{z,i}(R, \hat e, \hat \pi, \hat \mu) ] - \E_n [ {\psi}_{z,i}(\tilde{R},  e,  \pi,  \mu) ] = o_p(n^{-\frac 12}).$$

Next we show that the estimator with feasible nuisance estimators converges to the estimator with oracle knowledge of the nuisance functions 
$$ \sqrt{n} (\tilde{\tau}^{(t,k)}_{AIPW}
- \hat{\tau}^{*,(t,k)}_{AIPW})
\to_p 0.$$

The result follows by the standard limit theorem applied to the estimator with oracle nuisance functions. 

\textbf{Step 1 }

Throughout this proof, let
\[
\kappa_n:=\frac{n_1}{n},
\qquad
\kappa_{1,n}:=\kappa_n,
\qquad
\kappa_{2,n}:=1-\kappa_n.
\]
For $t\in\{1,2\}$, let $p_t^0(z,x)$ denote the population
counterpart of the implemented, possibly clipped, batch-$t$
annotation probability, so that
\[
p_1^0(z,x)=\pi_1(z,x),
\qquad
p_2^0(z,x)=\pi_{2,z}^\dagger(x).
\]
Define the finite-sample pooled oracle probability by
\[
\pi_n^\dagger(z,x)
=
\kappa_n p_1^0(z,x)
+
(1-\kappa_n)p_2^0(z,x).
\]
Because $\kappa_n\to\kappa$, we have
$\pi_n^\dagger\to\pi^\dagger$. For notational simplicity,
throughout this proof only, we write $\pi^\dagger$ for
$\pi_n^\dagger$ until the limiting calculation in Step~3.
% Let $\tilde{R}_i = \mathbb{I}[U_i \leq \pi^*(Z_i,X_i)].$
Let
$\widetilde R_i=\mathbb{I}[U_i\leq p^0_{T_i,i}]$,
using the same uniform variable as the realized draw
$R_i=\mathbb{I}[U_i\leq\widehat p_{T_i,i}]$.
Restricting attention to a single treatment value $z\in \{0,1\}$, 
we want to show that: 
\begin{align*}
& \sum_{t=1}^2 \sum_{k=1}^K \frac{n_{t, k}}{n} \sum_{(t, i) \in \mathcal{I}_k} \frac{1}{n_{t, k}}\left\{\hat{\psi}_{1, i}(\tilde{R}, \hat e, \hat \pi, \hat \mu)-\hat{\psi}_{1, i}(R, \hat e, \hat \pi, \hat \mu)\right\} \\
   & 
  = \sum_{t=1}^2 \sum_{k=1}^K \frac{n_{t, k}}{n} \sum_{(t, i) \in \mathcal{I}_k} \frac{1}{n_{t, k}}
    \left\{\frac{\mathbb{I}[Z_i=z] \tilde{R}_i  (Y_i-\hat \mu_z(X_i))}{ \hat e_z(X_i)  \hat \pi(z,X_i)}-\frac{\mathbb{I}[Z_i=z] R_i  (Y_i-\hat \mu_z(X_i))}{\hat e_z(X_i)  \hat \pi(z,X_i)} \right\} 
    = o_p(n^{-1/2}).
\end{align*}
Without loss of generality we further consider one summand on batch-$t$, fold-$k$ data, the same argument will apply to the other summands and the final estimator.

Note that by consistency of potential outcomes, for any data point we have that
$$
\frac{\mathbb{I}[Z_i=z] \tilde{R}_i  (Y_i-\hat \mu_z(X_i))}{ \hat e_z(X_i)  \hat \pi(z,X_i)}-\frac{\mathbb{I}[Z_i=z] R_i  (Y_i-\hat \mu_z(X_i))}{\hat e_z(X_i)  \hat \pi(z,X_i)} 
= \frac{\mathbb{I}[Z_i=z] (\tilde{R}_i-R_i)  (Y_i(z)-\hat \mu_z(X_i))}{ \hat e_z(X_i)  \hat \pi(z,X_i)}
$$

%For each batch $t=1,\hdots,T$ and fold $k=1,\hdots,K$, according to the CSBAE crossfitting procedure, we observe that conditional on $\mathcal{I}_{(-k)}$ for a given batch and the observed covariates, the summands (namely $R_i = \mathbb{I}[U_i \leq \hat\pi^{(-k)}(X_i)]$) are independent mean-zero. The final estimator will consist of the sum over batches and folds. 
Conditional on $\mathcal I_{(-k)}$, the fitted nuisance functions,
and the evaluation covariates, the annotation draws are independent
across observations in the evaluation fold. They are not individually
mean-zero; their conditional means and variances are handled below
through the common-uniform coupling.

We start by looking at the estimator over one batch $t$ and one fold $k$ and the rest follows for the other batches and folds.
% We will show this by Chebyshev's inequality. At a high level, consistency of nuisance estimates implies the second moments are small and overall the error term concentrates quickly. 

% First we bound the second moment. 
For one batch-fold block of size $m=n_{t, k}$, define the actual and oracle probabilities $p_{t, i}$ and $p_{t, i}^0$. We study a coupling of $R_i, \tilde{R}_i$ under the same uniform random variable.
Let
\[
R_i=\mathbb{I}[U_i\leq \hat p_{t,i}],
\qquad
\widetilde R_i=\mathbb{I}[U_i\leq p^0_{t,i}],
\qquad
\Delta_i=\widetilde R_i-R_i.
\]
Under the common-uniform coupling,
\[
\E[\Delta_i\mid\mathcal F_{-k},X_i,Z_i]
=p^0_{t,i}-\hat p_{t,i},
\qquad
\E[\Delta_i^2\mid\mathcal F_{-k},X_i,Z_i]
=|p^0_{t,i}-\hat p_{t,i}|.
\]
Writing
\[
Y_i-\hat\mu_z(X_i)
=
\{Y_i-\mu_z(X_i)\}
+
\{\mu_z(X_i)-\hat\mu_z(X_i)\},
\]
the first component has conditional mean zero and
\[
Var\!\left[
\frac{1}{\sqrt m}\sum_{i\in\mathcal I_{t,k}}
\frac{\mathbb{I}[Z_i=z]\Delta_i
      \{Y_i-\mu_z(X_i)\}}
     {\hat e_z(X_i)\hat\pi(z,X_i)}
\,\middle|\,\mathcal F_{-k}
\right]
\leq
C P_{t,k}\!\left[
|p_t^0-\hat p_t|\sigma_z^2
\right]
=o_p(1).
\]
The conditional mean of the second component is bounded by
\[
C\sqrt m\,
\|\hat p_t-p_t^0\|_{2,m}
\|\hat\mu_z-\mu_z\|_{2,m}
=o_p(1),
\]
and its conditional variance is bounded by
\(C\|\hat\mu_z-\mu_z\|_{2,m}^2=o_p(1)\).
Therefore the batch-fold score difference is \(o_p(m^{-1/2})\).

\textbf{Step 2 (feasible estimator converges to oracle)}

If we look at one term for one treatment and datapoint in the above (the rest follows for the others), we obtain the following decomposition into error and product-error terms: 

\begin{align*}
%  &  = \frac{Z_i R_i  (Y_i-\hat\mu_1(1,X_i))}{\hat e_1(X_i)  \hat\pi(1,X_i)} -  
%  \frac{Z_i R_i  (Y_i-\hat \mu_1(1,X_i))}{ e_1(X_i)  \pi(1,X_i)} 
%  + \frac{Z_i R_i  (Y_i-\hat \mu_1(1,X_i))}{ e_1(X_i)  \pi(1,X_i)} - \frac{Z_i R_i  (Y_i-\mu_1(1,X_i))}{ e_1(X_i)  \pi(1,X_i)}
% +(\hat\mu_1(1,X_i)-\mu_1(1,X_i))
% \\
&  \frac{Z_i \tilde{R}_i  (Y_i-\hat\mu_1(X_i))}{\hat e_1(X_i)  \hat\pi(1,X_i)} -  \frac{Z_i \tilde{R}_i  (Y_i-\mu_1(X_i))}{ e_1(X_i)  \pi(1,X_i)}
 +
 (\hat\mu_1(X_i)-\mu_1(X_i))  \\
&=   (\mu_1(X_i)-\hat\mu_1(X_i))
\left( \frac{Z_i \tilde{R}_i  }{ e_1(X_i)  \pi(1,X_i)} - 1 \right) 
    + Z_i \tilde{R}_i  (Y_i- \hat \mu_1(X_i))(\frac{1}{\hat e_1(X_i)  \hat\pi(1,X_i)} 
    - \frac{1}{ e_1(X_i)  \pi(1,X_i)})\tag{by  $\pm  \frac{Z_i \tilde{R}_i  (Y_i-\hat \mu_1(X_i))}{ e_1(X_i)  \pi(1,X_i)} $}
    \\
    &=    (\mu_1(X_i)-\hat\mu_1(X_i))
\left( \frac{Z_i \tilde{R}_i  }{ e_1(X_i)  \pi(1,X_i)} - 1 \right) 
    + Z_i \tilde{R}_i  (Y_i- \mu_1(X_i))(\frac{1}{\hat e_1(X_i)  \hat\pi(1,X_i)} 
    - \frac{1}{ e_1(X_i)  \pi(1,X_i)})\\
   & \qquad 
     + Z_i \tilde{R}_i  (\mu_1(X_i)-\hat\mu_1(X_i))(\frac{1}{\hat e_1(X_i)  \hat\pi(1,X_i)} 
    - \frac{1}{ e_1(X_i)  \pi(1,X_i)}) \tag{by $\pm Z_i \tilde{R}_i  \mu_1(X_i)(\frac{1}{\hat e_1(X_i)  \hat\pi(1,X_i)} 
    - \frac{1}{ e_1(X_i)  \pi(1,X_i)})$}\\
    &     =    (\mu_1(X_i)-\hat\mu_1(X_i))
\left( \frac{Z_i \tilde{R}_i  }{ e_1(X_i)  \pi(1,X_i)} - 1 \right) \\
& \qquad 
    + Z_i \tilde{R}_i  (Y_i- \mu_1(X_i))
    \left({\hat\pi(1,X_i)^{-1}}(
{\hat e_1(X_i)}^{-1} -   { e_1(X_i)^{-1}}) 
    + { e_1(X_i)}^{-1}({ \hat  \pi(1,X_i)}^{-1}-
 {   \pi(1,X_i)}^{-1})
    \right) \\
   & \qquad 
     + Z_i \tilde{R}_i  (\mu_1(X_i)-\hat\mu_1(X_i))
  \left({\hat\pi(1,X_i)^{-1}}(
{\hat e_1(X_i)}^{-1} -   { e_1(X_i)^{-1}}) 
    + { e_1(X_i)}^{-1}({ \hat  \pi(1,X_i)}^{-1}-
 {   \pi(1,X_i)}^{-1})
    \right)
    \tag{by $\pm \frac{1}{ e \hat\pi }$}
\end{align*}

We want to show that 
$$ \sqrt{n_{t,k}} (\hat{\tau}^{(t,k)}_{AIPW}
- \hat{\tau}^{*,(t,k)}_{AIPW})
\to_p 0$$

Now that we have written out this expansion for one datapoints, we can write out this expansion within a batch-$t$, fold-$k$ subset, and write out the cross-fitting terms for reference: 

% \resizebox{\textwidth}{!}{%
% $
% \begin{aligned}
\begin{align*}
&\sqrt{n_{t,k}}\left(\hat{\tau}_{A I P W}^{(t, k)}-\hat{\tau}_{A I P W}^{*,(t, k)}\right) \\
&=\frac{1}{\sqrt{n_{t,k}}} \sum_{i:(t,i)\in\mathcal{I}_k}   (\mu_1(X_i)-\hat\mu_1^{(-k)}(1,X_i))
\left( \frac{Z_i \tilde{R}_i  }{ e_1(X_i)  \pi(1,X_i)} - 1 \right) \\
&\quad +
\frac{1}{\sqrt{n_{t,k}}} \sum_{i:(t,i)\in\mathcal{I}_k}  Z_i \tilde{R}_i  (Y_i- \mu_1(X_i)) \times  
\\
& \qquad\qquad\qquad\qquad\qquad 
    \left({\hat\pi^{(-k)}(1,X_i)^{-1}}(
{\hat e_1^{(-k)}(X_i)}^{-1} -   { e_1(X_i)^{-1}}) 
    + { e_1(X_i)}^{-1}({ \hat  \pi^{(-k)}(1,X_i)}^{-1}-
 {   \pi(1,X_i)}^{-1})
    \right) \\
 &\quad + 
 \frac{1}{\sqrt{n_{t,k}}}\sum_{i:(t,i)\in\mathcal{I}_k} Z_i \tilde{R}_i  (\mu_1(X_i)-\hat\mu_1^{(-k)}(1,X_i)) \times \\
 & \qquad\qquad\qquad\qquad\qquad 
  \left({\hat\pi^{(-k)}(1,X_i)^{-1}}(
{\hat e_1^{(-k)}(X_i)}^{-1} -   { e_1(X_i)^{-1}}) 
    + { e_1(X_i)}^{-1}({ \hat  \pi^{(-k)}(1,X_i)}^{-1}-
 {   \pi(1,X_i)}^{-1})
    \right)
\end{align*}
%However, for the remainder of the analysis, we omit cross-fitting notation on the nuisance functions for convenience, but still note when we condition on each fold. 

\textbf{Bound for third term}: 

\begin{align*}
     &\frac{1}{\sqrt{n_{t,k}}} \sum_{i:(t,i)\in\mathcal{I}_k} Z_i \tilde{R}_i  (\mu_1(X_i)-\hat\mu_1^{(-k)}(X_i))({\hat\pi^{(-k)}(1,X_i)^{-1}}(
\hat{e}^{(-k)}_1(X_i)^{-1} - e_1(X_i)^{-1}) 
  \\
 & \qquad \qquad\qquad + { e_1(X_i)}^{-1}(\hat{\pi}^{(-k)}(1,X_i)^{-1}-\pi(1,X_i)^{-1}) \\
 &= \frac{1}{\sqrt{n_{t,k}}} \sum_{i:(t,i)\in\mathcal{I}_k} Z_i \tilde{R}_i \hat\pi^{(-k)}(1,X_i)^{-1}  (\mu_1(X_i)-\hat\mu^{(-k)}_1(X_i))(
{\hat e_1^{(-k)}(X_i)}^{-1} -   { e_1(X_i)^{-1}}) \\
    &\qquad \qquad\qquad+ Z_i \tilde{R}_i e_1(X_i)^{-1}(\mu_1(X_i)-\hat\mu^{(-k)}_1(X_i))( \hat{\pi}^{(-k)}(1,X_i)^{-1}-
 \pi(1,X_i)^{-1}) \\
 &\leq (\lambda_{\pi} + \nu_e)\frac{1}{\sqrt{n_{t,k}}} \sum_{i:(t,i)\in\mathcal{I}_k}  (\mu_1(X_i)-\hat\mu^{(-k)}_1(X_i))(
{\hat e_1^{(-k)}(X_i)}^{-1} -   { e_1(X_i)^{-1}})\\
&\qquad\qquad \qquad\qquad+ (\mu_1(X_i)-\hat\mu^{(-k)}_1(X_i))( \hat{\pi}^{(-k)}(1,X_i)^{-1}-
 \pi(1,X_i)^{-1}) \end{align*}
By positivity, the inverse maps are Lipschitz on the relevant
probability ranges. Hence the absolute value of the preceding
root-$n_{t,k}$ term is bounded by
\begin{align*}
C\sqrt{n_{t,k}}\,
\norm{\hat\mu_1^{(-k)}-\mu_1}_2
\left\{
\norm{\hat e_1^{(-k)}-e_1}_2
+
\norm{
\hat\pi^{\dagger,(-k)}(1,\cdot)
-\pi^\dagger(1,\cdot)
}_2
\right\}
=o_p(1)
\end{align*}
by \Cref{asn-product-error-rates}, since $n_{t,k}\asymp n$.

\textbf{Bound for the first term, pooled across batches.}

Let $n_k=n_{1,k}+n_{2,k}$ and suppose the folds are formed
proportionally within each batch, so that
$n_{1,k}/n_k=\kappa_n+O(n^{-1})$, where $\kappa_n=n_1/n$.
Define
\[
\pi_n^\dagger(z,x)
=
\kappa_n p_1^0(z,x)
+
(1-\kappa_n)p_2^0(z,x).
\]
Throughout this proof, the denominator $\pi^\dagger$ denotes
$\pi_n^\dagger$.

For $d_{z,k}(x)=\mu_z(x)-\hat\mu_z^{(-k)}(x)$, define
\[
T_{\mu,z,k}
=
\frac{1}{\sqrt{n_k}}
\sum_{t=1}^2
\sum_{i:(t,i)\in\mathcal I_k}
d_{z,k}(X_i)
\left\{
\frac{
\mathbb I[Z_i=z]\widetilde R_i
}{
e_z(X_i)\pi_n^\dagger(z,X_i)
}
-1
\right\}.
\]
Conditional on the training data,
\begin{align*}
&\sum_{t=1}^2
\kappa_{t,n}
\E\left[
\frac{
\mathbb I[Z_i=z]\widetilde R_i
}{
e_z(X_i)\pi_n^\dagger(z,X_i)
}
-1
\,\middle|\,
X_i,T_i=t
\right] \\
&\qquad=
\sum_{t=1}^2
\kappa_{t,n}
\left\{
\frac{p_t^0(z,X_i)}
{\pi_n^\dagger(z,X_i)}
-1
\right\}
=0.
\end{align*}
Thus the first term is centered after pooling the two batches.
The $O(n^{-1})$ fold-rounding difference is negligible under
root-$n$ scaling. By positivity,
\[
\operatorname{Var}
\left(
T_{\mu,z,k}
\mid\mathcal F_{-k}
\right)
\leq
C\norm{\hat\mu_z^{(-k)}-\mu_z}_2^2
=o_p(1).
\]
Therefore,
\[
T_{\mu,z,k}=o_p(1).
\]

\textbf{Bound for the second term}: 
We bound the second term following a similar argument as above.

Since $\E[Y_i-\mu_1(X_i)\mid Z_i=1,X_i]=0$ and
$\tilde R_i\perp Y_i\mid Z_i,X_i,T_i$, both normalized sums below
have conditional mean zero. Hence

\begin{align*}
&\operatorname{Var}\left[
\frac{1}{\sqrt{n_{t,k}}}
\sum_{i:(t,i)\in\mathcal{I}_k}
Z_i\tilde R_i(Y_i-\mu_1(X_i))
\hat\pi^{(-k)}(1,X_i)^{-1}
\left\{
\hat e_1^{(-k)}(X_i)^{-1}
-
e_1(X_i)^{-1}
\right\}
\,\middle|\,
\mathcal{I}_{(-k)},\{X_i\}
\right] \\
&\quad+
\operatorname{Var}\left[
\frac{1}{\sqrt{n_{t,k}}}
\sum_{i:(t,i)\in\mathcal{I}_k}
Z_i\tilde R_i(Y_i-\mu_1(X_i))
e_1(X_i)^{-1}
\left\{
\hat\pi^{(-k)}(1,X_i)^{-1}
-
\pi(1,X_i)^{-1}
\right\}
\,\middle|\,
\mathcal{I}_{(-k)},\{X_i\}
\right] \\
% &\E\bigg[\frac{1}{\sqrt{n_{t,k}}} \sum_{i:(t,i)\in\mathcal{I}_k} \bigg(Z_i \tilde{R}_i  (Y_i-\mu_1(X_i))
%   \left({\hat\pi^{(-k)}(1,X_i)^{-1}}(
% {\hat e_1^{(-k)}(X_i)}^{-1} - { e_1(X_i)^{-1}})\right)^2\mid \mathcal{I}_{(-k)},\{X_i\}\bigg] \\
%     &+ \E\bigg[\frac{1}{\sqrt{n_{t,k}}} \sum_{i:(t,i)\in\mathcal{I}_k} \bigg(Z_i \tilde{R}_i  (Y_i-\mu_1(X_i))\left({ e_1(X_i)}^{-1}({ \hat  \pi^{(-k)}(1,X_i)}^{-1}-
%  {   \pi(1,X_i)}^{-1})
%     \right)\bigg)^2\mid \mathcal{I}_{(-k)},\{X_i\}\bigg] \\
% & \qquad = \mathrm{Var}\bigg[\frac{1}{\sqrt{n_{t,k}}} \sum_{i:(t,i)\in\mathcal{I}_k} \bigg(Z_i \tilde{R}_i  (Y_i-\mu_1(X_i))
%   \left({\hat\pi^{(-k)}(1,X_i)^{-1}}(
% {\hat e_1^{(-k)}(X_i)}^{-1} - { e_1(X_i)^{-1}})\right)\mid \mathcal{I}_{(-k)},\{X_i\} \bigg] \\
% &+ \mathrm{Var}\bigg[\frac{1}{\sqrt{n_{t,k}}} \sum_{i:(t,i)\in\mathcal{I}_k} \bigg(Z_i \tilde{R}_i  (Y_i-\mu_1(X_i))\left({ e_1(X_i)}^{-1}({ \hat  \pi^{(-k)}(1,X_i)}^{-1}-
%  {   \pi(1,X_i)}^{-1})
%     \right)\mid \mathcal{I}_{(-k)},\{X_i\} \bigg] \\
&\qquad \leq
\frac{C}{n_{t,k}}
\sum_{i:(t,i)\in\mathcal{I}_k}
\Bigg[
\left\{
\hat e_1^{(-k)}(X_i)^{-1}
-
e_1(X_i)^{-1}
\right\}^{2}
+
\left\{
\hat\pi^{(-k)}(1,X_i)^{-1}
-
\pi(1,X_i)^{-1}
\right\}^{2}
\Bigg] \\
&\qquad =o_p(1).
\end{align*}
where $C<\infty$ follows from positivity and
$\sigma_1^2(X)\leq B_{\sigma^2}$. The final equality follows from
the $L_2$-consistency of $\hat e_1^{(-k)}$ and
$\hat\pi^{(-k)}$, because inversion is Lipschitz on probabilities
bounded away from zero. Conditional Chebyshev's inequality therefore
implies that each of the two already root-$n_{t,k}$-scaled components
is $o_p(1)$, and hence the second term is $o_p(1)$. This holds for
both treatment arms. 

The second and third terms are $o_p(1)$ batch by batch, while the
first term is $o_p(1)$ after pooling the two batches. Since the
number of folds is fixed and each fold has size proportional to
$n$, summing over folds and treatment arms gives
\[
\sqrt n
\left(
\widetilde\tau_{AIPW}
-
\widehat\tau^*_{AIPW}
\right)
=o_p(1).
\]

  Conditional on the batch split, the centered oracle scores form an
independent triangular array satisfying the Lindeberg condition by
\Cref{asn-positivity,asn-consistent-bounded}, and their average variance
converges to $V_{AIPW}$ because
$\kappa_n p_1^0(z,x)+(1-\kappa_n)p_2^0(z,x)
=\pi_n^\dagger(z,x)\to\pi^\dagger(z,x)$; hence the
Lindeberg--Feller central limit theorem and Slutsky's theorem, together
with Steps~1--2, imply
\[
\sqrt n\left(\hat\tau_{AIPW}-\tau\right)
\Rightarrow \mathcal N(0,V_{AIPW}).
\]  
%     Putting these results from Step 1 and Step 2 together, along with the fact that $\frac{n_{t,k}}{n_t}\to\frac{1}{K}$ and
% $\frac{n_{t,k}}{n}\to\frac{\kappa_t}{K}$,
% where $\kappa_1=\kappa$ and $\kappa_2=1-\kappa$, gives the theorem. 
\endproof

\section{Additional lemmas}
\subsection{Results appearing in other works, stated for completeness.}
\begin{lemma}[Conditional convergence implies unconditional convergence, from \citep{chernozhukov2018double}]
    Lemma 6.1. (Conditional Convergence implies unconditional) Let $\left\{X_m\right\}$ and $\left\{Y_m\right\}$ be sequences of random vectors. (a) If, for $\epsilon_m \rightarrow 0, \operatorname{Pr}\left(\left\|X_m\right\|>\epsilon_m \mid Y_m\right) \rightarrow_{\operatorname{Pr}} 0$, then $\operatorname{Pr}\left(\left\|X_m\right\|>\epsilon_m\right) \rightarrow 0$. In particular, this occurs if $E\left[\left\|X_m\right\|^q / \epsilon_m^q \mid Y_m\right] \rightarrow_{P r} 0$ for some $q \geq 1$, by Markov's inequality. (b) Let $\left\{A_m\right\}$ be a sequence of positive constants. If $\left\|X_m\right\|=O_P\left(A_m\right)$ conditional on $Y_m$, namely, that for any $\ell_m \rightarrow \infty$, $\operatorname{Pr}\left(\left\|X_m\right\|>\ell_m A_m \mid Y_m\right) \rightarrow_{P r} 0$, then $\left\|X_m\right\|=O_P\left(A_m\right)$ unconditionally, namely, that for any $\ell_m \rightarrow \infty$, $\operatorname{Pr}\left(\left\|X_m\right\|>\ell_m A_m\right) \rightarrow 0$.
\end{lemma}
\begin{lemma}[Chebyshev's inequality]
Let $X$ be a random variable with mean $\mu$ and variance $\sigma^2$. Then, for any $t>0$, we have

$$
{P}(|X-\mu| \geq t) \leq \frac{\sigma^2}{t^2}
$$

\end{lemma}

\begin{lemma}[Theorem 8.3.23 (Empirical processes via VC dimension), \citep{vershynin2018high}]\label{lemma-chaining-vc}
Let $\mathcal{F}$ be a class of Boolean functions on a probability space $(\Omega, \Sigma, \mu)$ with finite $V C$ dimension $\operatorname{vc}(\mathcal{F}) \geq 1$. Let $X, X_1, X_2, \ldots, X_n$ be independent random points in $\Omega$ distributed according to the law $\mu$. Then

$$
\mathbb{E} \sup _{f \in \mathcal{F}}\left|\frac{1}{n} \sum_{i=1}^n f\left(X_i\right)-\mathbb{E} f(X)\right| \leq C \sqrt{\frac{\operatorname{vc}(\mathcal{F})}{n}}
$$

\end{lemma}

\subsection{Lemmas}
\begin{lemma}[Convergence of $\hat \pi$]
\label{lemma-convergence-nuissance}
  Assume that with high probability, for some large constant $K$, 
$\norm{\hat e(X) - e(X)}_2 \leq K n^{-r_e},\norm{\hat \sigma^2(X)-\sigma^2(X)}_2 \leq K n^{-r_\sigma}$. Assume \Cref{asn-vc}. Assume that $\sigma^2(X) >0$ so that its inverse is bounded $1/\sigma^2(X) \leq \gamma_{\sigma}.$
Recall that \Cref{thm-global-budget-solution} gives that \begin{align*}
&\pi^*(z,X) 
% \\
% &
= 
 \sqrt{\frac{\sigma^2_z(X)}{  e_z^2(X)}}B
\left(
\Eb{ \mathbb{I}[Z=1 ]\sqrt{\frac{\sigma^2_1(X)}{  e_1^2(X)}} 
+ 
\mathbb{I}[Z=0]\sqrt{\frac{\sigma^2_0(X)}{  e_0^2(X)}} }\right)^{-1} 
\end{align*} 
Define $\hat\pi^*(z,x)$ to be a plug-in version of the above (with $\hat\sigma^2,\hat e$, and $\Enb{\cdot}$).
Then $$\norm{\hat\pi^*(z,X)-\pi^*(z,X)}_2 = o_p(n^{-\min (r_e, r_\sigma, 1/2)}),
\norm{\hat\pi^\dagger(z,\cdot)-\pi^\dagger(z,\cdot)}_2=o_p(1).$$
\end{lemma}

\proof{Proof.}
\label{proof:lemma-plug-in}
Let $a=\frac{\sigma^2_z(X)}{  e_z^2(X)}$, $b=\Eb{ \mathbb{I}[Z=1 ]\sqrt{\frac{\sigma^2_1(X)}{  e_1^2(X)}} 
+ 
\mathbb{I}[Z=0]\sqrt{\frac{\sigma^2_0(X)}{  e_0^2(X)}} }$.

Let $c=\frac{\hat\sigma^2_z(X)}{  \hat e_z^2(X)}$, $d=\Enb{ \mathbb{I}[Z=1 ]\sqrt{\frac{\hat \sigma^2_1(X)}{  \hat e_1^2(X)}} 
+ 
\mathbb{I}[Z=0]\sqrt{\frac{\hat \sigma^2_0(X)}{  \hat e_0^2(X)}} }$. 

Then $\norm{\pi^*(z,X)-\hat\pi^*(z,X)}_2 =\norm{a/b-c/d}_2.$ 

Positivity of $\sigma^2_z(X)$ gives the elementary equality that $\frac{a}{b} - \frac{c}{d} = \left(\frac{a - b}{b}\right) + \left(\frac{d - c}{d}\right)$. 

Therefore, by triangle inequality and boundedness, 
\begin{align}
&\norm{\pi^*(z,X)-\hat\pi^*(z,X)}_2 \leq 
\gamma_\sigma \norm{
\sqrt{ {\sigma^2 (X)}/{e^2(X)}}
-
\sqrt{{\hat \sigma^2 (X)}/{\hat e^2(X)}}
%{\sqrt{{\hat \sigma^2 (X)}/{\hat e^2(X)}}}
}_2 \nonumber
 \\
& 
+ 
\gamma_\sigma \abs{
\Enb{ \mathbb{I}[Z=1 ]\sqrt{\frac{\hat \sigma^2_1(X)}{  \hat e_1^2(X)}}+\mathbb{I}[Z=0]\sqrt{\frac{\hat \sigma^2_0(X)}{  \hat e_0^2(X)}} }-\Eb{ \mathbb{I}[Z=1 ]\sqrt{\frac{\sigma^2_1(X)}{  e_1^2(X)}} 
+ 
\mathbb{I}[Z=0]\sqrt{\frac{\sigma^2_0(X)}{  e_0^2(X)}} }
}    
\label{eq-decomposition}
\end{align}

Next we show that for $z\in\{0,1\},$
\begin{equation}
    \norm{\sqrt{\hat \sigma^2_z (X)/\hat e^2_z(X)}
-\sqrt{ \sigma^2_z (X)/ e^2_z(X)}}_2 \leq \nu_e B_{\sigma^2} 
(\norm{\sqrt{\hat \sigma^2_z (X)}- \sqrt{ \sigma^2_z (X)}  }_2
+\norm{{e_z(X)} - {\hat e_z(X)} }_2) \label{eqn-num-bound}
\end{equation} 
In the below, we drop the $z$ argument. 

    By the triangle inequality, boundedness of $1/\hat e(X) \leq \nu_e $, and of $\sigma^2(X) \leq B_{\sigma^2}$:
\begin{align*}
&    \norm{\sqrt{\hat \sigma^2 (X)/\hat e^2(X)}
-\sqrt{ \sigma^2 (X)/ e^2(X)}}_2\\
& =     \norm{\sqrt{\hat \sigma^2 (X)/\hat e^2(X)}
\pm \sqrt{ \sigma^2 (X)/ \hat e^2(X)}
-\sqrt{ \sigma^2 (X)/ e^2(X)}}_2
\\
& \leq 
\nu_e \norm{\sqrt{\hat \sigma^2 (X)}- \sqrt{ \sigma^2 (X)}  }_2
+B_{\sigma^2} \norm{\frac{1}{e(X)} - \frac{1}{\hat e(X)} }_2
\end{align*}
For the second term:
\begin{align*}
    B_{\sigma^2} \norm{ \frac{1}{e(X)} - \frac{1}{\hat e(X)} }_2 \leq  B_{\sigma^2} \norm{\frac{1}{e(X)} - \frac{1}{\hat e(X)} }_2
    \leq B_{\sigma^2} \nu_e \norm{ e(X) - \hat e(X) }_2 
\end{align*}
since $1/e(X)$ is Lipschitz on the assumed bounded domain (overlap assumption).

For the first term: 
$$
\nu \norm{\sqrt{\hat \sigma^2 (X)}- \sqrt{ \sigma^2 (X)}  }_2 \leq \nu_e B_{\sigma^2} \norm{\hat{\sigma}^2(X) - \sigma^2(X)}_2$$
since $\sigma^2(X)$ is bounded away from 0, then $\sqrt{\sigma^2(X)}$ is Lipschitz. 

This proves \Cref{eqn-num-bound}, which bounds the first term of \Cref{eq-decomposition}. For the second term, denote for brevity
$$ \hat\beta (\sigma, e) = \Enb{ \mathbb{I}[Z=1 ]\sqrt{\frac{ \sigma^2_1(X)}{   e_1^2(X)}}+\mathbb{I}[Z=0]\sqrt{\frac{ \sigma^2_0(X)}{ e_0^2(X)}} },$$
and $\beta(\sigma,e)$ to be the above with $\Eb{\cdot}$ instead of $\Enb{\cdot}.$ Then the second term of \Cref{eq-decomposition} is $\hat\beta (\hat\sigma, \hat e) - \beta (\sigma, e),$ and decomposing further, that
\begin{align*}
    \hat\beta (\hat\sigma, \hat e) - \beta (\sigma, e) = \hat\beta (\hat\sigma, \hat e) - \hat\beta (\sigma,  e) 
    + \hat \beta (\sigma, e)- \beta (\sigma, e).
\end{align*}
Note that by Cauchy-Schwarz inequality,  and \Cref{lemma-chaining-vc} (chaining with VC-dimension), 
$$\hat\beta (\hat\sigma, \hat e) - \hat\beta (\sigma,  e)  \leq 2 \nu_e B_{\sigma^2}\left(\left\|\sqrt{\hat{\sigma}_z^2(X)}-\sqrt{\sigma_z^2(X)}\right\|_2+\left\|e_z(X)-\hat{e}_z(X)\right\|_2\right) + 2C \sqrt{\frac{\operatorname{vc}(\mathcal{F}_{\sqrt{\frac{\sigma^2}{e}}})}{n}} $$
And another application of \Cref{lemma-chaining-vc} gives that 
\begin{align*}
    \hat \beta (\sigma, e)- \beta (\sigma, e) & = \textstyle  (\mathbb{E}_n-\E)\left[\mathbb{I}[Z=1] \sqrt{\frac{\sigma_1^2(X)}{e_1^2(X)}}+\mathbb{I}[Z=0] \sqrt{\frac{\sigma_0^2(X)}{e_0^2(X)}}\right]
    % \\
    % & \textstyle 
    \leq 2C \sqrt{\frac{\operatorname{vc}(\mathcal{F}_{\sqrt{\frac{\sigma^2}{e}}})}{n}}.
\end{align*}
Combining the above bounds with \Cref{eq-decomposition}, we conclude that $\left\|\pi^*(z, X)-\hat{\pi}^*(z, X)\right\|_2 = o_p(n^{-\min (r_e, r_\sigma, 1/2)}).$

For the feasible clipped rule, assume that the population budget
equation has a unique solution $c^\dagger$. Since the clipping map
$[\,\cdot\,]_0^1$ is $1$-Lipschitz and the empirical budget map
converges uniformly to its population counterpart, $\hat c\to_p
c^\dagger$. Consequently, $\norm{\hat\pi^\dagger(z,\cdot)-\pi^\dagger(z,\cdot)}_2=o_p(1)$.
\endproof

\section{Additional experiments, details, and discussion \label{appendix-data}}
\subsection{Additional details}\label{apx-additional-details} 

All experiments using our full \cref{alg:batch-allocation-full} were conducted on a 2021 13-inch MacBook Pro equipped with a 2.3 GHz Quad-Core Intel Core i7 processor and 32 GB of memory; or on a parallel computing cluster with 32 cores. This setup was used to train standard nuisance models using machine learning, evaluated our algorithm, and conduct the analysis tasks reported in this paper. The average compute time for the experiments on real world data with 20 trials was less than 30 minutes, while the simulated data with 100 trials took less than 60 minutes. Additionally, for all experiments, we allocate $55\%$ of the data to batch 1 and $45\%$ to batch 2.  

We run the ML nuisance models, logistic regression, random forest and support vectors machines, using popular Python packages (i.e. sklearn and scipy). We use logistic regression to estimate the propensity scores. For the outcome and variance models, we use random forest with the following hyperparameters:  max\_depth: None, min\_samples\_leaf: 4, min\_samples\_split: 10, n\_estimators: 100, random\_state: 42. 

We also use support vector machines for the outcome models incorporating LLM predictions, and we use the following hyperparameters: kernel: 'rbf', C: 1.

We chose these hyperparameters by doing a grid search over hyperparameters and chose the ones that performed the best. We ensemble predictions from the best performing random forest model trained on $X$ and the best performing SVM model trained on $X$ and $f(X,\tilde{Y})$ for our outcome model $\mu_z(X,\tilde{Y})$.

We run LLM calls on Together.AI since they provide enterprise-secure deployments of local models, which is required for sensitive data. Because we need to use local LLMs for the real-world street outreach data, we also use the same local LLMs for the other experiments. We use ``Llama-3.3-70B-Instruct-Turbo" for all experiments using LLMs. (Larger models provide effectively similar performance). 

To solve our optimization problem, we used the python package CVXPY and we specifically used the Splitting Conic Solver (SCS) solver. 

Once the experiments are run, we display the means and $95\%$ confidence interval bands, obtained through bootstrapping, in each of our figures.

\subsection{Synthetic data\label{appendix-synthetic-exp}}

Before running our batch adaptive algorithm, we split the data into a validation set ($35\%$ of data), which we use to estimate the oracle ATE. Then we use the remainder (65\%) of the data to run our algorithm, which splits that data into the two batches in the way we described previously.

\paragraph{Data Generating Process for Binary Treatment.} We generate a dataset $\mathcal{D} = \{X
% X_{1i},X_{2i},X_{3i},X_{4i},X_{5i}
,Z,Y,Y(1),Y(0) \}$, of size $1000$ and where the true ATE $\tau = \E[Y(1)] - \E[Y(0)] = 3$. We sample each covariate $X\in\mathbb{R}^5$ from a standard normal distribution, 
$X \sim \mathcal{N}(0,I_5)$. 
% $X \sim \mathcal{N}(0,1)$ for $j=1,\hdots,5$. 
Treatment $Z$ is drawn with logistic probability $\gamma_z(X) = (1 + e^{-{X_{2} + X_{3}+0.5}})$. We define $\sigma_{z}^2(X)$ as follows: 
\begin{align*} \sigma_{1}^2(X)&:= \max[1.3 + 0.4\mathrm{sin}(X_{1}),0] \\
 \sigma_{0}^2(X)&:= \max[3.5 + 0.3\mathrm{cos}(X_{3}),0]. 
 \end{align*}
Finally, the outcome models are defined as:
\begin{align*} Y(0) &= 5 + X_1 - 2X_2 + \epsilon_0\\
Y(1) &= Y(0) + \theta_0 + \epsilon_1,\end{align*}
where $\epsilon_0\sim\mathcal{N}(0,\sigma^2_0(X))$ and $\epsilon_1 \sim \mathcal{N}(0,\sigma^2_1(X))$. The observed outcomes are $Y = Z\cdot Y(1) + (1-Z)\cdot Y(0)$.

\paragraph{Data Generating Process for Continuous Treatment.} We generate a dataset of size 1000 following a similar process to the \cite{colangelo2020double}. We generate $v \sim \mathcal{N}(0,1)$ and  $\epsilon \sim \mathcal{N}(0,1)$, 
\begin{align*}
    X &= (X_1, ..., X_{100})^\prime \sim \mathcal{N}(0,\Sigma), \\
    Z &= \Phi(3X^\prime\theta) + 0.75\nu - 0.5, \\
    Y &= 1.2Z + 1.2X^\prime\theta + Z^2 + ZX_1 + \epsilon \cdot \sqrt{0.5+\Phi(X_1)}, 
\end{align*}

where $\theta_j = 1/j^2$, diag($\Sigma$)=1, the (i,j)-entry $\Sigma_{ij} = 0.5$ for $|i-j| = 1$ and $\Sigma_{ij} = 0$ for $|i-j| > 1$ for $i,j = 1, ..., 100,$ and $\Phi$ is the CDF of $\mathcal{N}(0,1)$. Thus the potential is $Y(z) = 1.2z + 1.2X^\prime\theta + z^2 + zX_1 + \epsilon \cdot \sqrt{0.5+\Phi(X_1)}$ and the parameters of interest are the average dose response function at varying values of $z$, i.e. $\beta_z = \Eb{Y(z)}$. We can solve for the true dose response function value by plugging in the value of z that we are evaluating. 

% Removed standalone simulated MSE / interval-width figures (previously Figs 8-9):
% they duplicated the simulated-data column of the main-text lineplot figure,
% differing only in the extended budget axis (to 0.9).

\begin{figure*}[ht!]
    \centering
    \includegraphics[width=\textwidth]{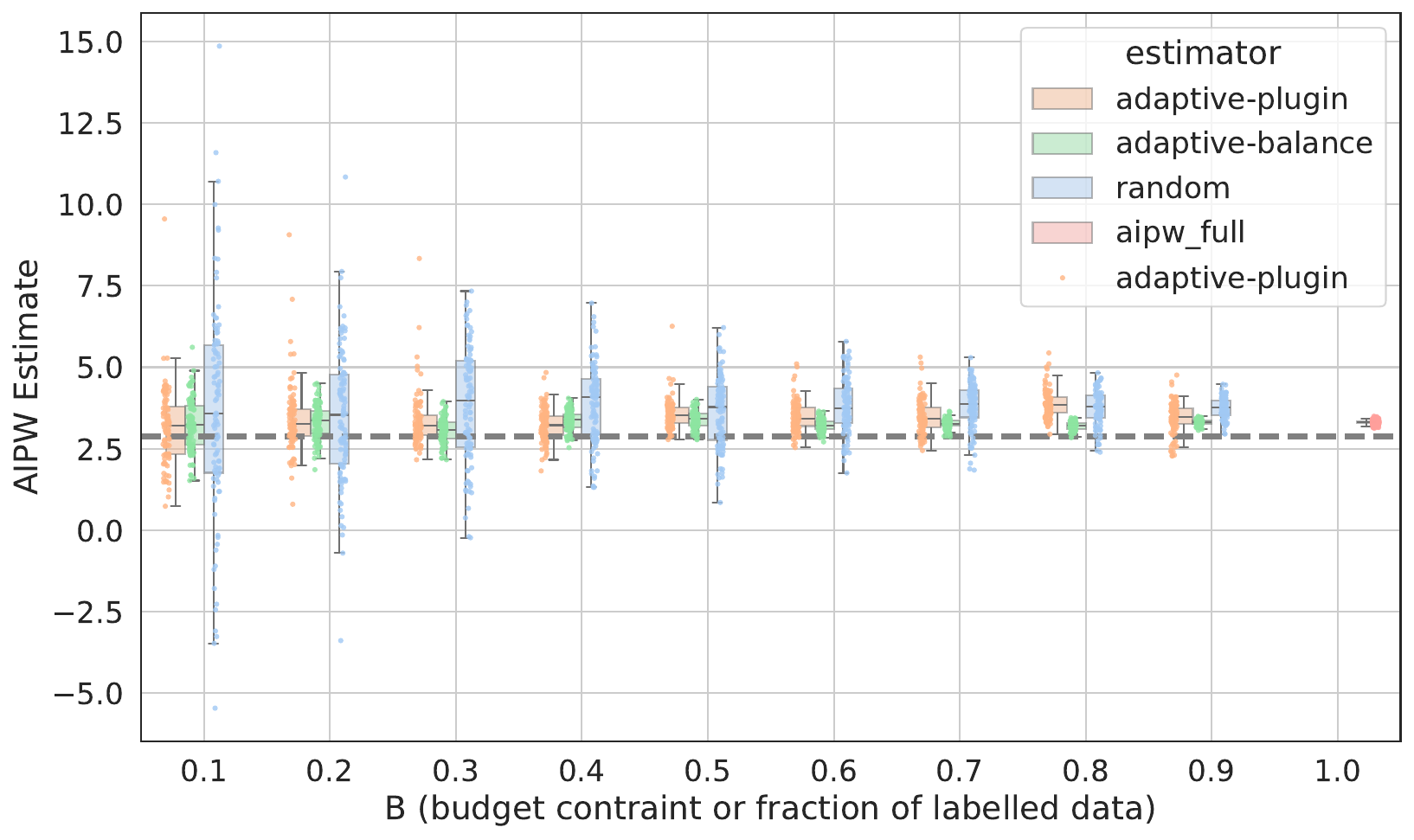}
    \caption{Boxplots of ATE estimates compared to skyline $\hat{\tau}_{AIPW}$ when the labeling budget is the entire dataset in red and the grey dotted line is $\tau$.}
    \label{fig:estimates_simulation_boxplot}
\end{figure*}

\paragraph{Results} The leftmost column of \Cref{fig:retailhero_lineplots} reports the mean squared error and average confidence interval width on the synthetic data across budgets, and \Cref{fig:estimates_simulation_boxplot} shows boxplots of the ATE estimates against the full-information skyline. We see the greatest advantage with our adaptive estimation for budgets between 0.1 and 0.4. While for larger budgets, even as the MSE for both estimators converge, the interval width for the adaptive estimator is still relatively small. Adaptive annotation with a larger budget introduces additional variation in inverse annotation probabilities, as compared to uniform sampling, which is equivalent to full-information estimation at a marginally smaller budget. This regime of improvement for small budgets is nonetheless practically relevant and consistent with other works. 

To stabilize the estimation of the inverse annotation probabilities, we use the plug-in estimator following \cref{eqn-RZ-estimation} and the ForestReisz method to estimate the balancing weights \citep{chernozhukov2022riesznet}.This approach provides an automatic machine learning debiasing procedure to learn the Reisz representer, or unique weights that automatically balances functions between treated and control groups using a random forest model.  

\subsection{Real-world dataset details \label{appendix:real-data-exp}} 

We provide further details about the treatment, covariates and outcomes for each dataset. \Cref{tab:dataset-descriptions-retailhero} and \cref{tab:dataset-descriptions-outreach} describe the variables in the retail hero and outreach datasets, respectively. We refer the reader to \cite{dhawanend} for further details about the dataset. For the outreach data, we constructed the binary treatment variable by binning the frequency of outreach engagements for each client within the first 6 months of the treatment period. We checked for overlap in propensity scores and decided to use treatments in the middle of the distribution as they had the most overlap. \Cref{fig:bg-treatment-dist} shows the distribution of street outreach engagements across the client population, which we binned to construct the treatment variable. Additionally, by \cref{corollary-rel-eff}, our method does well even when the propensity scores do not have good overlap.

\begin{table}[h!]
    \centering
    \small
    \renewcommand{\arraystretch}{1.2}
    \begin{tabular}{%
        p{0.15\textwidth}%
        p{0.50\textwidth}%
        p{0.33\textwidth}%
    }
        \hline
        \textbf{Variable} & \textbf{Description} & \textbf{Discrete Category} \\
        \hline
        \multicolumn{3}{l}{\textbf{Outcome}} \\[0.2em]
        Placement
          & The greatest housing placement attained by the client between 2019--2021
          & [3:permanent housing, 2: shelter/transitional housing, 1: other (e.g., hospital), 0: streets] \\[0.4em]
        \multicolumn{3}{l}{\textbf{Treatment}} \\[0.2em]
        Street outreach
          & Binned frequency of outreach within the first six months months of 2019
          & [Q0: 0, Q1: 1--2, Q2: 3--15, and Q4: 16--226] \\[0.4em]
        \multicolumn{3}{l}{\textbf{Covariates}} \\[0.2em]
        DateFirstSeen
          & Ordinal date when the client was first seen by the outreach team
          & NA \\[0.2em]
        Program
          & Outreach or service program the client belonged to
          & [Brooklyn Library, Grand Central Partnership, Hospital to Home, K-Mart Alley, Macy’s, MetLife, Penn Post Office, Pyramid Park, S2H Bronx, S2H Brooklyn, S2H Manhattan, S2H Queens, Starbucks, Superblock, Vornado, Williamsburg Stabilization Bed] \\[0.2em]
        BelievedChronic & Perceived by outreach workers as chronically homeless individual & [Yes, No] \\[0.2em]
        Gender & Perceived or disclosed gender of client & [Female, Male, Transgender] \\[0.2em]
        Race & Perceived or disclosed race of client & [American Indian/Alaskan Native, Asian, Black/African American, Native Hawaiian/Pacific Islander, White/Caucasian]\\[0.2em]
        Ethnicity & Perceived or disclosed ethnicity of client & [Hispanic/Latino, Non-hispanic/latino] \\[0.2em] 
        Age & Perceived or disclosed age range of client & [$<$ 30 years old, 30--50 years old, $>$ 50 years old] \\[0.2em] 
        Was311Call & Whether outreach workers were responding to a 311 city call & [Yes, No] \\[0.2em]
        Was911Call & Whether 911 was called to the scene & [Yes, No] \\[0.2em]
        Removal958 & Whether outreach workers were responding to removal hotline call & [Yes, No] \\[0.2em]
        Housing application & Whether any mention of the housing application was found in casenotes & [Yes, No] \\[0.2em] 
        Service refusal & Whether outreach worker documented that a client refused their services in casenotes & [Yes, No] \\[0.2em]
        Important documents & Whether there was mention of any important documents (i.e. social security card, drivers license, etc,) in casenotes & [Yes, No] \\[0.2em]
        Benefits & Whether there was any mention of social service benefits in the casenotes (i.e. foodstamps, SSI) & [Yes, No]\\[0.2em] 
        num contacts & number of engagements with an outreach worker prior to 2019 & NA \\[0.2em]
        max Placement & maximum housing placement reached before 2019 & [3:permanent housing, 2: shelter/transitional housing, 1: other (e.g., hospital), 0: streets] \\[0.2em]
        \hline
    \end{tabular}
    \caption{Covariates, treatment, and outcome descriptions and discrete category definitions for the Street Outreach dataset.}
    \label{tab:dataset-descriptions-outreach}
\end{table}

\begin{figure*}[h!]
    \centering
    \includegraphics[width=0.5\textwidth]{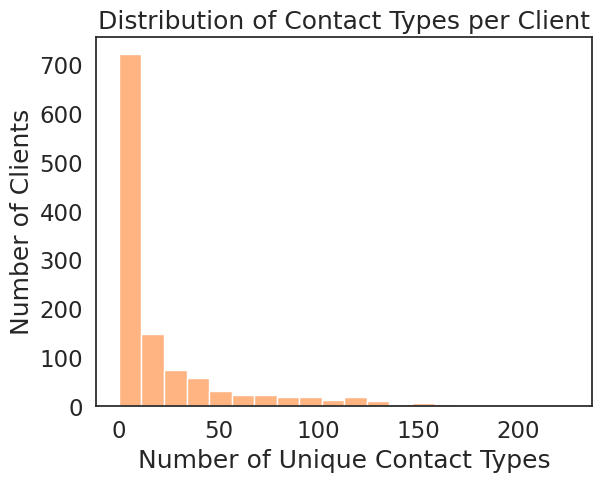}
    \caption{Distribution of street outreach engagements for client population.}
    \label{fig:bg-treatment-dist}
\end{figure*}

%\subsection{Additional Context on Street Outreach }

\subsection{Robustness check on Street Outreach data}

To further demonstrate the utility of our approach, we run experiments on the Street Outreach data with $\tilde{Y}$. To recap, our setup consists of covariates $X$, which includes client characteristics at baseline and LLM-generated summaries of case notes recorded before the treatment period. In the main text, we used LLMs to summarize casenotes prior to outreach during the interventional period, and used them in zero-shot prediction of later placement outcomes. Here we also incorporate LLM-generated summaries of case notes recorded post-treatment. These represent $\tilde{Y}$ in our framework. 

\begin{figure*}[ht!]
    \centering
    \includegraphics[width=0.9\textwidth]{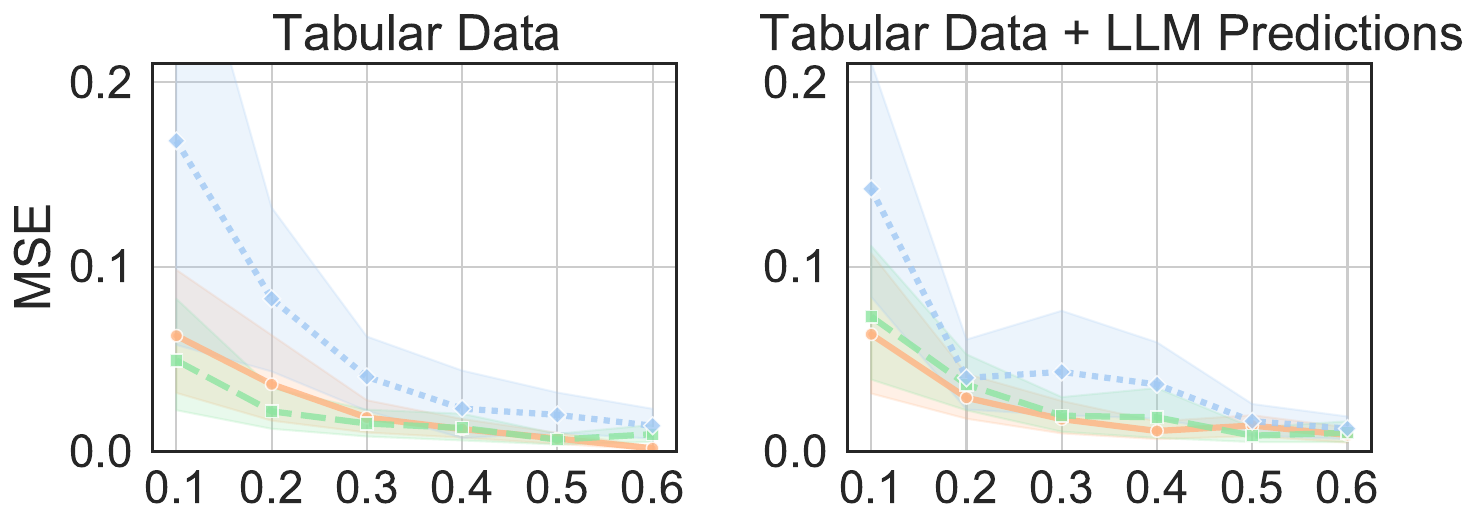}

    %\hspace{-0.1cm}
    \includegraphics[width=0.9\textwidth]{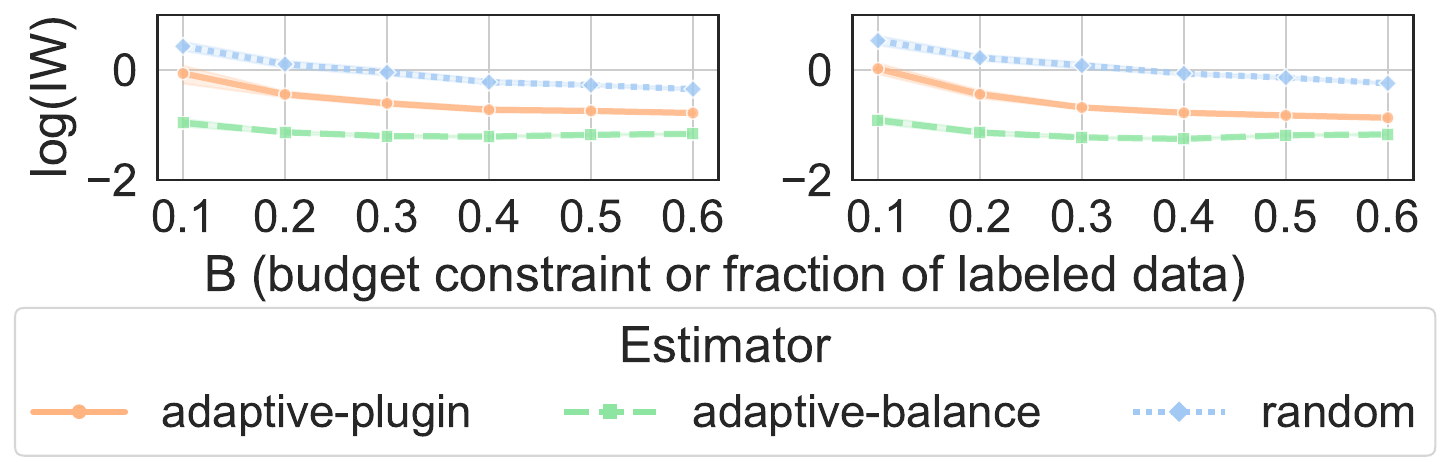}
    \caption{\textbf{Street Outreach Data with pretreatment summaries (no $\tilde{Y}$).} Mean squared error and $95\%$ confidence interval width averaged over 20 trials across budget percentages of the data. This plot makes use of tabular data and the best-performing random forest outcome model (left) and text-encoded outcomes using LLMs (right).} 
    \label{fig:outreach_robustness}
\end{figure*}

\begin{figure*}[ht!]
    \centering
    \includegraphics[width=0.9\textwidth]{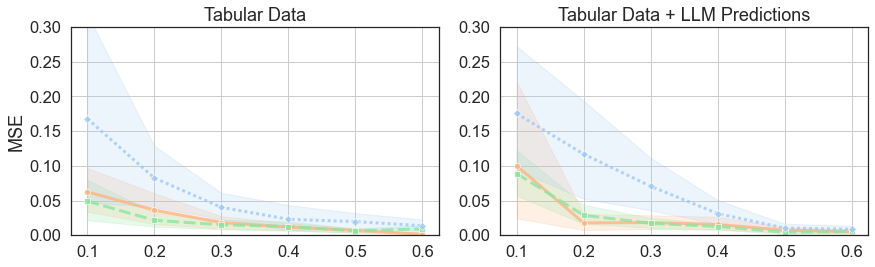}

    \includegraphics[width=0.9\textwidth]{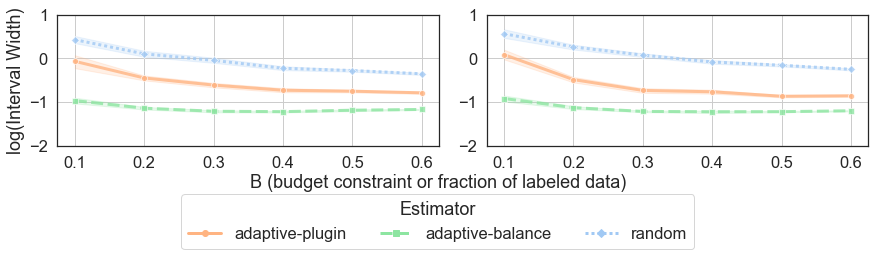}
    \caption{\textbf{Street outreach data with pre- and post-treatment summaries.} Mean squared error and $95\%$ confidence interval width averaged over 20 trials across budget percentages of the data. This plot makes use of tabular data and the best-performing random forest outcome model (left) and text-encoded outcomes using LLMs (right).} 
    \label{fig:outreach_robustness_pre_post}
\end{figure*}

In \Cref{fig:outreach_robustness} and \Cref{fig:outreach_robustness_pre_post}, we see that our results and analysis are preserved, and qualitatively similar. Our adaptive approach still shows improvements over uniform random sampling. The MSE is tripled when going from our adaptive estimators to random sampling in the tabular data. The MSE is five times higher when going from adaptive to random sampling in the setting where we have added LLM predictions using post-treatment summaries $\tilde{Y}$ only and it is nearly doubled when using both pre- and post-treatment summaries. 

In this experimental setup, we find that tabular estimation with ground-truth validated codes overall performs comparably as using more advanced LLM estimation. In this setup, we use placement outcomes as the measure of interest, in part because it is (nearly) fully recorded in our dataset, and hence we can consider it as having access to the ``ground-truth" outcome in our methodological setup. On the other hand, we also expect that casenotes are weakly informative of placement, as compared with other outcomes we might seek to extract from casenotes (but do not have the ground-truth for). Nonetheless, this validates the usefulness of the method, and we leave further empirical developments for future work. 
%\clearpage
\begin{revision}
\subsection{Progress analysis}\label{apx:pgs-full}

\paragraph{Further fine-tuning model analysis}

\Cref{tab:supp-band-error} complements the model-level validation metrics of \Cref{tab:gemma3-july-qwen-summary} by reporting signed error (prediction $-$ human label) by gold progress band on the same validation split, showing that fine-tuning primarily reduces under-coding of high-progress notes.

 \begin{table}[h!]
      \centering
      \small
      \renewcommand{\arraystretch}{1.3}
      \setlength{\tabcolsep}{5pt}
    \begin{revision}
      \begin{tabular}{p{0.19\textwidth} c c c}
          \hline
          \textbf{Model} & \textbf{Low ($n=299$)}
          & \textbf{High ($n=80$)} & \textbf{Overall} \\
          \hline
          G3 12B base
            & -0.12 [-0.17, -0.07] & -0.92 [-1.14, -0.71] &
  -0.29 [-0.35, -0.22] \\
          G3 12B FT
            & +0.05 [+0.00, +0.10] & -0.33 [-0.49, -0.18] &
  -0.03 [-0.08, +0.02] \\
          G4 31B FT
            & -0.05 [-0.11, +0.02] & -0.53 [-0.78, -0.31] &
  -0.15 [-0.22, -0.07] \\
          \hline
          Qwen3-235B ZS
            & -0.10 [-0.16, -0.06] & -0.78 [-1.02, -0.55] &
  -0.25 [-0.32, -0.18] \\
          \hline
      \end{tabular}
      \end{revision}
      \caption{Signed error (prediction \(-\) human label) by subgroup (low or high progress) on
      the same $n=379$ validation split, with 95\% bootstrap
  confidence intervals
      (4000~within-band resamples). Negative values indicate the model under-codes progress.
  Low is progress \(\leq 1.5\)
      (conversational / early-engagement notes); High is
  progress \(\geq 2\)
      (concrete-progress and milestone notes). Although every model under-codes progress, the fine-tune is least biased marginally and conditionally.}
      \label{tab:supp-band-error}
  \end{table}

\paragraph{Additional experimental results}

Three-panel versions (MSE, log interval width, coverage of the full-data
benchmark) for both treatment definitions. For llm-selection-adjusted, the balance arm is
scored against its own full-data benchmark as in the main text; coverage
bands are 95\% bootstrap intervals over the 200 trials.

\begin{figure}[h]
\centering
\includegraphics[width=\textwidth]{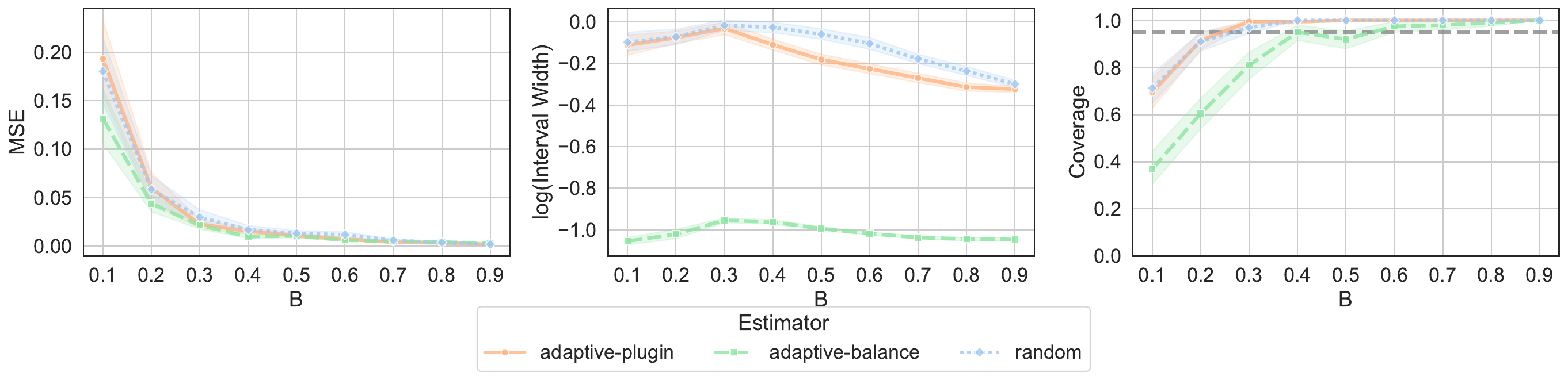}
\caption{Treatment: Q1 vs Q3, LLM-adjusted estimator, with coverage.}
\end{figure}

\begin{figure}[h]
\centering
\includegraphics[width=\textwidth]{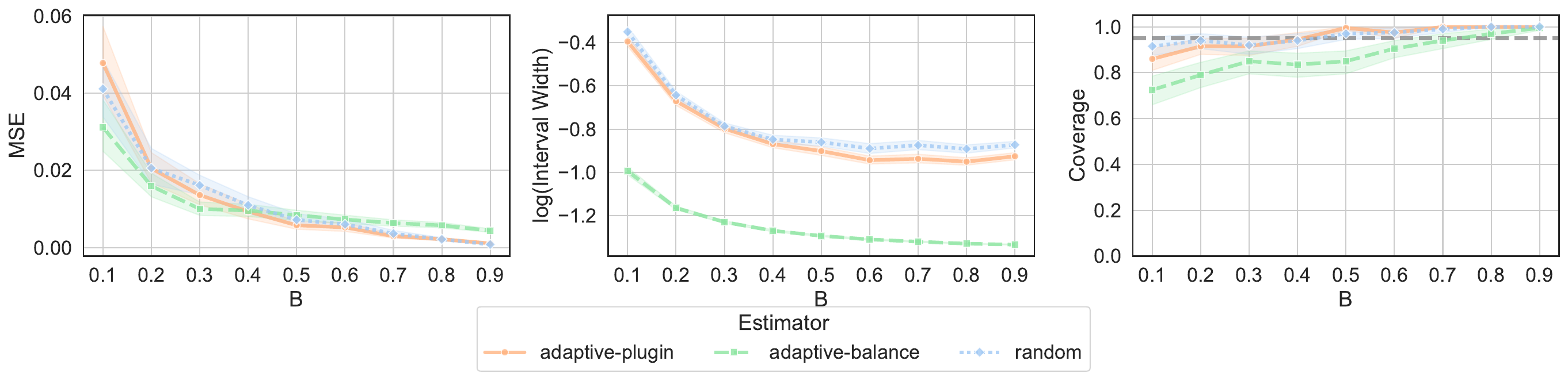}
\caption{Treatment: Q2 vs Q3, LLM-adjusted estimator,, with coverage.}
\end{figure}

\begin{figure}[h]
\centering
\includegraphics[width=\textwidth]{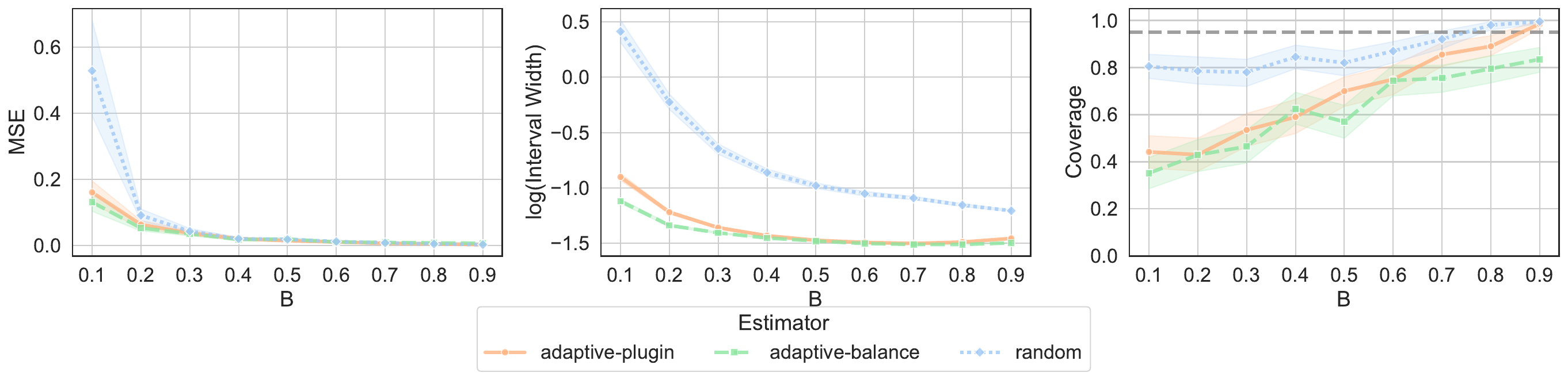}
\caption{Treatment: Q1 vs Q3, LLM-unadjusted estimator, with coverage (all arms scored against
the full-data benchmark, $0.295$).}
\end{figure}

\begin{figure}[h]
\centering
\includegraphics[width=\textwidth]{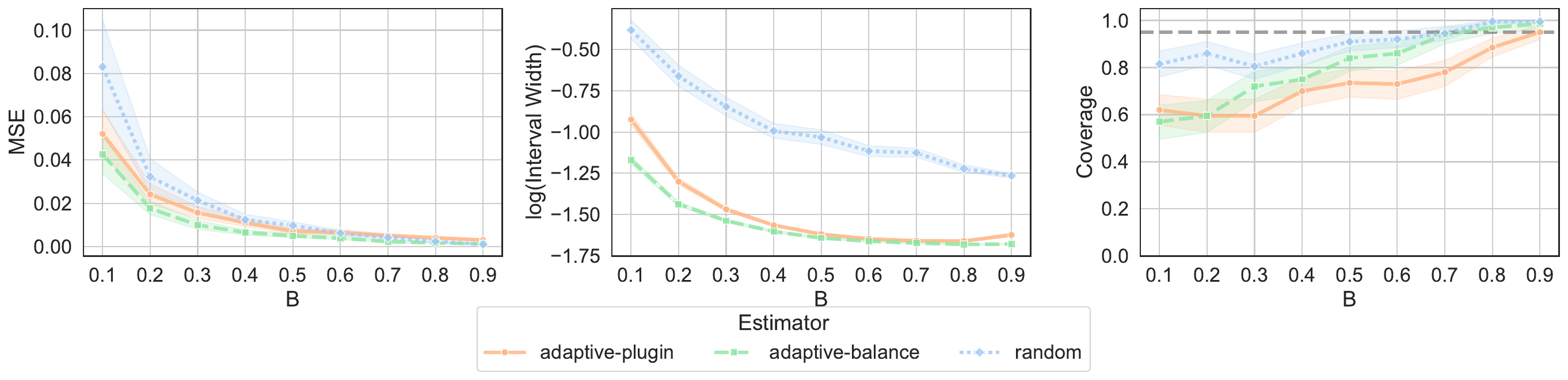}
\caption{Treatment: Q2 vs Q3, LLM-unadjusted estimator, with coverage (benchmark $0.049$).}
\end{figure}

\paragraph{Continuous treatment results}
\label{appendix:continuous-progress} We present the full data estimates of the average dose response function evaluated at fixed treatment values and percent shifted treatment values. We present the llm-selection-adjusted estimators and the llm-adjusted estimators with our adaptive-plugin and adaptive balancing estimators. We present three-panel versions that include the coverage of the full-data benchmark.

\begin{table}[tb]
\caption{Full-data ADRF estimates (all units annotated), plug-in weights. Point estimate with standard error in parentheses.}
\label{tab:cts_full_effect}
\begin{tabular}{lccc}
\toprule
 & 5\% more outreach & 10\% more outreach & z = 5 engagements \\
Outcome model &  &  &  \\
\midrule
LLM-adjusted & 2.016 (0.034) & 2.015 (0.034) & 2.009 (0.036) \\
LLM-selection-adjusted & 2.015 (0.034) & 2.013 (0.035) & 1.956 (0.038) \\
\bottomrule
\end{tabular}
\end{table}

\begin{table}[tb]
\caption{Full-data ADRF estimates (all units annotated), balance weights. Point estimate with standard error in parentheses.}
\label{tab:cts_full_effect}
\begin{tabular}{lccc}
\toprule
 & 5\% more outreach & 10\% more outreach & z = 5 engagements \\
Outcome model &  &  &  \\
\midrule
LLM-adjusted & 2.016 (0.034) & 2.015 (0.034) & 2.009 (0.036) \\
LLM-selection-adjusted & 2.015 (0.034) & 2.013 (0.035) & 1.955 (0.038) \\
\bottomrule
\end{tabular}
\end{table}

\begin{figure}[h]
\centering
\includegraphics[width=\textwidth]{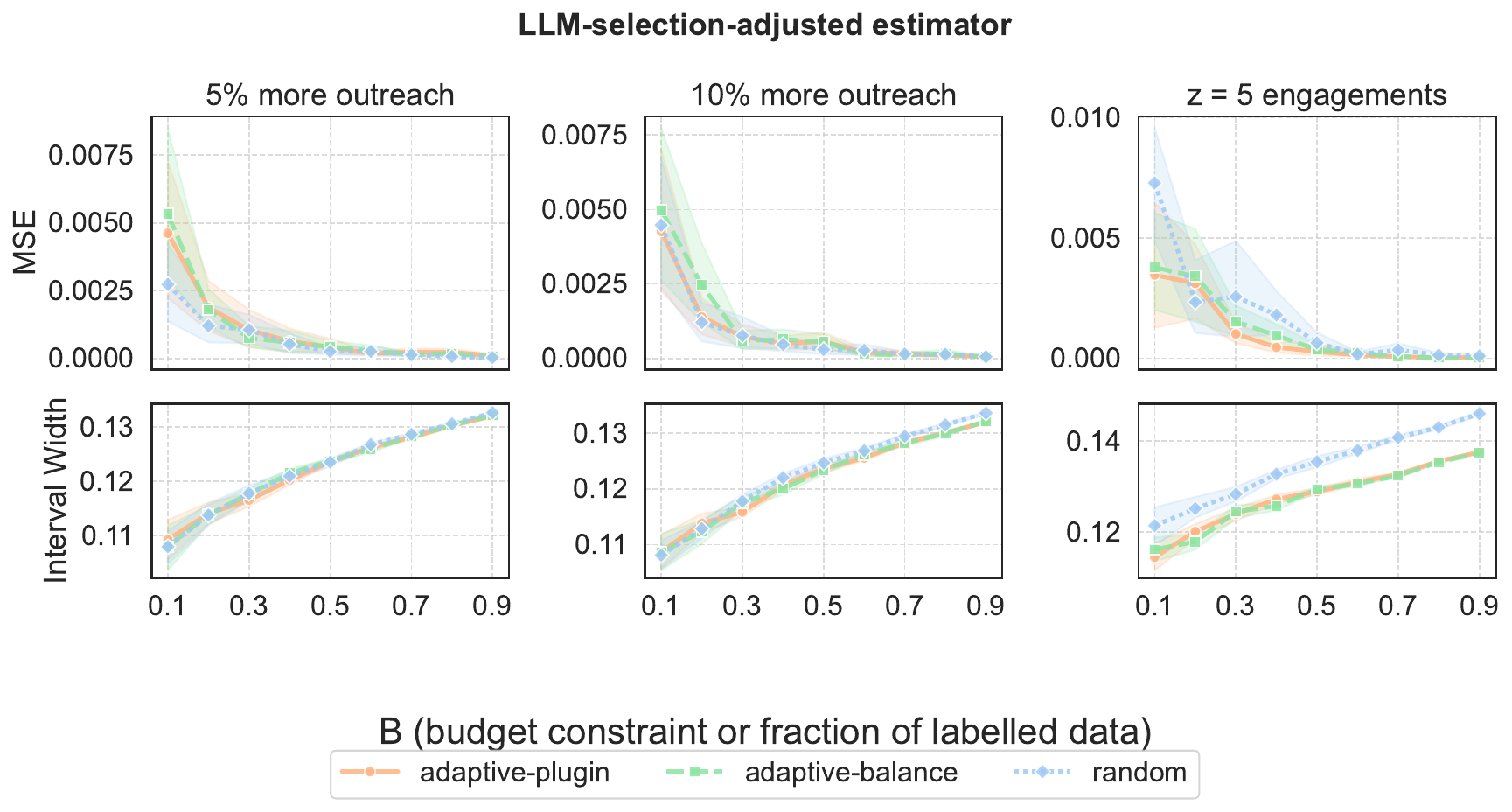}
\caption{LLM-selection-adjusted estimator. MSE and interval width vs annotation budget B for continuous treatments evaluated at $5\%$ more outreach, $10\%$ more outreach and a fixed z value at 5}
\end{figure}

\begin{figure}[h]
\centering
\includegraphics[width=\textwidth]{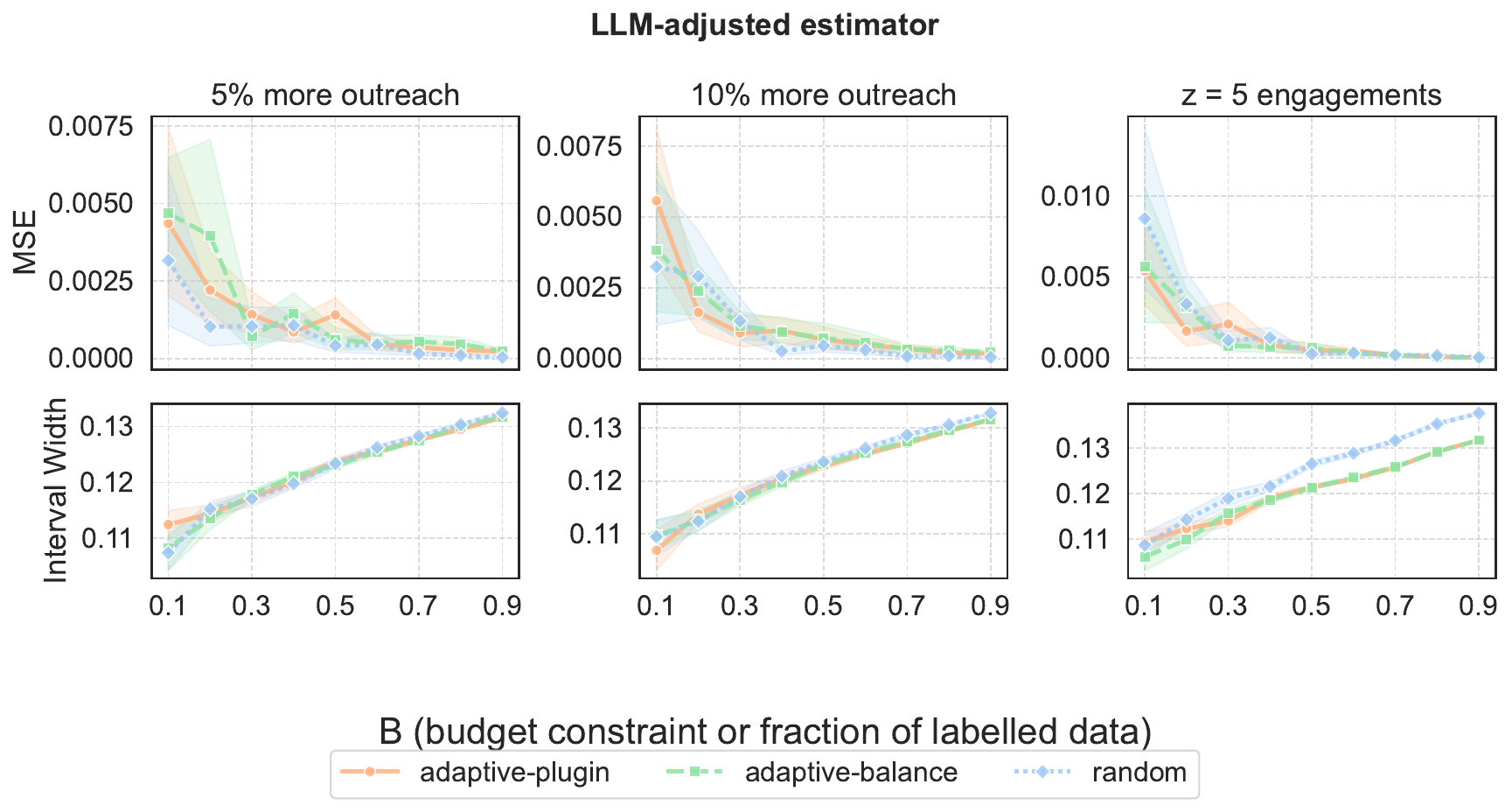}
\caption{LLM-adjusted estimator. MSE and interval width vs annotation budget B for continuous treatments evaluated at $5\%$ more outreach, $10\%$ more outreach and a fixed z value at 5}
\end{figure}

\begin{figure}[h]
\centering
\includegraphics[width=\textwidth]{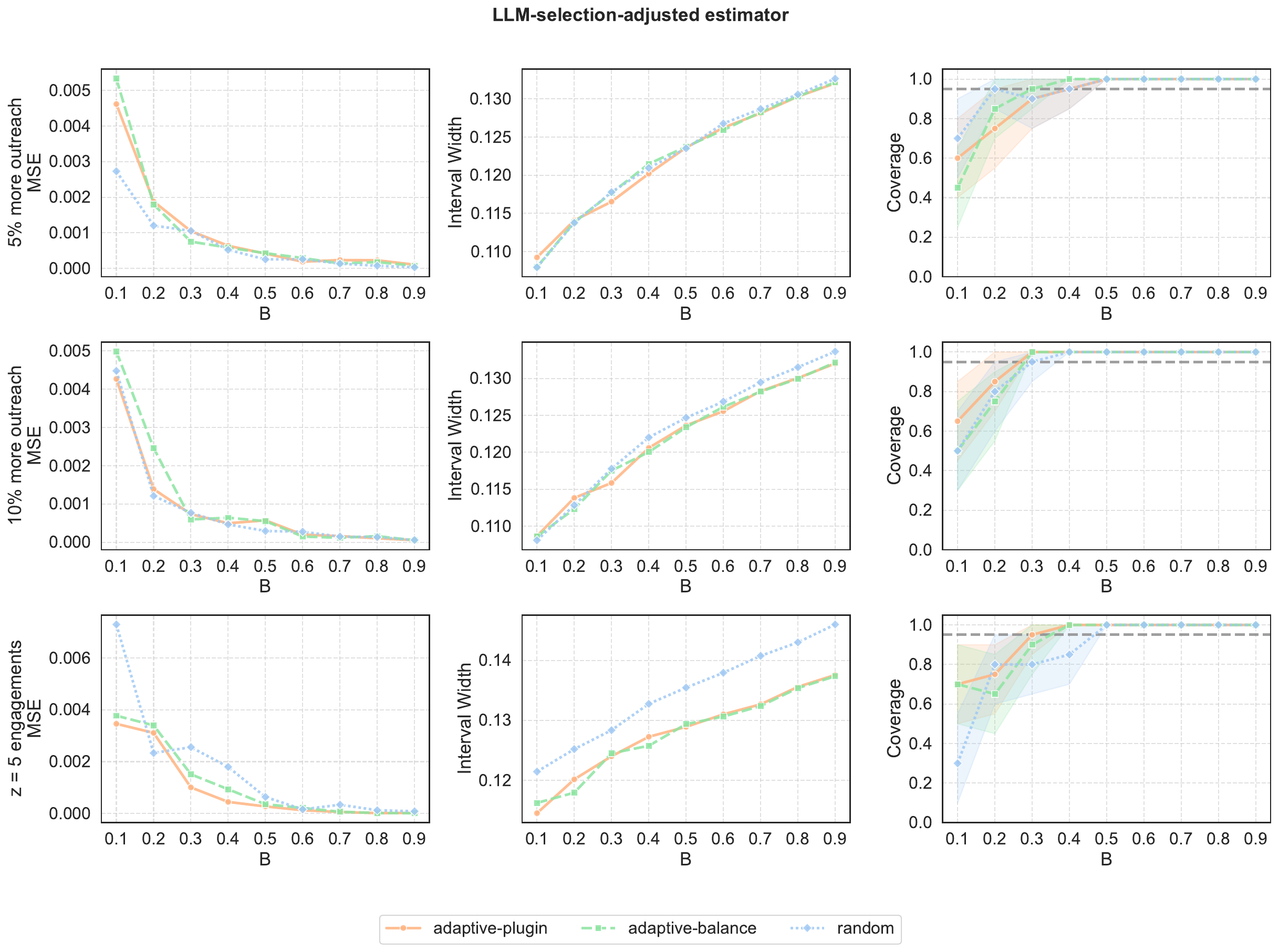}
\caption{LLM-selection-adjusted estimator, with coverage (all arms scored against
the full-data benchmark, $1.95$ for 5$\%$ more, $2.012$ for 10$\%$ more and $2.014$ for fixed z=5).}
\end{figure}

\begin{figure}[h]
\centering
\includegraphics[width=\textwidth]{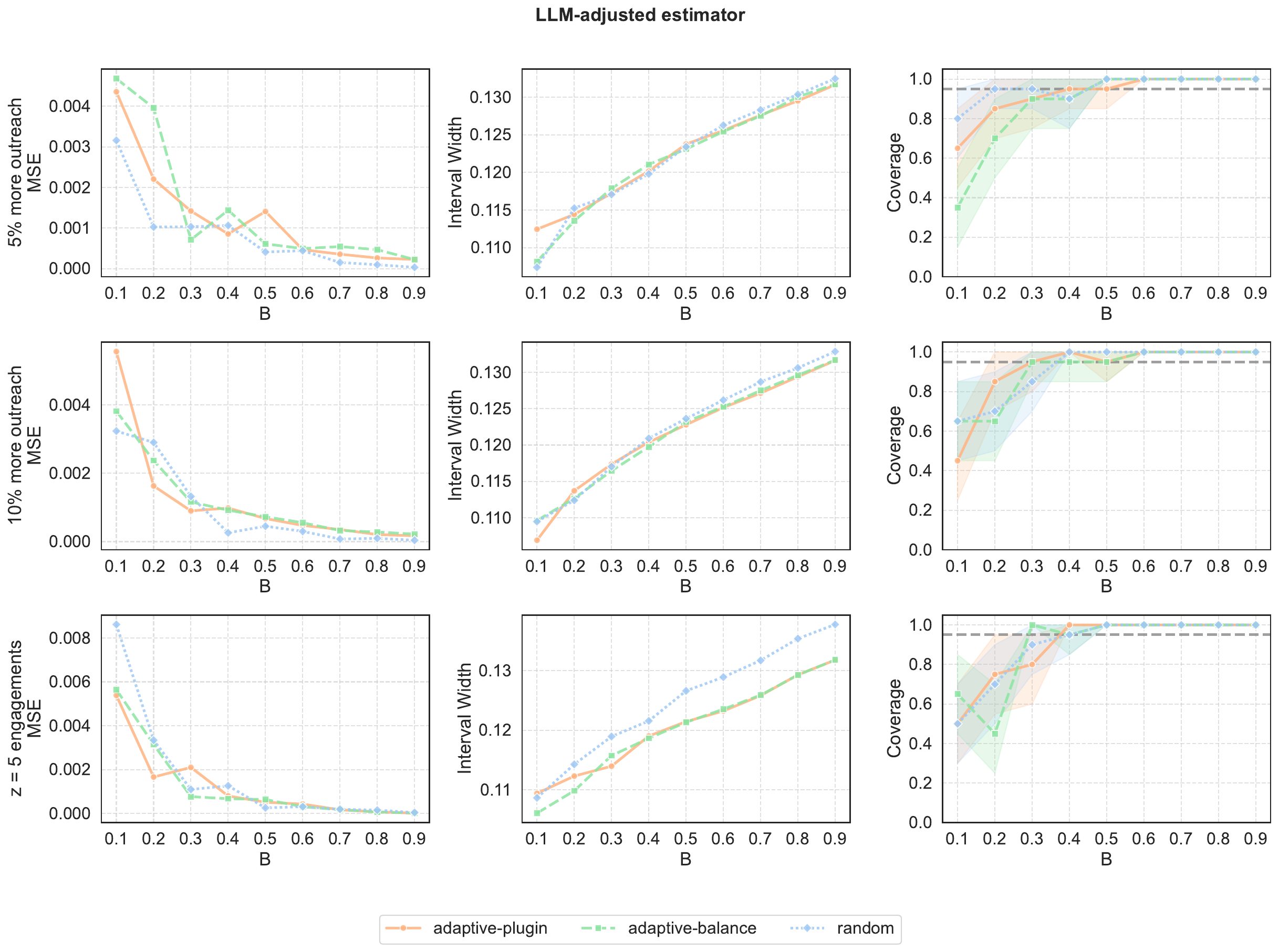}
\caption{LLM-adjusted estimator, with coverage (all arms scored against
the full-data benchmark, $2.016$ for 5$\%$ more, $2.015$ for 10$\%$ more and $2.008$ for fixed z=5).}
\end{figure}

\end{revision}
\newpage 
\subsection{LLM prompts}
\begin{promptbox}[Retail Hero Prompt]
\noindent
You are a user who used a website for online purchases in the past one year and want to share your background and experience with the purchases on
social media.

Attributes: The following are attributes that you have, along with their descriptions.

$\{features\}$

Personality Traits: The following dictionary describes your personality with levels (High or Low) of the Big Five personality traits.

$\{traits\}$

\textit{Your Instructions:}

Write a social media post in first-person, accurately describing the
information provided. Write this post in the tone and style of someone
with the given personality traits, without simply listing them.

Only return the post that you can broadcast on social media and nothing more.

---

$\{post\}$ 

---

\end{promptbox}

\begin{promptbox}[Street Outreach Casenote Summaries Prompt]
\noindent
Objective:
    Your task is to summarize a trajectory of case notes of a client in street homelessness outreach, focusing on client interactions, the challenges they are facing, goals they are working towards, and progress towards housing placement. These are all from the same client. This summary is designed to help caseworkers and organizations assess client history at a glance, remind of prior personal information and important challenges mentioned (like veteran status or other information that is relevant for eligibility for housing, medical issues, and status of their support network), allocate resources effectively, and improve support for individuals experiencing chronic homelessness.

Context:
$\{task\_context\}$

The summary should be a concise overview of the client's situation, highlighting key points from the case notes. It should not include any personal opinions or assumptions about the client's future or potential outcomes. The goal is to provide a clear and informative summary that can be used by caseworkers and organizations to better understand the client's history and current status.

Here are the case notes for batch $\{batch\_num\}$  of $\{total\_batches\}$:

--- START NOTES ---

$\{notes\}$

--- END NOTES ---

Based *only* on the notes provided above for this batch, generate a comprehensive summary focusing on key events, decisions, and progress during this specific period. The target length is approximately $\{target\_length\}$ words. Ensure the summary strictly reflects the content of these notes.

\end{promptbox}

\begin{promptbox}[Street Outreach Classification]
\noindent
You are an expert analyst specializing in predicting long-term housing stability for individuals experiencing homelessness. Your task is to analyze client data, including demographic information, historical interactions, and case note summaries, to predict the **most stable housing placement level** the client is likely to achieve and maintain over the **next two years**.

**Input Data:**

You will be provided with the following information for each client:
    
**Prediction Task:**

Based *only* on the provided attributes and the case notes summary, predict the single most stable housing placement level the client is likely to maintain over the next two years.

**Housing Placement Levels (Prediction Output):**

Your prediction must be an integer between 0 and 3:
\begin{itemize}
    \item **0**: No stable placement (remains on the street or in emergency shelters).
    \item **1**: Transitional Housing (temporary placement with support, aiming for longer-term housing).
    \item **2**: Rapid Re-housing (time-limited rental assistance and services).
    \item **3**: Permanent Supportive Housing (long-term housing with ongoing support services).
\end{itemize}

**Reasoning Guidance (Internal Thought Process - Do Not Output This):**
\begin{itemize}
    \item Consider factors that promote stability: housing application progress, possession of documents, benefit acquisition, engagement with services (unless contacts are excessive without progress), prior successful placements (even if temporary), positive recent developments in the case notes.
    \item Consider factors that hinder stability: chronic homelessness indicators, frequent service refusals, mental health crises (Removal958), lack of documents/income, lack of prior placements, patterns of instability noted in the summary.
    \item Weigh the structured data against the nuances presented in the case note summary. The summary provides vital context.
\end{itemize}

    **Client Information:**

    **Prediction:**

    Provide *only* the predicted number (0, 1, 2, or 3) as the output. Do not include any other text, explanation, or formatting.

    **Examples:**
    $\{examples\}$
\end{promptbox}

\begin{promptbox}[Street Outreach Progress Classification Prompt]
\noindent
\textbf{System message:}

\medskip
\noindent
Breaking Ground Progress Classification Schema - Prior

You are an expert caseworker assistant specializing in analyzing outreach and case-management notes for clients experiencing chronic homelessness.

Classify one case note at a time. Assign exactly one progress label from this allowed set:

\texttt{-1}, \texttt{0}, \texttt{0.25}, \texttt{0.5}, \texttt{0.75}, \texttt{1}, \texttt{1.5}, \texttt{2}, \texttt{2.5}, \texttt{3}, \texttt{3.5}, \texttt{4}

Return valid JSON only:

\texttt{\{"label":"0.5","confidence":"medium"\}}

Do not quote or reproduce the note. Do not include client identifiers.

\textbf{Prior Progress Scale}
\begin{itemize}
    \item \texttt{-1}: Regression or challenges, such as relapse, hostility, deterioration, major barriers, or concerning disappearance.
    \item \texttt{0}: No progress; no contact or no substantive interaction.
    \item \texttt{0.25}: Failed phone/text contact attempt.
    \item \texttt{0.5}: General conversation or brief check-in without services, appointments, or housing discussion.
    \item \texttt{0.75}: Client engagement, interest, self-disclosure, emotional engagement, or willingness to talk.
    \item \texttt{1}: Appointment or service discussed generally.
    \item \texttt{1.5}: Client need met, referral/help/supplies/transport/problem-solving provided.
    \item \texttt{2}: Concrete planning for appointment/service/housing, with specific coordination or next steps.
    \item \texttt{2.5}: Tangible steps or paperwork completed toward appointment/service/housing.
    \item \texttt{3}: Appointment completed, documentation received, interview completed, or goal achieved short of new permanent housing.
    \item \texttt{3.5}: Notable improvement in client condition.
    \item \texttt{4}: New permanent housing placement.
\end{itemize}

Choose the single best label supported by the note.

\textbf{Clarification Appendix V3}

Use these clarifications as appendages to the prior progress scale above. They do not replace the label definitions; they only resolve common boundary cases.

\textbf{No Progress vs Failed Contact Attempt}
\begin{itemize}
    \item Keep \texttt{0} as the default for no progress, no client contact, no client presence, or no substantive interaction.
    \item Use \texttt{0.25} only when the note clearly documents a specific direct contact attempt that failed, such as an unanswered phone call, text message, voicemail, direct message, or explicitly attempted direct outreach to the client.
    \item Do not use \texttt{0.25} merely because the client was absent, not seen, not found, unavailable, or there was no progress. If the note does not document a specific attempted direct contact, choose \texttt{0}.
    \item If a note includes both no substantive interaction and an explicit failed direct phone/text/message/contact attempt, choose \texttt{0.25}.
\end{itemize}

\textbf{Brief Contact vs Client Engagement}
\begin{itemize}
    \item Use \texttt{0.5} for brief contact, a greeting, a welfare check, a routine check-in, or general conversation that does not include meaningful client self-disclosure, a specific service need, an appointment, housing, documentation, or a next step.
    \item Use \texttt{0.75} when the client engages beyond brief contact by sharing meaningful personal information, emotional state, preferences, barriers, willingness, interest, or other substantive information, but no specific service, appointment, housing step, documentation need, or concrete goal is discussed.
    \item If the interaction is polite or positive but only general, choose \texttt{0.5} rather than \texttt{0.75}.
\end{itemize}

\textbf{Service Discussion vs Concrete Planning vs Completed Steps}
\begin{itemize}
    \item Use \texttt{1} when a specific service, appointment, documentation need, housing step, or goal is discussed, but the note does not document concrete logistics, coordination, paperwork, or completion.
    \item Use \texttt{1.5} when staff provide immediate help, referral information, supplies, transportation, problem-solving, or other support, but there is no concrete plan, completed paperwork, completed appointment, or documented service milestone.
    \item Use \texttt{2} when the note documents concrete planning or coordination, such as a date, time, location, reminder plan, follow-up plan, scheduling logistics, named next step, or coordination with another party.
    \item Use \texttt{2.5} when a tangible administrative step is completed toward the goal, such as forms, signatures, paperwork, document submission, intake steps, or comparable concrete progress.
    \item Use \texttt{3} when the appointment, documentation receipt, interview, service milestone, or goal is completed.
\end{itemize}

\textbf{Tie-Breaking}
\begin{itemize}
    \item Choose the highest label clearly supported by concrete evidence in the note.
    \item When deciding between \texttt{0} and \texttt{0.25}, choose \texttt{0} unless a specific failed direct contact attempt is documented.
    \item When deciding between adjacent labels, choose the lower label unless the note explicitly documents the extra evidence required for the higher label.
    \item Do not infer progress from vague intent, general positivity, routine follow-up language, or staff plans that were not discussed with or acted on for the client.
\end{itemize}

\textbf{User message template:}

\medskip
\noindent
Classify the following protected case note. Return valid JSON only with keys label and confidence.

Case note:

$\{case\_note\}$
\end{promptbox}

\end{APPENDICES}

% Acknowledgments here
% \ACKNOWLEDGMENT{}

% References here (outcomment the appropriate case)

% CASE 1: BiBTeX used to constantly update the references
%   (while the paper is being written).
%\bibliographystyle{informs2014} % outcomment this and next line in Case 1
%\bibliography{<your bib file(s)>} % if more than one, comma separated

%\bibliographystyle{informs2014} % outcomment this and next line in Case 1
%\bibliography{sample} % if more than one, comma separated

% CASE 2: BiBTeX used to generate mypaper.bbl (to be further fine tuned)
%\input{mypaper.bbl} % outcomment this line in Case 2

%If you don't use BiBTex, you can manually itemize references as shown below.

%\bibliographystyle{nonumber}

% \begin{thebibliography}{3}
% \providecommand{\natexlab}[1]{#1}
% \providecommand{\url}[1]{\texttt{#1}}
% \providecommand{\urlprefix}{URL }

% \bibitem[{Smith(2005)}]{smith2005}
% Smith J (2005) Optimal resource allocation in humanitarian logistics.
%   \emph{Journal of Operations Research} 30(2):123--135.
  
% \bibitem[{Jones(2010)}]{jones2010}
% Jones S (2010) Stochastic programming models for humanitarian logistics.
%   \emph{INFORMS Mathematics of Operations Research} 35(4):567--580.

% \bibitem[{Brown(2015)}]{brown2015}
% Brown D (2015) \emph{Introduction to Stochastic Programming} (Springer).

% \end{thebibliography}

%%%%%%%%%%%%%%%%%
\end{document}